\documentclass[11pt]{article}

\usepackage[final]{neurips_2021}
\usepackage{mycustom}
\setcitestyle{authoryear}

\usepackage[utf8]{inputenc} 
\usepackage[T1]{fontenc}    

\usepackage[colorlinks=true, allcolors=blue, pdfencoding=auto, psdextra]{hyperref} 

\usepackage{url}            
\usepackage{booktabs}       
\usepackage{amsfonts}       
\usepackage{nicefrac}       
\usepackage{microtype}      
\usepackage{xcolor}         

\usepackage{amsmath,amsthm,amssymb}
\usepackage{graphicx}
\usepackage{cleveref}
\usepackage{subcaption}
\usepackage{cleveref}
\usepackage{comment}
\usepackage{bbm}
\usepackage[ruled, vlined]{algorithm2e}
\usepackage{textcomp}
\usepackage{wrapfig}
\usepackage{float}
\SetKwInOut{Parameter}{parameter}

\usepackage{pgfplots}
\pgfplotsset{compat=1.8}
\tikzset{elegant/.style={smooth,thick,samples=500,magenta}}

\theoremstyle{plain}
\newtheorem{theorem}{Theorem}[section]
\newtheorem{lemma}[theorem]{Lemma}

\newtheorem{corollary}[theorem]{Corollary}
\theoremstyle{definition}
\newtheorem{definition}[theorem]{Definition}

\newtheorem{assumption}[theorem]{Assumption}

\crefname{assumption}{Assumption}{Assumptions}

\usepackage{amsmath,amsfonts,bm}

\def\1{\bm{1}}

\def\eps{{\epsilon}}

\def\vzero{{\bm{0}}}

\def\vmu{{\bm{\mu}}}
\def\vtheta{{\bm{\theta}}}
\def\vdelta{{\bm{\delta}}}
\def\valpha{{\bm{\alpha}}}

\def\vlambda{{\bm{\lambda}}}

\def\vb{{\bm{b}}}
\def\vc{{\bm{c}}}

\makeatletter
\newcommand{\ve}{\@ifnextchar\bgroup{\velong}{{\bm{e}}}}
\newcommand{\velong}[1]{{\bm{#1}}}
\makeatother

\def\vf{{\bm{f}}}
\def\vg{{\bm{g}}}
\def\vh{{\bm{h}}}

\def\vp{{\bm{p}}}
\def\vq{{\bm{q}}}

\def\vs{{\bm{s}}}

\def\vu{{\bm{u}}}
\def\vv{{\bm{v}}}
\def\vw{{\bm{w}}}
\def\vx{{\bm{x}}}
\def\vy{{\bm{y}}}
\def\vz{{\bm{z}}}

\def\mI{{\bm{I}}}

\def\mM{{\bm{M}}}

\def\mP{{\bm{P}}}

\DeclareMathAlphabet{\mathsfit}{\encodingdefault}{\sfdefault}{m}{sl}
\SetMathAlphabet{\mathsfit}{bold}{\encodingdefault}{\sfdefault}{bx}{n}

\newcommand{\E}{\mathbb{E}}
\newcommand{\R}{\mathbb{R}}

\newcommand{\normltwo}{L^2}

\newcommand{\normmax}{L^\infty}

\DeclareMathOperator*{\argmax}{arg\,max}
\DeclareMathOperator*{\argmin}{arg\,min}

\newcommand{\dotp}[2]{\left<#1, #2\right>}
\newcommand{\dotpsm}[2]{\langle #1, #2\rangle}
\newcommand{\paral}{\scriptscriptstyle\parallel}

\newcommand{\barvw}{\bar{\vw}}
\newcommand{\bara}{\bar{a}}
\newcommand{\barb}{\bar{b}}
\newcommand{\hatb}{\hat{b}}
\newcommand{\barvb}{\bar{\vb}}
\newcommand{\hatvb}{\hat{\vb}}
\newcommand{\barvtheta}{\bar{\vtheta}}
\newcommand{\tildevtheta}{\tilde{\vtheta}}
\newcommand{\barvthetaopt}{\bar{\vtheta}^{\ast}}

\newcommand{\hata}{\hat{a}}

\newcommand{\Normal}{\mathcal{N}}
\newcommand{\unif}{\mathrm{unif}}

\newcommand{\Loss}{\mathcal{L}}

\newcommand{\DatS}{\mathcal{S}} 

\newcommand{\OmegaS}{\Omega_{\DatS}}

\newcommand{\abs}[1]{\left\lvert #1 \right\rvert}
\newcommand{\abssm}[1]{\lvert #1 \rvert}

\newcommand{\norm}[1]{\left\| #1 \right\|}
\newcommand{\normsm}[1]{\| #1 \|}
\newcommand{\normtwo}[1]{\norm{#1}_2}
\newcommand{\normtwosm}[1]{\normsm{#1}_2}

\newcommand{\norminfsm}[1]{\normsm{#1}_{\infty}}

\newcommand{\sphS}{\mathbb{S}}

\newcommand{\conv}{\mathrm{conv}}

\newcommand{\sgn}{\mathrm{sgn}}
\newcommand{\barcpartial}{\bar{\partial}^{\circ}}
\newcommand{\barcpartialr}{\bar{\partial}^{\circ}_{\mathrm{r}}}
\newcommand{\barcpartialperp}{\bar{\partial}^{\circ}_{\perp}}
\newcommand{\cpartial}{\partial^{\circ}}
\newcommand{\cder}[1]{#1^{\circ}}

\newcommand{\onec}[1]{\mathbbm{1}_{[#1]}}

\newcommand{\dd}{\textup{\textrm{d}}}

\newcommand{\sigmainit}{\sigma_{\mathrm{init}}}
\newcommand{\sigmainitmax}{\sigma_{\mathrm{init}}^{\max}}
\newcommand{\cAinit}{c_{\mathrm{ainit}}}
\newcommand{\alphaLK}{\alpha_{\mathrm{leaky}}}
\newcommand{\barvmu}{\bar{\vmu}}
\newcommand{\scaledvmu}{\tilde{\vmu}}

\newcommand{\phitheta}{\varphi}
\newcommand{\phisg}{\tilde{\varphi}}

\newcommand{\mMmu}{\mM_{\scaledvmu}}
\newcommand{\normM}[1]{\norm{#1}_{\mathrm{M}}}
\newcommand{\normMsm}[1]{\normsm{#1}_{\mathrm{M}}}

\newcommand{\normP}[1]{\norm{#1}_{\mathrm{P}}}
\newcommand{\normPsm}[1]{\normsm{#1}_{\mathrm{P}}}

\newcommand{\normRsm}[1]{\normsm{#1}_{\mathrm{R}}}

\newcommand{\Dinit}{\mathcal{D}_{\mathrm{init}}}
\newcommand{\Diniteq}{\tilde{\mathcal{D}}_{\mathrm{init}}}

\newcommand{\hai}[1][i]{h^{(1)}_{#1}}
\newcommand{\hbi}[1][i]{h^{(2)}_{#1}}
\newcommand{\sai}[1][i]{\sigma^{(1)}_{#1}}
\newcommand{\sbi}[1][i]{\sigma^{(2)}_{#1}}
\newcommand{\ski}[1][i]{\sigma^{(k)}_{#1}}
\newcommand{\hki}[1][i]{h^{(k)}_{#1}}

\newcommand{\vwopt}{\vw^*}
\newcommand{\vwperp}{\vw_{\perp}}
\newcommand{\vthetaopt}{\vtheta^*}
\newcommand{\vwgar}{\vw^{\diamond}}
\newcommand{\barvu}{\bar{\vu}}

\newcommand{\gammaopt}{\gamma^*}

\newcommand{\gammagar}{\gamma^{\diamond}}
\newcommand{\tildegamma}{\tilde{\gamma}}

\newcommand{\qmin}{q_{\min}}

\newcommand{\mPgar}{\mP^{\diamond}}

\newcommand{\Gronwall}{Gr\"{o}nwall\xspace}

\newcommand{\hatvw}{\hat{\vw}}
\newcommand{\hatvtheta}{\hat{\vtheta}}

\newcommand{\hatsetw}{\widehat{\mathcal{W}}}
\newcommand{\hatgap}{\hat{\Delta}}

\newcommand{\tildevw}{\tilde{\vw}}
\newcommand{\tildea}{\tilde{a}}
\newcommand{\tildeg}{\tilde{g}}

\newcommand{\brmsplus}[1][b]{#1_{+}}
\newcommand{\brmsminus}[1][b]{#1_{-}}

\newcommand{\tildeLoss}{\tilde{\Loss}}

\newcommand{\vvplus}{\vw^+}
\newcommand{\vvsubt}{\vw^-}

\newcommand{\vwplus}{\vw^+}
\newcommand{\vwsubt}{\vw^-}
\newcommand{\gammaplus}{\gamma^+}
\newcommand{\gammasubt}{\gamma^-}
\newcommand{\vxplus}{\vx^+}
\newcommand{\vxsubt}{\vx^-}
\newcommand{\vmuplus}{\vmu^+}
\newcommand{\vmusubt}{\vmu^-}
\newcommand{\barvmuplus}{\bar{\vmu}^+}
\newcommand{\barvmusubt}{\bar{\vmu}^-}

\newcommand{\finftime}{f^{\infty}}

\newcommand{\simiid}{\overset{\text{i.i.d.}}{\sim}}

\newcommand{\cone}[1][]{{\mathcal{C}^{#1}}}
\newcommand{\ocone}[1][]{{\mathring{\mathcal{C}}^{#1}}}
\newcommand{\kone}[1][]{\mathcal{K}^{#1}}
\newcommand{\hone}[1][]{\mathcal{H}^{#1}}
\newcommand{\mup}[1][]{\mathcal{M}^{#1}_+}
\newcommand{\mun}[1][]{\mathcal{M}^{#1}_-}
\newcommand{\dist}[2]{\mathrm{dist}\left(#1,#2\right)}
\newcommand{\distsm}[2]{\mathrm{dist}(#1,#2)}

\newcommand{\zvmuplus}{{\barvmu_2^+}}
\newcommand{\zvmusubt}{{\barvmu_2^-}}

\newcommand{\tmpts}{t_{\mathrm{s}}}

\newcommand{\Splus}{\mathcal{S}^+}
\newcommand{\Lambdaplus}{\Lambda^+}

\title{Gradient Descent on Two-layer Nets: \\ Margin Maximization and Simplicity Bias}

\author{%
  Kaifeng Lyu\thanks{Equal contribution}~~\thanks{Most of the work is done when Kaifeng Lyu and Runzhe Wang were at Tsinghua University.} \\
  Princeton University\\
  \texttt{klyu@cs.princeton.edu} \\
  \And
  Zhiyuan Li\footnotemark[1] \\
  Princeton University\\
  \texttt{zhiyuanli@cs.princeton.edu} \\
  \AND
  Runzhe Wang\footnotemark[1]~~\footnotemark[2] \\
  Princeton University\\
  \texttt{runzhew@princeton.edu} \\
  \And
  Sanjeev Arora \\
  Princeton University\\
  \texttt{arora@cs.princeton.edu}
}

\begin{document}

\maketitle

\begin{abstract}
The generalization mystery of overparametrized deep nets has motivated efforts
to understand how gradient descent (GD) converges to low-loss solutions that
generalize well. Real-life neural networks are initialized from small random
values and trained with cross-entropy loss for classification (unlike the "lazy"
or "NTK" regime of training where analysis was more successful), and a recent
sequence of results \citep{lyu2020gradient,chizat20logistic,ji2020directional}
provide theoretical evidence that GD may converge to the "max-margin" solution
with zero loss, which presumably generalizes well. However, the global
optimality of margin is proved only in some settings where neural nets are
infinitely or exponentially wide. The current paper is able to establish this
global optimality for two-layer Leaky ReLU nets trained with gradient flow on
linearly separable and symmetric data, regardless of the width. The analysis
also gives some theoretical justification for recent empirical findings
\citep{kalimeris2019sgd} on the so-called simplicity bias of GD towards linear
or other "simple" classes of solutions, especially early in training. On the
pessimistic side, the paper suggests that such results are fragile. A simple
data manipulation can make gradient flow converge to a linear classifier with
suboptimal margin.
\end{abstract}

\section{Introduction}

\vspace{-0.03in}
One major mystery in deep learning is why deep neural networks generalize
despite  overparameterization~\citep{zhang2017rethinking}. To tackle this issue,
many recent works turn to study the {\em implicit bias} of gradient descent (GD)
--- what kind of theoretical characterization can we give for the low-loss
solution found by GD?

The seminal works by \citet{soudry2018implicit,soudry2018iclrImplicit} revealed
an interesting connection between GD and margin maximization: for linear
logistic regression on linearly separable data, there can be multiple linear
classifiers that perfectly fit the data, but GD with any initialization always
converges to the max-margin (hard-margin SVM) solution, even when there is no
explicit regularization. Thus the solution found by GD has the same margin-based
generalization bounds as hard-margin SVM. Subsequent works on linear models have
extended this theoretical understanding of GD to
SGD~\citep{nacson2018stochastic}, other gradient-based
methods~\citep{gunasekar2018characterizing}, other loss functions with certain
poly-exponential tails~\citep{nacson2019convergence}, linearly non-separable
data~\citep{ji2018risk,ji2019risk}, deep linear
nets~\citep{ji2018gradient,gunasekar2018implicit}.

Given the above results, a natural question to ask is whether GD has the same
implicit bias towards max-margin solutions for machine learning models in
general. \citet{lyu2020gradient} studied the relationship between GD and margin
maximization on {\em deep homogeneous neural network}, i.e., neural network
whose output function is (positively) homogeneous with respect to its
parameters. For homogeneous neural networks, only the direction of parameter
matters for classification tasks. For logistic and exponential loss,
\citet{lyu2020gradient} assumed that GD decreases the loss to a small value and
achieves full training accuracy at some time point, and then provided an
analysis for the training dynamics after this time
point~(\Cref{thm:converge-kkt-margin}), which we refer to as \textit{late phase
analysis}.  It is shown that GD decreases the loss to $0$ in the end and
converges to a direction satisfying the Karush-Kuhn-Tucker (KKT) conditions of a
constrained optimization problem \eqref{eq:prob-P} on margin maximization.

However, given the non-convex nature of neural networks, KKT conditions do not
imply global optimality for margins. Several attempts are made to prove the
global optimality specifically for two-layer nets. \citet{chizat20logistic}
provided a mean-field analysis for infinitely wide two-layer Squared ReLU nets
showing that gradient flow converges to the solution with global max margin,
which also corresponds to the max-margin classifier in some non-Hilbertian space
of functions. \citet{ji2020directional} extended the proof to finite-width
neural nets, but the width needs to be exponential in the input dimension (due
to the use of a covering condition). Both works build upon late phase analyses.
Under a restrictive assumption that the data is orthogonally separable, i.e.,
any data point $\vx_i$ can serve as a perfect linear separator,
\citet{phuong2021the} analyzed the full trajectory of gradient flow on two-layer
ReLU nets with small initialization, and established the convergence to a
piecewise linear classifier that maximizes the margin, irrespective of network
width.

In this paper, we study the implicit bias of gradient flow on two-layer neural
nets with Leaky ReLU activation~\citep{maas2013leakyrelu} and logistic loss. To
avoid the {\em lazy}  or {\em Neural Tangent Kernel (NTK)} regime where the
weights are initialized to large random values and do not change much during
training
\citep{jacot2018ntk,chizat2019lazy,du2018gradient,du2018global,allenzhu2018gobeyond,allenzhu2018convergence,zou2018stochastic,arora2019fine},
we use small initialization to encourage the model to learn features actively,
which is closer to real-life neural network training.

When analyzing convergence behavior of training on neural networks, one can
simplify the problem and gain insights by assuming that the data distribution
has a simple structure. Many works particularly study the case where the labels
are generated by an unknown teacher network that is much smaller/simpler than
the (student) neural network to be trained.
Following~\citet{brutzkus2018sgd,sarussi2021towards} and many other works, we
consider the case where the dataset is linearly separable, namely the labels are
generated by a linear teacher, and study the training dynamics of two-layer
Leaky ReLU nets on such dataset.

\subsection{Our Contribution}

Among all the classifiers that can be represented
by the two-layer Leaky ReLU nets, we show \textbf{any global-max-margin
classifier is exactly linear} under one more data assumption: the dataset is
{\em symmetric}, i.e., if $\vx$ is in the training set, then so is $-\vx$. Note
that such symmetry can be ensured by simple data augmentation.

Still, little is known about what kind of classifiers neural network trained by
GD learns. Though \citet{lyu2020gradient} showed that gradient flow converges
to a classifier along KKT-margin direction, we note that this result is not
sufficient to guarantee the global optimality since such classifier can have
nonlinear decision boundaries. See \Cref{fig:decision_boundary} (left) for an example.

In this paper, we provide a multi-phase analysis for the full trajectory of
gradient flow, in contrast with previous late phase analyses which only analyzes
the trajectory after achieving $100\%$ training accuracy. We show that
\textbf{gradient flow with small initialization converges to a global-max-margin linear classifier}
(\Cref{thm:maxmar-linear}). The proof leverages power iteration to show that
neuron weights align in two directions in an early phase of training, inspired
by \citet{li2021towards}. We further show the alignment at any constant training
time by associating the dynamics of wide neural net with that of two-neuron
neural net, and finally, extend the alignment to the infinite time limit by
applying Kurdyka-Łojasiewicz (KL) inquality in a similar way as
\citet{ji2020directional}. The alignment at convergence implies that the
convergent classifier is linear.

The above results also justify a recent line of works studying the so-called
\textit{simplicity bias}: GD first learns linear functions in the early phase of
training, and the complexity of the solution increases as training goes
on~\citep{kalimeris2019sgd,hu2020surprising,shah2020pitfalls}. Indeed, our
result establishes a form of \textit{extreme simplicity bias} of GD: \textit{if
the dataset can be fitted by a linear classifier, then GD learns a linear
classifier not only in the beginning but also at convergence.}
 
On the pessimistic side, this paper suggests that such global margin
maximization result could be fragile. Even for linearly separable data,
global-max-margin classifiers may be nonlinear without the symmetry assumption.
In particular, we show that for any linearly separable dataset, \textbf{gradient
flow can be led to converge to a linear classifier with suboptimal margin by
adding only $3$ extra data points} (\Cref{thm:hinted_main}). See
\Cref{fig:decision_boundary} (right) for an example.

\section{Related Works}

\myparagraph{Generalization Aspect of Margin Maxmization.} Margin often appears
in the generalization bounds for neural
networks~\citep{bartlett2017norm,neyshabur2017norm}, and larger margin leads to
smaller bounds. \citet{jiang2020Fantastic} conducted an empirical study for the
causal relationships between complexity measures and generalization errors, and
showed positive results for normalized margin, which is defined by the output
margin divided by the product (or powers of the sum) of Frobenius norms of
weight matrices from each layer. On the pessimistic side, negative results are
also shown if Frobenius norm is replaced by spectral norm. In this paper, we do
use the normalized margin with Frobenius norm (see \Cref{sec:prelim}).

\myparagraph{Learning on Linearly Separable Data.} Some works
studied the training dynamics of (nonlinear) neural networks on linearly
separable data (labels are generated by a linear teacher).
\citet{brutzkus2018sgd} showed that SGD on two-layer Leaky ReLU nets with hinge
loss fits the training set in finite steps and generalizes well.
\citet{frei2021provable} studied online SGD (taking a fresh sample from the
population in each step) on the two-layer Leaky ReLU nets with logistic loss.
For any data distribution, they proved that there exists a time step in the
early phase such that the net has a test error competitive with that of the best
linear classifier over the distribution, and hence generalizes well on linearly
separable data. Both two papers reveal that the weight vectors in the first
layer have positive correlations with the weight of the linear teacher, but
their analyses do not imply that the learned classifier is linear. In the NTK
regime, \citet{ji2020polylogarithmic,chen2021how} showed that GD on shallow/deep
neural nets learns a kernel predictor with good generalization on linearly
separable data, and it suffices to have width polylogarithmic in the number of
training samples. Still, they do not imply that the learned classifier is
linear. \citet{pellegrini2020analytic} provided a mean-field analysis for
two-layer ReLU net showing that training with hinge loss and infinite data leads
to a linear classifier, but their analysis requires the data distribution to be
spherically symmetric (i.e., the probability density only depends on the
distance to origin), which is a more restrictive assumption than ours.
\citet{sarussi2021towards} provided a late phase analysis for gradient flow on
two-layer Leaky ReLU nets with logistic loss, which establishes the convergence
to linear classifier based on an assumption called \textit{Neural Agreement
Regime} (NAR): starting from some time point, for any training sample, the
outputs of all the neurons have the same sign. However, it is unclear why this
can happen a priori. Comparing with our work, we analyze the full trajectory of
gradient flow and establish the convergence to linear classifier without
assuming NAR. \citet{phuong2021the} analyzed the full trajectory for gradient
flow on orthogonally separable data, but every KKT-margin direction attains the
global max margin~(see \Cref{sec:orthogonal_separable}) in their setting, which
it is not necessarily true in general. In our setting, KKT-margin direction with
suboptimal margin does exist.
 
\myparagraph{Simplicity Bias.} \citet{kalimeris2019sgd} empirically observed
that neural networks in the early phase of training are learning linear
classifiers, and provided evidence that SGD learns functions of increasing
complexity. \citet{hu2020surprising} justified this view by proving that the
learning dynamics of two-layer neural nets and simple linear classifiers are
close to each other in the early phase, for dataset drawn from a data
distribution where input coordinates are independent after some linear
transformation. The aforementioned work by \citet{frei2021provable} can be seen
as another theoretical justification for online SGD on aribitrary data
distribution. \citet{shah2020pitfalls} pointed out that extreme simplicity bias
can lead to suboptimal generalization and negative effects on adversarial
robustness. 

\myparagraph{Small Initialization.} Several theoretical works studying neural
network training with small initialization can be connected to simplicity bias.
\citet{maennel2018gradient} uncovered a weight quantization effect in training
two-layer nets with small initialization: gradient flow biases the weight
vectors to a certain number of directions determined by the input data
(independent of neural network width). It is hence argued that gradient flow has
a bias towards ``simple'' functions, but their proof is not entirely rigorous
and no clear definition of simplicity is given. This weight quantization effect
has also been studied under the names of weight
clustering~\citep{brutzkus2019why},
condensation~\citep{luo2021phase,xu2021towards}. \citet{williams2019gradient}
studied univariate regression and showed that two-layer ReLU nets with small
initialization tend to learn linear splines. For the matrix factorization
problem,  which can be related to training neural networks with linear or
quadratic activations, we can measure the complexity of the learned solution by
rank. A line of works showed that gradient descent learns solutions with
gradually increasing
rank~\citep{li2018algorithmic,arora2019implicit,gidel2019implicit,gissin2020the,li2021towards}.
Such results have been generalized to tensor factorization where the complexity
measure is replaced by tensor rank~\citep{razin2021implicit}. Beyond small
initialization of our interest and large initialization in the lazy or NTK
regime,  \citet{woodworth20a,moroshko2020implicit,mehta2021extreme} studied
feature learning when the initialization scale transitions from small to large
scale.

\section{Preliminaries} \label{sec:prelim}

We denote the set $\{1, \dots, n\}$ by $[n]$ and the unit sphere $\{\vx \in \R^d
: \normtwosm{\vx} = 1\}$ by $\sphS^{d-1}$. We call a function $h: \R^D \to \R$
$L$-\textit{homogeneous} if $h(c\vtheta) = c^L h(\vtheta)$ for all $\vtheta \in
\R^D$ and $c > 0$. For $S\subseteq \R^D$, $\conv(S)$ denotes the convex hull of
$S$. For locally Lipschitz function $f: \R^D \to \R$, we define Clarke's
subdifferential
\citep{clarke1975generalized,clarke2008nonsmooth,davis2018stochastic} to be
$\cpartial f(\vtheta) := \conv\left\{ \lim_{n \to \infty} \nabla f(\vtheta_n) :
f \text{ differentiable at }\vtheta_n, \lim_{n \to \infty}\vtheta_n = \vtheta
\right\}$ (see also \Cref{sec:add-not}).

\begin{figure}[t]
\vspace{-0.1cm}
\centering
\includegraphics[width=\textwidth]{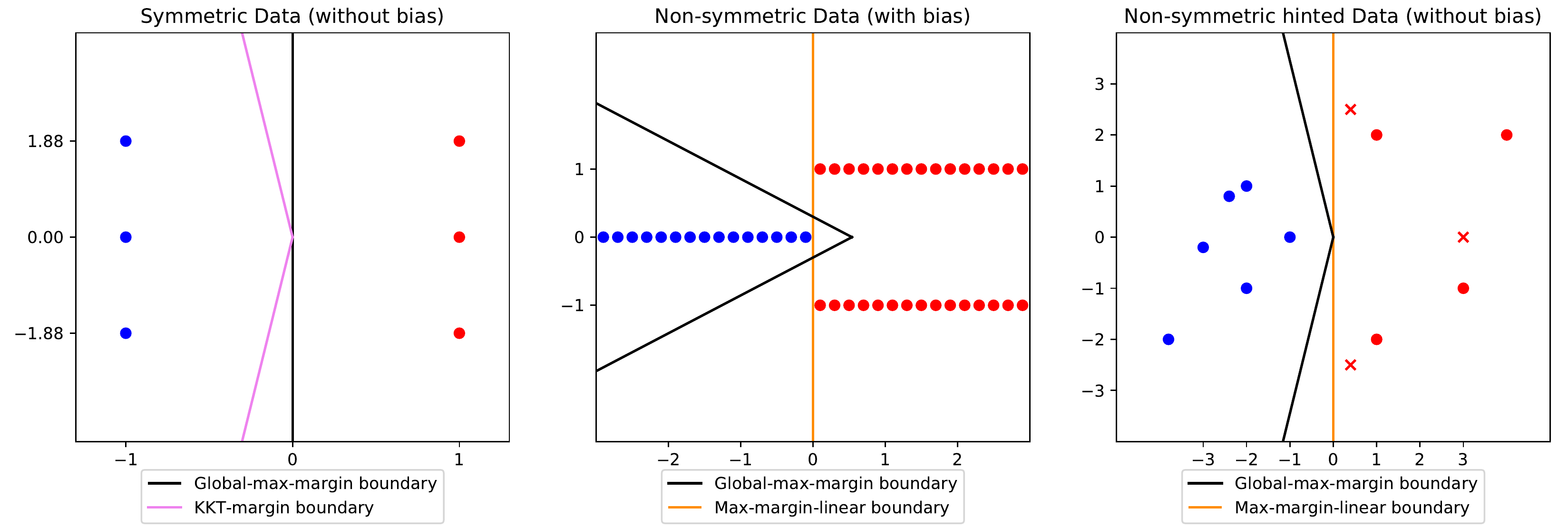}
\caption{Two-layer Leaky ReLU nets ($\alphaLK=1/2$) with KKT margin and global
	max margin on linearly separable data. See \Cref{sec:proof-figure} for detailed discussions. \textbf{Left}: \Cref{thm:sym_main} is
	not vacuous: a symmetric dataset can have KKT directions with suboptimal
	margin, but our theory shows that gradient flow from small initialization
	goes to global max margin. \textbf{Middle}: The linear classifier
	(orange) is along a KKT-margin direction with a much smaller margin
	comparing to the (nonlinear) global-max-margin classifier (black), but our
	theory suggests that gradient flow converges to the linear classifier.
	\textbf{Right}: Adding three extra data points (marked as ``\texttt{x}'';
	see \Cref{def:hint}) to a linearly separable dataset makes the linear
	classifier (orange) has suboptimal margin but causes the neural net to be
	biased to it.}\label{fig:decision_boundary}
	\vspace{-0.05in}
\end{figure}

\vspace{-0.04in}
\subsection{Logistic Loss Minimization and Margin Maximization} \label{sec:review-margin-maximization}

For a neural net, we use $f_{\vtheta}(\vx) \in \R$ to denote the output logit on
input $\vx \in \R^d$ when the parameter is $\vtheta \in \R^D$. We say that the
neural net is $L$-\textit{homogeneous} if $f_{\vtheta}(\vx)$ is $L$-homogeneous
with respect to $\vtheta$, i.e., $f_{c\vtheta}(\vx) = c^L f_{\vtheta}(\vx)$ for
all $\vtheta \in \R^D$ and $c > 0$. VGG-like CNNs can be made homogeneous if we
remove all the bias terms expect those in the first
layer~\citep{lyu2020gradient}.

Throughout this paper, we restrict our attention to $L$-homogeneous neural nets
with $f_{\vtheta}(\vx)$ definable with respect to $\vtheta$ in an o-minimal
structure for all $\vx$. (See \citealt{coste2000introduction} for reference for
o-minimal structures.) This is a technical condition needed by
\Cref{thm:converge-kkt-margin}, and it is a mild regularity condition as almost all modern
neural networks satisfy this condition, including the two-layer Leaky ReLU
networks studied in this paper.

For a dataset $\DatS = \{(\vx_1, y_1), \dots, (\vx_n, y_n)\}$, we define
$q_i(\vtheta) := y_i f_{\vtheta}(\vx_i)$ to be the \textit{output margin on the
data point} $(\vx_i, y_i)$, and $\qmin(\vtheta) := \min_{i \in [n]}
q_i(\vtheta)$ to be the \textit{output margin on the dataset} $\DatS$ (or
\textit{margin} for short). It is easy to see that $q_1(\vtheta), \dots,
q_n(\vtheta)$ are $L$-homogeneous functions, and so is $\qmin(\vtheta)$. We
define the \textit{normalized margin} $\gamma(\vtheta) :=
\qmin\left(\frac{\vtheta}{\normtwosm{\vtheta}}\right) =
\frac{\qmin(\vtheta)}{\normtwosm{\vtheta}^L}$ to be the output margin (on the
dataset) for the normalized parameter $\frac{\vtheta}{\normtwosm{\vtheta}}$.

We refer the problem of finding $\vtheta$ that maximizes $\gamma(\vtheta)$ as
\textit{margin maximization}. Note that once we have found an optimal solution
$\vthetaopt \in \R^D$, $c\vthetaopt$ is also optimal for all $c > 0$. We can put
the norm constraint on $\vtheta$ to eliminate this freedom on rescaling:
\begin{equation} \label{eq:prob-M}
	\max_{\vtheta \in \sphS^{D-1}}~~\gamma(\vtheta). \tag{M}
\end{equation}
Alternatively, we can also constrain the margin to have $\qmin \ge 1$ and minimize the norm:
\begin{equation} \label{eq:prob-P}
	\min~~\frac{1}{2}\normtwosm{\vtheta}^2 \quad \text{s.t.}\quad  q_i(\vtheta) \ge 1, \quad \forall i \in [n]. \tag{P}
\end{equation}
One can easily show that $\vthetaopt$ is a global maximizer of \eqref{eq:prob-M}
if and only if $\frac{\vthetaopt}{(\qmin(\vthetaopt))^{1/L}}$ is a global
minimizer of \eqref{eq:prob-P}. For convenience, we make the following
convention: if $\frac{\vtheta}{\normtwosm{\vtheta}}$ is a local/global maximizer
of \eqref{eq:prob-M}, then we say $\vtheta$ is along a \textit{local-max-margin
direction}/\textit{global-max-margin direction}; if
$\frac{\vtheta}{(\qmin(\vtheta))^{1/L}}$ satisfies the KKT conditions of
\eqref{eq:prob-P}, then we say $\vtheta$ is along a \textit{KKT-margin
direction}.

Gradient flow with logistic loss is defined by the following differential
inclusion, 
\begin{equation}\label{eq:subgradient_flow}
\frac{\dd \vtheta}{\dd t} \in -\cpartial \Loss(\vtheta), \quad \textrm{with } 	\Loss(\vtheta) := \frac{1}{n}\sum_{i=1}^{n} \ell(q_i(\vtheta)), 
\end{equation}
where $\ell(q):=\ln(1 + e^{-q})$ is the logistic loss.
\citet{lyu2020gradient,ji2020directional} showed that
$\vtheta(t)/\normtwosm{\vtheta(t)}$ always  converges to a KKT-margin direction.
We restate the results below.
\begin{theorem}[\citealt{lyu2020gradient,ji2020directional}]
	\label{thm:converge-kkt-margin} For homogeneous neural networks, if $\Loss(\vtheta(0)) < \frac{\ln 2}{n}$, then
	$\Loss(\vtheta(t)) \to 0$, $\normtwosm{\vtheta(t)} \to +\infty$, and
	$\frac{\vtheta(t)}{\normtwosm{\vtheta(t)}}$ converges to a KKT-margin
	direction as $t \to +\infty$.
\end{theorem}

\subsection{Two-Layer Leaky ReLU Networks on Linearly Separable Data}

Let $\phi(x) = \max\{x, \alphaLK x\}$ be Leaky ReLU, where $\alphaLK \in (0,
1)$. Throughout the following sections, we consider a two-layer neural net
defined as below,
\[
f_{\vtheta}(\vx) = \sum_{k=1}^{m} a_k \phi(\vw_k^{\top} \vx).
\]
where $\vw_1, \dots, \vw_m \in \R^d$ are the weights in the first layer, $a_1,
\dots, a_m \in \R$ are the weights in the second layer, and $\vtheta = (\vw_1,
\dots, \vw_m, a_1, \dots, a_m) \in \R^D$ is  the concatenation of all trainable
parameters, where $D = md+m$. We can verify that $f_{\vtheta}(\vx)$ is
$2$-homogeneous with respect to $\vtheta$.

Let $\DatS := \{(\vx_1, y_1), \dots, (\vx_n, y_n)\}$ be the training set. For
simplicity, we assume that $\normtwosm{\vx_i} \le 1$. We focus on linearly
separable data, thus we assume that $\DatS$ is linearly separable throughout the
paper.
\begin{assumption}[Linear Separable] \label{ass:lin}
	There exists a $\vw \in \R^d$ such that $y_i \dotp{\vw}{\vx_i} \ge 1$ for all $i \in [n]$.
\end{assumption}
\begin{definition}[Max-margin Linear Separator]
	For the linearly separable dataset $\DatS$, we say that $\vwopt \in
	\sphS^{d-1}$ is the max-margin linear separator if $\vwopt$ maximizes
	$\min_{i \in [n]} y_i \dotp{\vw}{\vx_i}$ over $\vw \in \sphS^{d-1}$.
\end{definition}

\section{Training on Linearly Separable and Symmetric Data} \label{sec:sym}
In this section, we study the implicit bias of gradient flow assuming the
training data is linearly separable and \textit{symmetric}. We say a dataset is
symmetric if whenever $\vx$ is present in the training set, the input $-\vx$ is
also present. By linear separability, $\vx$ and $-\vx$ must have different
labels because $\dotp{\vwopt}{\vx} = -\dotp{\vwopt}{-\vx}$, where $\vwopt$ is
the max-margin linear separator. The formal statement for this assumption is
given below.

\begin{assumption}[Symmetric] \label{ass:sym}
	$n$ is even and $\vx_{i} = -\vx_{i+n/2}, y_i = 1, y_{i+n/2} = -1$ for $1 \le i \le n/2$.
\end{assumption}

This symmetry can be ensured via data augmentation. Given a dataset, if it is
known that the ground-truth labels are produced by an unknown linear classifier,
then one can augment each data point $(\vx, y)$ by flipping the sign, i.e.,
replace it with two data points $(\vx, y)$, $(-\vx, -y)$ (and thus the dataset
size is doubled). 

Our results show that gradient flow directionally converges to a
global-max-margin direction for two-layer Leaky ReLU networks, when the dataset
is linearly separable and symmetric. To achieve such result, the key insight is
that any global-max-margin direction represents a linear classifier, which we
will see in \Cref{sec:sym-kkt}. Then we will present our main convergence
results in \Cref{sec:sym-dynamics}.

\subsection{Global-Max-Margin Classifiers are Linear} \label{sec:sym-kkt}

\Cref{thm:maxmar-linear} below characterizes the global-max-margin direction in
our case by showing that margin maximization and simplicity bias coincide with
each other: a network that representing the \textit{max-margin linear
classifier} (i.e., $f_{\vtheta}(\vx) = c\dotp{\vwopt}{\vx}$ for some $c > 0$)
can simultaneously achieve the goals of being simple and maximizing the margin.

\begin{theorem}\label{thm:maxmar-linear} Under \Cref{ass:lin,ass:sym}, for the
	two-layer Leaky ReLU network with width $m \ge 2$, any global-max-margin
	direction $\vthetaopt \in \sphS^{D-1}$, $f_{\vthetaopt}$ represents a linear
	classifier. Moreover, we have $f_{\vthetaopt}(\vx) = \frac{1 + \alphaLK}{4}
	\dotp{\vwopt}{\vx}$ for all $\vx \in \R^d$, where $\vwopt$ is the max-margin
	linear separator.
\end{theorem}

The result of \Cref{thm:maxmar-linear} is based on the observation that
replacing each neuron $(a_k,\vw_k)$ in a network with two neurons of oppositing
parameters $(a_k,\vw_k)$ and $(-a_k,-\vw_k)$ does not decrease the normalized
margin on the symmetric dataset, while making the classifier linear in function
space. Thus if any direction attains the global max margin, we can construct a
new global-max-margin direction which corresponds to a linear classifier. We can
show that every weight vector $\vw_k$ of this linear classifier must be in the
direction of $\vwopt$ or $-\vwopt$. Then the original classifier must also be
linear in the same direction. 

\subsection{Convergence to Global-Max-Margin Directions} \label{sec:sym-dynamics}

Though \Cref{thm:converge-kkt-margin} guarantees that gradient flow
directionally converges to a KKT-margin direction if the loss is optimized
successfully, we note that KKT-margin directions can be non-linear and have
complicated decision boundaries. See \Cref{fig:decision_boundary} (left) for an
example. Therefore, to establish the convergence to linear classifiers,
\Cref{thm:converge-kkt-margin} is not enough and we need a new analysis for the
trajectory of gradient flow.

We use initialization $\vw_k \simiid \Normal(\vzero, \sigmainit^2 \mI)$, $a_k
\simiid \Normal(0, \cAinit^2\sigmainit^2)$, where $\cAinit$ is a fixed constant
throughout this paper and $\sigmainit$ controls the initialization scale. We
call this distribution as $\vtheta_0 \sim \Dinit(\sigmainit)$. An alternative
way to generate this distribution is to first draw $\barvtheta_0 \sim
\Dinit(1)$, and then set $\vtheta_0 = \sigmainit \barvtheta_0$. With small
initialization, we can establish the following convergence result.
\begin{theorem}\label{thm:sym_main} Under \Cref{ass:lin,ass:sym} and certain
	regularity conditions (see \Cref{ass:dotp-mu-x-non-zero,ass:non-branching}
	below), consider gradient flow on a Leaky ReLU network with width $m \ge 2$
	and initialization $\vtheta_0 = \sigmainit \barvtheta_0$ where $\barvtheta_0
	\sim \Dinit(1)$. With probability $1 - 2^{-(m-1)}$ over the random draw of
	$\barvtheta_0$, if the initialization scale is sufficiently small, then
	gradient flow directionally converges and $\finftime(\vx) := \lim_{t \to
	+\infty} f_{\vtheta(t)/\normtwosm{\vtheta(t)}}(\vx)$ represents
	the max-margin linear classifier. That is,
	\[
	\Pr_{\barvtheta_0 \sim \Dinit(1)}\left[\exists \sigmainitmax > 0 \text{ s.t. } \forall \sigmainit < \sigmainitmax, \forall \vx \in \R^d, \finftime(\vx) = C\dotp{\vwopt}{\vx}\right] \ge 1 - 2^{-(m-1)},
	\]
	where $C := \frac{1+\alphaLK}{4}$ is a scaling factor.
\end{theorem}
Combining \Cref{thm:maxmar-linear} and \Cref{thm:sym_main}, we can conclude that
gradient flow achieves the global max margin in our case.
\begin{corollary}
	In the settings of \Cref{thm:sym_main}, gradient flow on linearly separable
	and symmetric data directionally converges to the global-max-margin
	direction with probability $1 - 2^{-(m-1)}$.
\end{corollary}

\subsection{Additional Notations and Assumptions}
Let $\vmu := \frac{1}{n}\sum_{i=1}^{n} y_i \vx_i$, which is non-zero since
$\dotp{\vmu}{\vw_*} = \frac{1}{n} \sum_{i \in [n]} y_i \vw_*^{\top} \vx_i \ge
1$. Let $\barvmu := \frac{\vmu}{\normtwosm{\vmu}}$. We use $\phitheta(\vtheta_0,
t) \in \R^d$ to the value of $\vtheta$ at time $t$ for $\vtheta(0) = \vtheta_0$.

We make the following technical assumption, which holds if we are allowed to add
a slight perturbation to the training set.
\begin{assumption} \label{ass:dotp-mu-x-non-zero}
	For all $i \in [n]$, $\dotp{\vmu}{\vx_i} \ne 0$.
\end{assumption}
Another technical issue we face is that the gradient flow may not be unique due
to non-smoothness. It is possible that $\phitheta(\vtheta_0, t)$ is not
well-defined as the solution of \eqref{eq:subgradient_flow} may not be unique.
See \Cref{sec:nonunique} for more discussions. In this case, we assign
$\phitheta(\vtheta_0, \,\cdot\,)$ to be an arbitrary gradient flow trajectory
starting from $\vtheta_0$. In the case where $\phitheta(\vtheta_0, t)$ has only
one possible value for all $t \ge 0$, we say that $\vtheta_0$ is a
\textit{non-branching starting point}. We assume the following technical
assumption.
\begin{assumption} \label{ass:non-branching} For any $m \ge 2$, there exist $r,
	\epsilon > 0$ such that $\vtheta$ is a non-branching starting point if its
	neurons can be partitioned into two groups: in the first group, $a_k =
	\normtwosm{\vw_k} \in (0, r)$ and all $\vw_k$ point to the same direction
	$\vvplus \in \sphS^{d-1}$ with $\normtwosm{\vvplus - \barvmu} \le \epsilon$;
	in the second group, $-a_k = \normtwosm{\vw_k} \in (0, r)$ and all $\vw_k$
	point to the same direction $\vvsubt \in \sphS^{d-1}$ with
	$\normtwosm{\vvsubt + \barvmu} \le \epsilon$.
\end{assumption}

\section{Proof Sketch for the Symmetric Case} \label{sec:sketch}

In this section, we provide a proof sketch for \Cref{thm:sym_main}. Our proof
uses a multi-phase analysis, which divides the training process into $3$ phases,
from small initialization to the final convergence. We will now elaborate the
analyses for them one by one.

\subsection{Phase I: Dynamics Near Zero} \label{sec:sketch-phase-i}

Gradient flow starts with small initialization. In Phase I, we analyze the
dynamics when gradient flow does not go far away from zero. Inspired by
\citet{li2021towards}, we relate such dynamics to power iterations and show that
every weight vector $\vw_k$ in the first layer moves towards the directions of
either $\barvmu$ or $-\barvmu$. To see this, the first step is to note that
$f_{\vtheta}(\vx_i) \approx 0$ when $\vtheta$ is close to $\vzero$. Applying
Taylor expansion on $\ell(y_i f_{\vtheta}(\vx_i))$,
\begin{equation} 
	\Loss(\vtheta) = \frac{1}{n} \sum_{i \in [n]} \ell(y_i f_{\vtheta}(\vx_i)) \approx \frac{1}{n}\sum_{i \in [n]} \left(\ell(0) + \ell'(0) y_i f_{\vtheta}(\vx_i)\right).
\end{equation}
Expanding $f_{\vtheta}(\vx_i)$ and reorganizing the terms, we have
\begin{align*}
	\Loss(\vtheta) \approx \frac{1}{n}\sum_{i \in [n]} \ell(0) + \frac{1}{n}\sum_{i \in [n]} \ell'(0) \sum_{k \in [m]} y_ia_k \phi(\vw_k^{\top} \vx_i) &= \ell(0) + \frac{\ell'(0)}{n} \sum_{k \in [m]}\sum_{i \in [n]} y_ia_k \phi(\vw_k^{\top} \vx_i) \\
	&= \ell(0) - \sum_{k \in [m]} a_kG(\vw_k),
\end{align*}
where $G$-function~\citep{maennel2018gradient} is defined below:
\[
	G(\vw) := \frac{-\ell'(0)}{n}\sum_{i \in [n]} y_i \phi(\vw^{\top} \vx_i) = \frac{1}{2n}\sum_{i \in [n]} y_i \phi(\vw^{\top} \vx_i).
\]
This means gradient flow optimizes each $-a_kG(\vw_k)$ separately near origin.
\begin{equation} \label{eq:phase-i-a-G-approx}
	\frac{\dd \vw_k}{\dd t} \approx a_k \cpartial G(\vw_k), \qquad
	\frac{\dd a_k}{\dd t} \approx G(\vw_k).
\end{equation}
In the case where \Cref{ass:sym} holds, we can pair each $\vx_i$ with $-\vx_i$
and use the identity $\phi(z) - \phi(-z) = \max\{z, \alphaLK z\} - \max\{-z,
-\alphaLK z\} = (1+\alphaLK)z$ to show that $G(\vw)$ is linear:
\begin{align*}
G(\vw) = \frac{1}{2n} \sum_{i \in [n/2]} \left(\phi(\vw^{\top} \vx_i) - \phi(-\vw^{\top} \vx_i)\right) &= \frac{1}{2n} \sum_{i \in [n/2]} (1 + \alphaLK) \vw^{\top}\vx_i = \dotp{\vw}{\scaledvmu},
\end{align*}
where $\scaledvmu := \frac{1+\alphaLK}{2}\vmu = \frac{1+\alphaLK}{2n} \sum_{i
\in [n]} y_i \vx_i$. Substituting this formula for $G$ into
\eqref{eq:phase-i-a-G-approx} reveals that the dynamics of two-layer neural nets
near zero has a close relationship to power iteration (or matrix exponentiation)
of a matrix $\mMmu \in \R^{(d+1) \times (d+1)}$ that only depends on data.
\[
	\frac{\dd}{\dd t} \begin{bmatrix}
	\vw_k \\
	a_k
	\end{bmatrix} \approx \mMmu\begin{bmatrix}
	\vw_k \\
	a_k
	\end{bmatrix}, \qquad \text{where} \qquad \mMmu := \begin{bmatrix}
	\vzero & \scaledvmu \\
	\scaledvmu^{\top} & 0
	\end{bmatrix}.
\]
Simple linear algebra shows that $\lambda_0 := \normtwosm{\scaledvmu}$,
$\frac{1}{\sqrt{2}}(\barvmu, 1) \in \R^{d+1}$ are the unique top eigenvalue and
eigenvector of $\mMmu$, which suggests that $(\vw_k(t), a_k(t)) \in \R^{d+1}$
aligns to this top eigenvector direction if the approximation
\eqref{eq:phase-i-a-G-approx} holds for a sufficiently long time. 
With small initialization, this can indeed be true and we obtain the following lemma.
\begin{definition}[M-norm] \label{def:norm-M} For parameter vector $\vtheta =
	(\vw_1, \dots, \vw_m, a_1, \dots, a_m)$, we define the M-norm to be
	$\normMsm{\vtheta} = \max_{k \in [m]} \left\{ \max\{\normtwosm{\vw_k},
	\abssm{a_k}\} \right\}$.
\end{definition}
\begin{lemma}\label{lm:phase1-diff} Let $r > 0$ be a small value. With
	probability $1$ over the random draw of $\barvtheta_0 = (\bar\vw_1, \dots,
	\barvw_m, \bara_1, \dots, \bara_m) \sim \Dinit(1)$, if we take $\sigmainit
	\le \frac{r^3}{\sqrt{m}\normM{\barvtheta_0}}$, then any neuron $(\vw_k,
	a_k)$ at time $T_1(r) := \frac{1}{\lambda_0} \ln
	\frac{r}{\sqrt{m}\sigmainit\normM{\barvtheta_0}}$ can be decomposed into
	\[
		\vw_k(T_1(r)) = r\barb_k \barvmu + \Delta\vw_k, \qquad a_k(T_1(r)) = r\barb_k + \Delta a_k,
	\]
	where $\barb_k := \frac{\dotpsm{\barvw_k}{\barvmu} + \bara_k}{2\sqrt{m} \normM{\barvtheta_0}}$ and the error term $\Delta \vtheta := (\Delta\vw_1, \dots, \Delta \vw_m, \Delta a_1, \dots, \Delta a_m)$ is bounded by $\normMsm{\Delta \vtheta} \le \frac{Cr^3}{\sqrt{m}}$ for some universal constant $C$.
\end{lemma}

\subsection{Phase II: Near-Two-Neuron Dynamics}

By \Cref{lm:phase1-diff}, we know that at time $T_1(r)$ we have $\vw_k(T_1(r))
\approx r\barb_k \barvmu$ and $a_k(T_1(r)) \approx r\barb_k$, where $\barvb \in
\R^{d}$ is some fixed vector. This motivates us to couple the training dynamics
of $\vtheta(t) = (\vw_1(t), \dots, \vw_m(t), a_1(t), \dots, a_m(t))$ after the
time $T_1(r)$ with another gradient flow starting from the point $(r\barb_1
\barvmu, \dots, r\barb_m \barvmu, r\barb_1, \dots, r\barb_m)$. Interestingly,
the latter dynamic can be seen as a dynamic of two neurons ``embedded''
into the $m$-neuron neural net, and we will show that $\vtheta(t)$ is close to this
``embedded'' two-neuron dynamic for a long time. Now we first
introduce our idea of embedding a two-neuron network into an $m$-neuron network. 

\myparagraph{Embedding.} For any $\vb \in \R^m$, we say that $\vb$ is a
\textit{good embedding vector} if it has at least one positive entry and one
negative entry, and all the entries are non-zero. For a good embedding vector
$\vb$, we use $\brmsplus := \sqrt{\sum_{j\in[m]} \onec{b_j>0}b_j^2}$ and
$\brmsminus := -\sqrt{\sum_{j\in[m]} \onec{b_j<0}b_j^2}$ to denote the
root-sum-squared of the positive entries and the negative root-sum-squared of
the negative entries. For parameter $\hatvtheta := (\hat{\vw}_1, \hat{\vw}_2,
\hat{a}_1, \hat{a}_2)$ of a two-neuron neural net with $\hat{a}_1 > 0$ and
$\hat{a}_2 < 0$, we define the \textit{embedding} from two-neuron into
$m$-neuron neural nets as $\pi_{\vb}(\hat{\vw}_1, \hat{\vw}_2, \hat{a}_1,
\hat{a}_2) = (\vw_1, \dots, \vw_m, a_1, \dots, a_m)$, where
\[
a_k   = \begin{cases}
\frac{b_k}{\brmsplus} \hat{a}_1, & \textrm{if } b_k > 0 \\
\frac{b_k}{\brmsminus} \hat{a}_2,& \textrm{if } b_k < 0 \\
\end{cases}, \qquad
\vw_k = \begin{cases}
\frac{b_k}{\brmsplus} \hat{\vw}_1, & \textrm{if } b_k > 0 \\
\frac{b_k}{\brmsminus} \hat{\vw}_2, &\textrm{if } b_k < 0 \\
\end{cases}.
\]
It is easy to check that $f_{\hatvtheta}(\vx) = f_{\pi_{\vb}(\hatvtheta)}(\vx)$
by the homogeneity of the activation ($\phi(cz) = c\phi(z)$ for $c > 0$):
\begin{align*}
	f_{\pi_{\vb}(\hatvtheta)}(\vx) &= \sum_{b_k > 0} a_k \phi(\vw_k^\top \vx) + \sum_{b_k < 0} a_k \phi(\vw_k^\top \vx) \\
	&= \sum_{b_k > 0} \frac{b_k^2}{\brmsplus^2} \hat{a}_1 \phi(\hat{\vw}_1^{\top} \vx) + \sum_{b_k < 0} \frac{b_k^2}{\brmsminus^2} \hat{a}_2 \phi(\hat{\vw}_2^{\top} \vx) = \hat{a}_1 \phi(\hat{\vw}_1^{\top} \vx) + \hat{a}_2 \phi(\hat{\vw}_2^{\top} \vx) = f_{\hatvtheta}(\vx).
\end{align*}
Moreover, by taking the chain rule, we can obtain the following lemma showing
that the trajectories starting from $\hatvtheta$ and $\pi_{\vb}(\hatvtheta)$ are
essentially the same.
\begin{lemma} \label{lm:two-neuron-embedding} Given $\hatvtheta := (\hat{\vw}_1,
	\hat{\vw}_2, \hat{a}_1, \hat{a}_2)$ with $\hat{a}_1 > 0$ and $\hat{a}_2 <
	0$, if both $\hatvtheta$ and $\pi_{\vb}(\hatvtheta)$ are non-branching
	starting points, then $\phitheta(\pi_{\vb}(\hatvtheta), t) =
	\pi_{\vb}(\phitheta(\hatvtheta, t))$ for all $t \ge 0$.
\end{lemma}

\myparagraph{Approximate Embedding.} Back to our analysis for Phase II, $\barvb$
is a good embedding vector with high probability (see lemma below). Let
$\hatvtheta := (\brmsplus[\barb], \brmsplus[\barb]\barvmu, \brmsminus[\barb],
\brmsminus[\barb]\barvmu)$. By \Cref{lm:phase1-diff}, $\pi_{\barvb}(r\hatvtheta)
= (r\barb_1 \barvmu, \dots, r\barb_m \barvmu, r\barb_1, \dots, r\barb_m) \approx
\vtheta(T_1(r))$, which means $r\hatvtheta \to \vtheta(T_1(r))$ is approximately
an embedding. Suppose that the approximation happens to be exact, namely
$\pi_{\barvb}(r\hatvtheta) = \vtheta(T_1(r))$, then $\vtheta(T_1(r) + t) =
\pi_{\barvb}(\phitheta(r\hatvtheta, t))$ by \Cref{lm:two-neuron-embedding}.
Inspired by this, we consider the case where $\sigmainit \to 0, r \to 0$ so that
the approximate embedding is infinitely close to the exact one, and prove the
following lemma. We shift the training time by $T_2(r)$ to avoid trivial limits
(such as $\vzero$).
\begin{lemma} \label{lm:phase2-main} Follow the notations in
	\Cref{lm:phase1-diff} and take $\sigmainit \le
	\frac{r^3}{\sqrt{m}\normM{\barvtheta_0}}$. Let $T_2(r) :=
	\frac{1}{\lambda_0}\ln \frac{1}{r}$, then $T_{12} := T_1(r) + T_2(r) =
	\frac{1}{\lambda_0} \ln \frac{1}{\sqrt{m} \sigmainit \normM{\barvtheta_0}}$
	regardless the choice of $r$. For width $m \ge 2$, with probability
	$1-2^{-(m-1)}$ over the random draw of $\barvtheta_0 \sim \Dinit(1)$, the
	vector $\barvb \in \R^m$ is a good embedding vector, and for the two-neuron
	dynamics starting with rescaled initialization in the direction of
	$\hatvtheta := (\brmsplus[\barb], \brmsplus[\barb]\barvmu,
	\brmsminus[\barb], \brmsminus[\barb]\barvmu)$, the following limit exists
	for all $t$,
	\begin{equation}
		\tildevtheta(t) := \lim_{r \to 0} \phitheta\left(r\hatvtheta, T_2(r) + t\right) \ne \vzero,
	\end{equation}
	and moreover, for the $m$-neuron dynamics of $\vtheta(t)$, the following holds for all $t$,
	\begin{equation} \label{eq:phase2-main-sigma-to-zero}
		\lim_{\sigmainit \to 0} \vtheta\left(T_{12} + t\right) = \pi_{\barvb}(\tildevtheta(t)).
	\end{equation}
\end{lemma}

\subsection{Phase III: Dynamics near Global-Max-Margin Direction}

With some efforts, we have the following characterization for the two-neuron dynamics.
\begin{theorem} \label{thm:two-neuron-max-margin}
	For $m = 2$, if initially $a_1 = \normtwosm{\vw_1}$, $a_2 = -\normtwosm{\vw_2}$, $\dotp{\vw_1}{\vwopt} > 0$ and $\dotp{\vw_2}{\vwopt} < 0$, then $\vtheta(t)$ directionally converges to the following global-max-margin direction,
	\[
	\lim_{t \to +\infty} \frac{\vtheta(t)}{\normtwosm{\vtheta(t)}} = \frac{1}{4}(\vwopt, -\vwopt, 1, -1),
	\]
	where $\vwopt$ is the max-margin linear separator.
\end{theorem}
It is not hard to verify that $\tildevtheta(t)$ satisfies the conditions
required by \Cref{thm:two-neuron-max-margin}. Given this result, a first attempt
to establish the convergence of $\vtheta(t)$ to global-max-margin direction is
to take $t \to +\infty$ on both sides of \eqref{eq:phase2-main-sigma-to-zero}.
However, this only proves that $\vtheta\left(T_{12} + t\right)$ directionally
converges to the global-max-margin direction if we take the limit $\sigmainit
\to 0$ first then take $t \to +\infty$, while we are interested in the
convergent solution when $t \to +\infty$ first then $\sigmainit \to 0$ (i.e.,
 solution gradient flow converges to with infinite training time, if it starts from sufficiently small initialization).
These two double limits are not equivalent because the order of limits cannot be
exchanged without extra conditions. 

To overcome this issue, we follow a similar proof strategy as
\citet{ji2020directional} to prove local convergence near a local-max-margin
direction, as formally stated below. \Cref{thm:phase-4-general} holds for
$L$-homogeneous neural networks in general and we believe is of independent
interest.
\begin{theorem} \label{thm:phase-4-general} Consider any $L$-homogeneous neural
	networks with logistic loss. Given a local-max-margin direction
	$\barvthetaopt \in \sphS^{D-1}$ and any $\delta > 0$, there exists
	$\epsilon_0 > 0$ and $\rho_0 \ge 1$ such that for any $\vtheta_0$ with norm
	$\normtwosm{\vtheta_0} \ge \rho_0$ and direction
	$\normtwo{\frac{\vtheta_0}{\normtwosm{\vtheta_0}} - \barvthetaopt} \le
	\epsilon_0$, gradient flow starting with $\vtheta_0$ directionally converges
	to some direction $\barvtheta$ with the same normalized margin $\gamma$ as
	$\barvthetaopt$, and $\normtwosm{\barvtheta - \barvthetaopt} \le \delta$.
\end{theorem}
Using \Cref{thm:phase-4-general}, we can finish the proof for
\Cref{thm:sym_main} as follows. First we note that the two-neuron
global-max-margin direction $\frac{1}{4}(\vwopt, -\vwopt, 1, -1)$ after
embedding is a global-max-margin direction for $m$-neurons, and we can prove
that any direction with distance no more than a small constant $\delta > 0$ is
still a global-max-margin direction. Then we can take $t$ to be large enough so
that $\pi_{\barvb}(\tildevtheta(t))$ satisfies the conditions in
\Cref{thm:phase-4-general}. According to \eqref{eq:phase2-main-sigma-to-zero},
we can also make the conditions hold for $\vtheta\left(T_{12} + t\right)$ by
taking $\sigmainit$ and $r$ to be sufficiently small. Finally, applying
\Cref{thm:phase-4-general} finishes the proof.

\section{Non-symmetric Data Complicates the Picture} \label{sec:nonsym}
Now we turn to study the case without assuming symmetry and the question is
whether the implicit bias to global-max-margin solution still holds.
Unfortunately, it turns out the convergence to global-max-margin classifier is
very fragile --- for any linearly separable dataset, we can add $3$ extra data
points so that every linear classifier has suboptimal margin but still gradient
flow with small initialization converges to a linear classifier.\footnote{Here
linear classifier refers to a classifier whose decision boundary is linear.}
See \Cref{def:hint} for the construction and \Cref{fig:decision_boundary}
(right) for an example.

Unlike the symmetric case, we use balanced Gaussian initialization instead of
purely random Gaussian initialization: $\vw_k \sim \Normal(\vzero, \sigmainit^2
\mI)$, $a_k = s_k \normtwosm{\vw_k}$, where $s_k \sim \unif\{\pm 1\}$. We call
this distribution as $\vtheta_0 \sim \Diniteq(\sigmainit)$. This adaptation can
greatly simplify our analysis since it ensures that $a_k(t) =
s_k\normtwosm{\vw_k(t)}$ for all $t \ge 0$~(\Cref{cor:weight-equal}). Similar as
the symmetric case, an alternative way to generate this distribution is to first
draw $\barvtheta_0 \sim \Diniteq(1)$, and then set $\vtheta_0 = \sigmainit
\barvtheta_0$.

\begin{definition}[($H$, $K$ $\epsilon$, $\vwperp$)-Hinted
	Dataset]\label{def:hint} Given a linearly separable dataset $\DatS$ with
	max-margin linear separator $\vwopt$, for constants $H, K, \epsilon > 0$ and
	unit vector $\vwperp \in \sphS^{d-1}$ perpendicular to $\vwopt$, we define
	the ($H$, $K$, $\epsilon$, $\vwperp$)-hinted dataset $\DatS'$ by the dataset
	containing all the data points in $\DatS$ and the following $3$ data points
	(numbered by $1, 2, 3$) that can serve as hints to the max-margin linear
	separator $\vwopt$:
	\[
		(\vx_1, y_1) = (H \vwopt, 1), \qquad (\vx_2, y_2) = (\epsilon \vwopt + K\vwperp, 1), \qquad (\vx_3, y_3) = (\epsilon \vwopt - K\vwperp, 1).
	\]
\end{definition}

\begin{theorem}\label{thm:hinted_main} Given a linearly separable dataset
	$\DatS$ and a unit vector $\vwperp \in \sphS^{d-1}$ perpendicular to the
	max-margin linear separator $\vwopt$, for any sufficiently large $H > 0, K >
	0$ and sufficiently small $\epsilon > 0$, the following statement holds for
	the ($H$, $K$, $\epsilon$, $\vwperp$)-Hinted Dataset $\DatS'$. Under a
	regularity assumption for gradient flow (see
	\Cref{ass:non-branching-non-sym}), consider gradient flow on a Leaky ReLU
	network with width $m \ge 1$ and initialization $\vtheta_0 = \sigmainit
	\barvtheta_0$ where $\barvtheta_0 \sim \Diniteq(1)$. With probability $1
	-2^{-m}$ over the draw of $\barvtheta_0$, if the initialization scale is
	sufficiently small, then gradient flow directionally converges and
	$\finftime(\vx) := \lim_{t \to +\infty}
	f_{\vtheta(t)/\normtwosm{\vtheta(t)}}(\vx)$ represents the one-Leaky-ReLU
	classifier $\frac{1}{2}\phi(\dotp{\vwopt}{\vx})$ with linear decision
	boundary. That is,
	\[
	\Pr_{\barvtheta_0 \sim \Dinit(1)}\left[\exists \sigmainitmax > 0 \text{ s.t. } \forall \sigmainit < \sigmainitmax, \forall \vx \in \R^d, \finftime(\vx) = \frac{1}{2}\phi(\dotpsm{\vwopt}{\vx})\right] \ge 1 - \delta.
	\]
	Moreover, the convergent classifier only attains a suboptimal margin.
\end{theorem}
\Cref{thm:hinted_main} is actually a simple corollary general theorem under data
assumptions that hold for a broader class of linearly separable data. From a
high-level perspective, we only require two assumptions: (1).~There is a direction
such that data points have large inner products with this direction on average;
(2).~The support vectors for the max-margin linear separator $\vwopt$ have
nearly the same labels. The first hint data point is for the first condition and
the second and third data point is for the second condition. We defer formal
statements of the assumptions and theorems to \Cref{sec:nonsym-thm}.

\section{Conclusions and Future Works}

We study the implicit bias of gradient flow in training two-layer Leaky ReLU
networks on linearly separable datasets. When the dataset is symmetric, we show
any global-max-margin classifier is exactly linear and gradient flow converges
to a global-max-margin direction. On the pessimistic side, we show such margin
maximization result is fragile --- for any linearly separable dataset, we can
lead gradient flow to converge to a linear classifier with suboptimal margin by
adding only $3$ extra data points. A critical assumption for our convergence
analysis is the linear separability of data. We left it as a future work to
study simplicity bias and global margin maximization without assuming linear
separability.

\begin{ack}
The authors acknowledge support from NSF, ONR, Simons Foundation, DARPA and SRC.
ZL is also supported by Microsoft Research PhD Fellowship.
\end{ack}

\bibliography{main}

\newpage

\appendix

\section{Theorem Statements for the Non-symmetric Case} \label{sec:nonsym-thm}

\subsection{Assumptions and Main Theorems}

For every $\vx_i$, define $\vxplus_i := \vx_i$ if $y_i = 1$ and $\vxplus_i :=
\alphaLK\vx_i$ if $y_i = -1$. Similarly, we define $\vxsubt_i := \alphaLK\vx_i$
if $y_i = 1$ and $\vxplus_i := \vx_i$ if $y_i = -1$. Then we define $\vmuplus$
to be the mean vector of $y_i \vxplus_i$, and $\vmusubt$ to be the mean vector
of $y_i \vxsubt_i$, that is,
\begin{align}\label{eq:mu_plus_minus}
\vmuplus := \frac{1}{n}\sum_{i \in [n]} y_i \vxplus_i, \qquad \vmusubt := \frac{1}{n}\sum_{i \in [n]} y_i \vxsubt_i.
\end{align}
\Cref{thm:hinted_main} is indeed a simple corollary of \Cref{thm:nonsym_main}
below which holds for a broader class of datasets. Now we illustrate the
assumptions one by one.

We first make the following assumption saying that there is a principal
direction $\vwgar \in \sphS^{d-1}$ such that data points on average have much
larger inner products with $\vwgar$ than any other direction perpendicular to
$\vwgar$. This ensures at small initialization, the moving direction of each
neurons lies in  a small cone around the the direction of $\pm\vwgar$, and thus
will converge to that cone eventually. The opening angle of this small cone is
$2\arcsin\frac{\gammagar}{\max_{i \in [n]} \normtwosm{\vx_i}}$, which ensures
the sign pattern inside the cone $\left\{\dotpsm{\vw}{\vx_i}\right\}_{i =1}^{n}$
is unique and indeed equal to $\{y_i\}_{i=1}^n$, and thus all neurons converge
to two directions, $\vmu^+$ and $\vmu^-$ (defined in \eqref{eq:mu_plus_minus}).
\begin{assumption}[Existence of Principal Direction] \label{ass:principal-dir}
	There exists a unit-norm vector $\vwgar$ such that $\gammagar := \min_{i \in [n]} y_i\dotp{\vwgar}{\vx_i} > 0$ and
    \[
        \frac{\frac{1}{n}\sum_{i \in [n]} \normtwosm{\mPgar \vx_i}}{\alphaLK\dotp{\vmu}{\vwgar}} < \frac{\gammagar}{\max_{i \in [n]} \normtwosm{\mPgar \vx_i}},
    \]
    where $\mPgar := \mI - \vwgar {\vwgar}^{\top}$ is the projection matrix onto
    the space perpendicular to $\vwgar$, and $\vmu := \frac{1}{n}\sum_{i \in
    [n]} y_i \vx_i$ is the mean vector of $y_i\vx_i$.
\end{assumption}

Indeed, our main theorem is based on a weaker assumption than
\Cref{ass:principal-dir}, which is \Cref{assump:cone} below, but the geometric
meaning of \Cref{assump:cone} is not as clear as \Cref{ass:principal-dir}. We
will show in \Cref{lem:assump_cone_imply_principal_direction} that
\Cref{ass:principal-dir} implies \Cref{assump:cone}.
\begin{assumption} \label{assump:cone}
	For all $i \in [n]$, we have
	\[
	\dotp{\vmu}{y_i\vx_i} > \frac{1-\alphaLK}{n \cdot \alphaLK}\sum_{j\in [n]}\max\{-\dotpsm{y_i\vx_i}{y_j\vx_j}, 0\}.
	\]
\end{assumption}

In general, the norms $\normtwosm{\vmuplus}$ and $\normtwosm{\vmusubt}$ should
not be equal: for any given dataset $\DatS$, we can make $\normtwosm{\vmuplus}
\ne \normtwosm{\vmusubt}$ by adding arbitrarily small perturbations to the data
points. This motivates us to assume that $\normtwosm{\vmuplus} \ne
\normtwosm{\vmusubt}$. Without loss of generality, we can assume that
$\normtwosm{\vmuplus} > \normtwosm{\vmusubt}$ for convenience
(\Cref{ass:vmuplus-greater}). When the reverse is true, i.e.,
$\normtwosm{\vmuplus} < \normtwosm{\vmusubt}$, we can change the direction of
the inequality by flipping all the labels in the dataset so that our theorems
can apply. We include the theorem statements for this reversed case in
\Cref{sec:nonsym-thm-reversed}.

\begin{assumption} \label{ass:vmuplus-greater}
	The norm of $\vmuplus$ is strictly larger than $\vmusubt$, i.e., $\normtwosm{\vmuplus} > \normtwosm{\vmusubt}$.
\end{assumption}

Now we define $\vwplus$ to be the max-margin linear separator of the dataset consisting of $(\vxplus_i, y_i)$, where $i \in [n]$, and define $\gammaplus$ to be this max margin. That is,
\[
	\vwplus := \argmax_{\vw \in \sphS^{d-1}} \left\{\min_{i \in [n]} y_i\dotp{\vw}{\vxplus_i}\right\}, \qquad \gammaplus := \max_{\vw \in \sphS^{d-1}} \left\{\min_{i \in [n]} y_i\dotp{\vw}{\vxplus_i}\right\}.
\]
The reason that we care about $\vwplus$ and $\gammaplus$ is because that it can
be related to margin maximization on one-neuron Leaky ReLU nets. The following
lemma is easy to prove.
\begin{lemma}
	For $m = 1$, if $\vtheta = (\vw_1, a_1) \in \sphS^{D-1}$ is a KKT-margin
	direction and $a_1 \ge 0$, then $\vtheta = (\frac{1}{\sqrt{2}}\vwplus,
	\frac{1}{\sqrt{2}})$, and it attains the global max margin $\frac{1}{2}
	\gammaplus$.
\end{lemma}
The third assumption we made is that this margin cannot be obtained when all
$a_i$ are negative, regardless of the width. This assumption holds when all the
support vectors $\vx_i^+$ have positive labels, i.e., $y_i=1$. Conceptually,
this assumption is about whether nearly all the support vectors have positive
labels (or negative labels in the reversed case where $\normtwosm{\vmuplus} <
\normtwosm{\vmusubt}$).
\begin{assumption} \label{ass:margin-negative-net}
	For any $m \ge 1$ and any $\vtheta = (\vw_1, \dots, \vw_m, a_1, \dots, a_m) \in \R^D$, if $a_k \le 0$ for all $k \in [m]$, then the normalized margin $\gamma(\vtheta)$ on the dataset $\{(\vx_i, y_i) : i \in [n], y_i\dotpsm{\vwplus}{\vxplus_i} = \gammaplus \}$ is less than $\frac{1}{2}\gammaplus$.
\end{assumption}

Similar to \Cref{ass:non-branching} in the symmetric case, we  need \Cref{ass:non-branching-non-sym} on non-branching starting point due to the technical difficulty for the potential non-uniqueness of gradient flow trajectory.
\begin{assumption} \label{ass:non-branching-non-sym} For any $m \ge 1$, there
	exist $r, \epsilon > 0$ such that $\vtheta$ is a non-branching starting
	point if $a_k = \normtwosm{\vw_k} \in (0, r)$ holds for all $k \in [m]$, and
	all $\vw_k$ point to the same direction $\vv \in \sphS^{d-1}$ with
	$\normtwo{\vv - \frac{\vmuplus}{\normtwosm{\vmuplus}}} \le \epsilon$.
\end{assumption}

Now we are ready to state our theorem, and we defer the proofs to
\Cref{sec:proof-nonsym}.
\begin{theorem}\label{thm:nonsym_main} Under
	\Cref{ass:lin,assump:cone,ass:vmuplus-greater,ass:margin-negative-net,ass:non-branching-non-sym},
	consider gradient flow on a Leaky ReLU network with width $m \ge 1$ and
	initialization $\vtheta_0 = \sigmainit \barvtheta_0$ where $\barvtheta_0
	\sim \Diniteq(1)$. With probability $1 - 2^{-m}$ over the draw of
	$\barvtheta_0$, if the initialization scale is sufficiently small, then
	gradient flow directionally converges and $\finftime(\vx) := \lim_{t \to
	+\infty} f_{\vtheta(t)/\normtwosm{\vtheta(t)}}(\vx)$ represents the
	one-Leaky-ReLU classifier $\frac{1}{2}\phi(\dotp{\vwplus}{\vx})$ with linear
	decision boundary. That is,
	\[
	\Pr_{\barvtheta_0 \sim \Dinit(1)}\left[\exists \sigmainitmax > 0 \text{ s.t. } \forall \sigmainit < \sigmainitmax, \forall \vx \in \R^d, \finftime(\vx) = \frac{1}{2}\phi(\dotpsm{\vwplus}{\vx})\right] \ge 1 - 2^{-m}.
	\]
\end{theorem}

\subsection{Applying Theorem \ref{thm:nonsym_main} to prove Theorem
\ref{thm:hinted_main}}
We give a proof of \Cref{thm:hinted_main} here given the result of \Cref{thm:nonsym_main}.
\begin{proof}
With a ($H$, $K$ $\epsilon$, $\vwperp$)-Hinted Dataset (\Cref{def:hint}) with proper $H, K, \epsilon$, we only need to show that \Cref{assump:cone,ass:vmuplus-greater,ass:margin-negative-net} hold for \Cref{thm:hinted_main}.  Specifically, we choose the parameters such that
\begin{itemize}
\item $K>0$;
\item $\epsilon<\alphaLK\min_{i>3}y_i\dotp{\vwopt}{\vx_i}$;
\item $H>\max\{\epsilon,H_0,n\normtwo{\vmusubt}+\normtwosm{\sum_{j>1}y_j\vxplus_j}\}$, where\\ 
$H_0=\frac{\max_{i \in [n]} \normtwosm{\mP^* \vx_i}\sum_{i \in [n]} \normtwosm{\mP^* \vx_i}}{\alphaLK\min_{i>1}\dotp{y_i\vx_i}{\vwopt}}-\sum_{i>1}\dotp{\vwopt}{y_i\vx_i}$ and $\mP^*=\mI - \vwopt {\vwopt}^{\top}$ is the projection matrix onto the orthogonal space of $\vwopt$.
\end{itemize}

Notice that $H_0$ is indepenent of $H$ as the data point $\vx_1$ has projection
$\normtwosm{\mP^* x_1}=0$. For \Cref{ass:principal-dir}, $\vwgar=\vwopt$ is a
valid principal direction in this case, as
\[
	\max_{i \in [n]} \normtwosm{\mPgar \vx_i}\frac{\frac{1}{n}\sum_{i \in [n]}\normtwosm{\mPgar\vx_i}}{\alphaLK\gammagar}=\frac{1}{n}(H_0+\sum_{i>1}\dotp{w^*}{y_i\vx_i})<\dotp{\vmu}{\vwgar}.
\]
Then \Cref{assump:cone} follows from \Cref{ass:principal-dir} by
\Cref{lem:assump_cone_imply_principal_direction}. Since
$H>n\normtwo{\vmusubt}+\normtwosm{\sum_{j>1}y_j\vxplus_j}$,
\[
	\normtwo{\vmuplus} \geq \frac{1}{n}H-\normtwo{\frac{1}{n}\sum_{j>1}y_j\vxplus_j}>\normtwo{\vmusubt},
\]
and thus \Cref{ass:vmuplus-greater} holds. Furthermore, with
$\epsilon<\alphaLK\min_{i>3}y_i\dotp{\vwopt}{\vx_i}$ and $H>\epsilon$,
$(\vx_2,y_2)=(\epsilon\vwopt+K\vwperp,1)$ and
$(\vx_3,y_3)=(\epsilon\vwopt-K\vwperp,1)$ are the only support vectors for the
linear margin problem on $\{(\vx_i,y_i)\}$ and that on $\{(\vxplus_i,y_i)\}$ as
well. Then $\vwplus=\vwopt$ and $\gammaplus=\epsilon$. For a neuron with
$a_k<0$, the total output margin on the hints $(\vx_2,y_2)$ and $(\vx_3,y_3)$ is
$a_k\phi(\vw_k^\top \vx_2)+a_k\phi(\vw_k^\top \vx_3)\leq 2\alphaLK\epsilon
|a_k|\normtwosm{\vw_k}\leq \alphaLK\epsilon (a^2_k+\normtwosm{\vw_k}^2)$. Thus
the normalized margin for multiple such neurons is at most
$\frac{\alphaLK\epsilon}{2}< \frac{\epsilon}{2}$, so
\Cref{ass:margin-negative-net} will also be true.
\end{proof}

\subsection{Results in the Reversed Case} \label{sec:nonsym-thm-reversed}

In a reversed case where $\normtwosm{\vmuplus} < \normtwosm{\vmusubt}$, we can apply \Cref{thm:nonsym_main}
by flipping the labels in the dataset. Below we state the assumptions and the theorem in the reversed case.

\begin{assumption} \label{ass:vmuplus-less}
	$\normtwosm{\vmuplus} < \normtwosm{\vmusubt}$.
\end{assumption}

Now similarly we define $\vwsubt$ and $\gammasubt$.
\[
	\vwsubt := \argmax_{\vw \in \sphS^{d-1}} \left\{\min_{i \in [n]} y_i\dotp{\vw}{\vxsubt}\right\}, \qquad \gammasubt := \max_{\vw \in \sphS^{d-1}} \left\{\min_{i \in [n]} y_i\dotp{\vw}{\vxsubt_i}\right\}.
\]
\begin{assumption} \label{ass:margin-positive-net}
	For any $m \ge 1$ and any $\vtheta = (\vw_1, \dots, \vw_m, a_1, \dots, a_m) \in \R^D$, if $a_k \le 0$ for all $k \in [m]$, then the normalized margin $\gamma(\vtheta)$ on the dataset $\{(\vx_i, y_i) : i \in [n], y_i\dotpsm{\vwsubt}{\vxsubt_i} = \gammasubt \}$ is less than $\frac{1}{2}\gammasubt$.
\end{assumption}

\begin{theorem}\label{thm:nonsym_reversed} Under
	\Cref{ass:lin,assump:cone,ass:vmuplus-less,ass:margin-positive-net,ass:non-branching-non-sym},
	consider gradient flow on a Leaky ReLU network with width $m \ge 1$ and
	initialization $\vtheta_0 = \sigmainit \barvtheta_0$ where $\barvtheta_0
	\sim \Diniteq(1)$. With probability $1 - 2^{-m}$ over the draw of
	$\barvtheta_0$, there is an sufficiently small initialization scale, such
	that gradient flow directionally converges and $\finftime(\vx) := \lim_{t
	\to +\infty} f_{\vtheta(t)/\normtwosm{\vtheta(t)}}(\vx)$ represents the
	one-Leaky-ReLU classifier $-\frac{1}{2}\phi(-\dotp{\vwsubt}{\vx})$ with
	linear decision boundary. That is,
	\[
	\Pr_{\barvtheta_0 \sim \Dinit(1)}\left[\exists \sigmainitmax > 0 \text{ s.t. } \forall \sigmainit < \sigmainitmax, \forall \vx \in \R^d, \finftime(\vx) = -\frac{1}{2}\phi(-\dotpsm{\vwsubt}{\vx})\right] \ge 1 - 2^{-m}.
	\]
\end{theorem}

\section{Additional Preliminaries and Lemmas}
In this section, we will introduce additional notations and give some
preliminary results for the dynamics of the two-layer Leaky ReLU network. The
only assumption we will use for the results in the section is that the input
norm is bounded $\max_{i\in [n]}\normtwo{\vx_i} \le 1$ and we do not assume
other properties of the dataset (such as symmetry) except we assume it
explicitly.

\subsection{Additional Notations} \label{sec:add-not}

For notational convenience for calculation with subgradients, we generalize the
following notations for vectors to vector sets. More specifically, we define
\begin{itemize}
	\item $\forall A,B \subseteq \R^d$, $A+B := \{\vx+\vy: \vx\in A,\vy\in B\}$ and $A- B := A+(-B)$;
	\item $\forall A \subseteq \R^d$, $\lambda\in \R$, $\lambda A := \{\lambda \vx: \vx\in A \}$;
	\item Let $\normsm{\cdot}$ be any norm on $\R^d$, $\forall A\subseteq \R^d$, $\norm{A}:= \{\norm{x}: x\in A\} \subseteq \R$;
	\item $\forall A\subseteq \R^d$ and $\vy \in\R^d$, $\dotp{\vy}{A}\equiv \dotp{A}{\vy}:= \{\dotp{\vx}{\vy}: \vx\in A\}$;
	\item We use $\distsm{\vx}{\vy} := \normtwosm{\vx - \vy}$ to denote the $\normltwo$-distance between $\vx \in \R^d$ and $\vy \in \R^d$, $\distsm{A}{\vy} := \inf_{\vx \in A}\normtwosm{\vx - \vy}$ to denote the minimum $\normltwo$-distance between any $\vx \in A$ and $\vy \in \R^d$, and $\distsm{A}{B} := \inf_{\vx \in A, \vy \in B}\normtwosm{\vx - \vy}$ to denote the minimum $\normltwo$-distance between any $\vx \in A$ and any $\vy \in B$.
\end{itemize}
By Rademacher theorem, any real-valued locally Lipschitz function on $\R^D$ is differentiable almost everywhere (a.e.) in the sense of Lebesgue measure. For a locally Lipschitz function $\Loss: \R^D \to \R$, we use $\nabla \Loss(\vtheta) \in \R^D$ to denote the usual gradient (if $\Loss$ is differentiable at $\vtheta$) and $\cpartial \Loss(\vtheta) \subseteq \R^D$ to denote Clarke's subdifferential. The definition of Clarke's subdifferential is given by \eqref{eq:def-clarke}: for any sequence of differentiable points converging to $\vtheta$, we collect convergent gradients from such sequences and take the convex hull as the Clarke's subdifferential at $\vtheta$.
\begin{equation} \label{eq:def-clarke}
\cpartial \Loss(\vtheta) := \conv\left\{ \lim_{n \to \infty} \nabla \Loss(\vtheta_n) : \Loss \text{ differentiable at }\vtheta_n, \lim_{n \to \infty}\vtheta_n = \vtheta \right\}.
\end{equation}
For any full measure set $\Omega \subseteq \R^D$ that does not contain any non-differentiable points, \eqref{eq:def-clarke} also has the following equivalent form:
\begin{equation} \label{eq:def-clarke2}
\cpartial \Loss(\vtheta) = \conv\left\{ \lim_{n \to \infty} \nabla \Loss(\vtheta_n) : \vtheta_n \in \Omega \text{ for all } n \text{ and } \lim_{n \to \infty}\vtheta_n = \vtheta \right\}.
\end{equation}
The Clarke's subdifferential $\cpartial \Loss(\vtheta)$ is convex compact
if $\Loss$ is locally Lipschitz, and it is upper-semicontinuous with respect to $\vtheta$ (or equivalently it has closed graph) if $\Loss$ is definable.
We use $\barcpartial \Loss(\vtheta) \in \R^D$ to denote the min-norm gradient vector in the Clarke's subdifferential at $\vtheta$, i.e., $\barcpartial \Loss(\vtheta) := \argmin_{\vg \in \cpartial \Loss(\vtheta)} \normtwosm{\vg}$. If $\Loss$ is continuously differentiable at $\vtheta$, then $\cpartial \Loss(\vtheta) = \{\nabla \Loss(\vtheta)\}$ and $\barcpartial \Loss(\vtheta) = \nabla \Loss(\vtheta)$.

If $\vtheta$ can be written as $\vtheta = (\vtheta_1, \vtheta_2) \in \R^{D_1} \times \R^{D_2}$, then we use $\frac{\partial \Loss(\vtheta)}{\partial \vtheta_1} \in \R^{D_1}$ to denote the usual partial derivatives (partial gradient) and $\frac{\cpartial \Loss(\vtheta)}{\partial \vtheta_1} \subseteq \R^{D_1}$ to denote the partial subderivatives (partial subgradient) in the sense of Clarke.

Furthermore, we use the following notations to denote the radial and spherical
components of $\barcpartial \Loss(\vtheta)$ (which will be used in analyzing Phase III):
\[
\barcpartialr \Loss(\vtheta) := \frac{\vtheta\vtheta^{\top}}{\normtwosm{\vtheta}^2}\barcpartial \Loss(\vtheta), \qquad \barcpartialperp \Loss(\vtheta) := \left(\mI - \frac{\vtheta\vtheta^{\top}}{\normtwosm{\vtheta}^2}\right)\barcpartial \Loss(\vtheta).
\]
For univariate function $f: \R \to \R$, we use $f'(z) \in \R$ to denote the usual derivative (if $f$ is differentiable at $z$) and $\cder{f}(z) \subseteq \R$ to denote the Clarke's subdifferential. 

The logistic loss is defined by $\ell(q) = \ln(1 + e^{-q})$, which satisfies
$\ell(0) = \ln 2$, $\ell'(0) = -1/2$, $\abssm{\ell'(q)} \le 1$,
$\abssm{\ell''(q)} \le 1$. Given a dataset $\DatS = \{(\vx_1, y_1), \dots,
(\vx_n, y_n)\}$, we consider gradient flow on two-layer Leaky ReLU network with
output function $f_{\vtheta}(\vx_i)$ and logistic loss $\Loss(\vtheta) :=
\frac{1}{n} \sum_{i \in [n]} \ell(q_i(\vtheta))$, where $q_i(\vtheta) := y_i
f_{\vtheta}(\vx_i)$. Following \citet{davis2018stochastic,lyu2020gradient}, we
say that a function $\vz(t) \in \R^D$ on an interval $I$ is an \textit{arc} if
$\vz$ is absolutely continuous on any compact subinterval of $I$. An arc
$\vtheta(t)$ is a trajectory of gradient flow on $\Loss$ if $\vtheta(t)$
satisfies the following gradient inclusion for a.e.~$t \ge 0$:
\[
	\frac{\dd \vtheta(t)}{\dd t} \in -\cpartial \Loss(\vtheta(t)).
\]
Let $\OmegaS$ be the set of parameter vectors $\vtheta = (\vw_1, \dots, \vw_m, a_1, \dots, a_m)$ so that $\dotpsm{\vw_k}{\vx_i} \ne 0$ for all $i \in [n], k \in [m]$, i.e., no activation function has zero input. For any $\vtheta \in \OmegaS$, $f_{\vtheta}(\vx_i)$ and $\Loss(\vtheta)$ are continuously differentiable at $\vtheta$, and the gradients are given by
\begin{align}
	\frac{\partial f_{\vtheta}(\vx)}{\partial \vw_k} &= a_k \phi'(\vw_k^{\top} \vx_i) \vx_i, & \frac{\partial f_{\vtheta}(\vx)}{\partial a_k} &= \phi(\vw_k^{\top} \vx_i). \label{eq:grad-f} \\
	\frac{\partial \Loss(\vtheta)}{\partial \vw_k} &= \frac{1}{n}\sum_{i \in [n]} \ell'(q_i(\vtheta)) y_i a_k \phi'(\vw_k^{\top} \vx_i) \vx_i, & \frac{\partial \Loss(\vtheta)}{\partial a_k} &= \frac{1}{n}\sum_{i \in [n]} \ell'(q_i(\vtheta)) y_i \phi(\vw_k^{\top} \vx_i). \label{eq:grad-Loss}
\end{align}
Then the Clarke's subdifferential for any $\vtheta$ can be computed from \eqref{eq:def-clarke2} with $\Omega = \OmegaS$ if needed.

Recall that $G$-function (\Cref{sec:sketch-phase-i}) is defined by
\[
G(\vw) := \frac{-\ell'(0)}{n}\sum_{i \in [n]} y_i \phi(\vw^{\top} \vx_i) = \frac{1}{2n}\sum_{i \in [n]} y_i \phi(\vw^{\top} \vx_i).
\]
Define $\tildeLoss(\vtheta)$ to the linear approximation of $\Loss(\vtheta)$:
\[
\tildeLoss(\vtheta) := \ell(0) - \sum_{k \in [m]} a_k G(\vw_k).
\]
For every $\vtheta_0 \in \R^D$, we define $\phitheta(\vtheta_0, t)$ to be the value of $\vtheta(t)$ for gradient flow on $\Loss(\vtheta)$ starting with $\vtheta(0) = \theta_0$. For every $\tildevtheta_0 \in \R^D$, we define $\phisg(\tildevtheta_0, t)$ to be the value of $\tildevtheta(t)$ for gradient flow on $\tildeLoss(\tildevtheta)$ starting with $\tildevtheta(0) = \tildevtheta_0$. In the case where the gradient flow trajectory may not be unique, we assign $\phitheta(\vtheta_0, \,\cdot\,)$ (or $\phisg(\tildevtheta_0, \,\cdot\,)$) by an arbitrary trajectory of gradient flow on $\Loss$ (or $\tildeLoss$) starting from $\vtheta_0$ (or $\tildevtheta_0$).

\subsection{\Gronwall's Inequality}

We frequently use \Gronwall's inequality in our analysis.

\begin{lemma}[\Gronwall's Inequality]
	Let $\alpha, \beta, u$ be real-valued functions defined on $[a, b)$. Suppose that $\beta, u$ are continuous and $\min\{\alpha, 0\}$ is integrable on every compact subinterval of $[a, b)$. If $\beta \ge 0$ and $u$ satisfies the following inequality for all $t \in [a, b)$:
	\[
	u(t) \le \alpha(t) + \int_{a}^{t} \beta(\tau) u(\tau) \dd \tau,
	\]
	then for all $t \in [a, b]$,
	\begin{equation} \label{eq:gron-1}
	u(t) \le \alpha(t) + \int_{a}^{t} \alpha(\tau) \beta(\tau) \exp\left(\int_{\tau}^{t} \beta(\tau') \dd \tau'\right) \dd \tau.
	\end{equation}
	Furthermore, if $\alpha$ is non-decreasing, then for all $t \in [a, b]$,
	\begin{equation} \label{eq:gron-2}
	u(t) \le \alpha(t) \exp\left(\int_{a}^{t} \beta(\tau) \dd \tau\right).
	\end{equation}
\end{lemma}

\subsection{Homogeneous Functions}

For $L \ge 0$, we say that a function $f: \R^d \to \R$ is (positively)
$L$-homogeneous if $f(c\vtheta) = c^L f(\vtheta)$ for all $c > 0$ and $\vtheta
\in \R^d$. The proof for the following two theorems can be found in
\citet[Theorem B.2]{lyu2020gradient} and \citet[Lemma C.1]{ji2020directional}
respectively.

\begin{theorem} \label{thm:homo-grad}
	For locally Lipschitz and $L$-homogeneous function $f: \R^d \to \R$, we have
	\[
	\cpartial f(c\vtheta) = c^{L-1} \cpartial f(\vtheta).
	\]
	for all $\vtheta \in \R^d$.
\end{theorem}

\begin{theorem}[Euler’s homogeneous function theorem] \label{thm:homo-euler}
	For locally Lipschitz and $L$-homogeneous function $f: \R^d \to \R$, we have
	\[
		\forall \vg \in \cpartial f(\vtheta): \quad \dotp{\vg}{\vtheta} = Lf(\vtheta),
	\]
	for all $\vtheta \in \R^d$.
\end{theorem}

For the maximizer of a homogeneous function on $\sphS^{d-1}$, we have the following useful lemma.

\begin{lemma} \label{lm:homo-max-grad}
	For locally Lipschitz and $L$-homogeneous function $f: \R^d \to \R$, if $\vtheta \in \sphS^{d-1}$ is a local/global maximizer of $f(\vtheta)$ on $\sphS^{d-1}$ and $f$ is differentiable at $\vtheta$, then $\nabla  f(\vtheta) = L f(\vtheta) \vtheta$.
\end{lemma}
\begin{proof}
	Since $\vtheta$ is a local/global maximizer of $f(\vtheta)$ on $\sphS^{d-1}$ and $f$ is differentiable at $\vtheta$, $\nabla f(\vtheta)$ is parallel to $\vtheta$, i.e., $\nabla f(\vtheta) = c\vtheta$ for some $c \in \R$. By \Cref{thm:homo-euler} we know that $\dotp{\nabla f(\vtheta)}{\vtheta} = L f(\vtheta)$. So $c = L f(\vtheta)$.
\end{proof}

The following is a direct corollary of \Cref{lm:homo-max-grad}.
\begin{lemma} \label{lm:G-opt-grad}
	If $\vw \in \sphS^{d-1}$ attains the maximum of $\abs{G(\vw)}$ on $\sphS^{d-1}$ and $G(\vw)$ is differentiable at $\vw$, then $\nabla G(\vw) = G(\vw) \vw$.
\end{lemma}
\begin{proof}
	Note that $G(\vw)$ is $1$-homogeneous. If $\vw$ attains the maximum of $\abs{G(\vw)}$ on $\sphS^{d-1}$, then $\vw$ is either a maximizer of $G(\vw)$ or $-G(\vw)$. Applying \Cref{lm:homo-max-grad} gives $\nabla G(\vw) = G(\vw) \vw$.
\end{proof}

\subsection{Karush-Kuhn-Tucker Conditions for Margin Maximization}

\begin{definition}[Feasible Point and KKT Point, \citealt{dutta2013approximate,lyu2020gradient}]
	Let $f, g_1, \dots, g_n: \R^D \to \R$ be locally Lipschitz functions. Consider the following constrained optimization problem for $\vtheta \in \R^D$:
	\begin{align*}
		\min          & \quad f(\vtheta) \\
		\text{s.t.}   & \quad g_i(\vtheta) \le 0, \qquad \forall i \in [n].
	\end{align*}
	We say that $\vtheta$ is a \textit{feasible point} if $g_i(\vtheta) \le 0$ for all $i \in [n]$. A feasible point $\vtheta$ is a \textit{KKT point} if it satisfies Karush-Kuhn-Tucker Conditions: there exist $\lambda_1, \dots, \lambda_n \ge 0$ such that
	\begin{enumerate}
		\item $\vzero \in \cpartial f(\vtheta) + \sum_{i \in [n]} \lambda_i \cpartial g_i(\vtheta)$;
		\item $\forall i \in [n]: \lambda_i g_i(\vtheta) = 0$.
	\end{enumerate} 
\end{definition}

Recall that we say that a parameter vector $\vtheta \in \R^D$ of a
$L$-homogeneous network is along a KKT-margin direction if
$\frac{\vtheta}{(\qmin(\vtheta))^{1/L}}$ is a KKT point of \eqref{eq:prob-P},
where $f(\vtheta) = \frac{1}{2}\normtwosm{\vtheta}^2$ and $g_i(\vtheta) = 1 -
q_i(\vtheta)$. Alternatively, we can use the following equivalent definition.
\begin{definition}[KKT-margin Direction for Homogeneous Network, \citealt{lyu2020gradient}] \label{def:kkt-margin}
	For a parameter vector $\vtheta \in \R^D$ of a homogeneous network, we say $\vtheta$ is along a KKT-margin direction if $q_i(\vtheta) > 0$ for all $i \in [n]$ and there exist $\lambda_1, \dots, \lambda_n \ge 0$ such that
	\begin{enumerate}
		\item $\vtheta \in \sum_{i \in [n]} \lambda_i \cpartial q_i(\vtheta)$;
		\item For all $i \in [n]$, if $q_i(\vtheta) \ne \qmin(\vtheta)$ then $\lambda_i = 0$.
	\end{enumerate}
\end{definition}

For two-layer Leaky ReLU network, $q_i(\vtheta) := y_i \sum_{k \in [m]} a_k \phi(\vw_k^{\top} \vx_i)$. Then the KKT-margin direction is defined as follows.
\begin{definition}[KKT-margin Direction for Two-layer Leaky ReLU Network] \label{def:kkt-margin-lkrelu}
	For a parameter vector $\vtheta = (\vw_1, \dots, \vw_m, a_1, \dots, a_m) \in \R^D$ of a two-layer Leaky ReLU network, we say $\vtheta$ is along a KKT-margin direction if $q_i(\vtheta) > 0$ for all $i \in [n]$ and there exist $\lambda_1, \dots, \lambda_n \ge 0$ such that
	\begin{enumerate}
		\item For all $k \in [m]$, $\vw_k \in \sum_{i \in [n]} \lambda_i y_i a_k \cder{\phi}(\vw_k^{\top}\vx_i) \vx_i$;
		\item For all $k \in [m]$, $a_k = \sum_{i \in [n]} \lambda_i y_i \phi(\vw_k^{\top}\vx_i)$;
		\item For all $i \in [n]$, if $q_i(\vtheta) \ne \qmin(\vtheta)$ then $\lambda_i = 0$.
	\end{enumerate}
\end{definition}
For $\vtheta$ along a KKT-margin direction of two-layer Leaky ReLU network, \Cref{lm:a-w-equal} below shows that $\abssm{a_k} = \normtwosm{\vw_k}$ for all $k \in [m]$.
\begin{lemma} \label{lm:a-w-equal}
	If $\vtheta = (\vw_1, \dots, \vw_m, a_1, \dots, a_m) \in \R^D$ is along a KKT-margin direction of a two-layer Leaky ReLU network, then $\abssm{a_k} = \normtwosm{\vw_k}$ for all $k \in [m]$.
\end{lemma}
\begin{proof}
	By \Cref{def:kkt-margin-lkrelu} and \Cref{thm:homo-euler}, we have
	\begin{align*}
		\normtwosm{\vw_k}^2 &\in \dotp{\vw_k}{\sum_{i \in [n]} \lambda_i y_i a_k \cder{\phi}(\vw_k^{\top}\vx_i) \vx_i} = \left\{\sum_{i \in [n]} \lambda_i y_i a_k \phi(\vw_k^{\top}\vx_i)\right\}, \\
		\abssm{a_k}^2 &= a_k \cdot \sum_{i \in [n]} \lambda_i y_i \phi(\vw_k^{\top}\vx_i) = \sum_{i \in [n]} \lambda_i y_i a_k \phi(\vw_k^{\top}\vx_i).
	\end{align*}
	Therefore $\normtwosm{\vw_k}^2 = \abssm{a_k}^2$.
\end{proof}

\subsection{Lemmas for Perturbation Bounds}

Recall that $\normMsm{\vtheta}$ is defined in \Cref{def:norm-M}.

\begin{lemma} \label{lm:f-output-ub}
	For $\normtwosm{\vx} \le 1$, $\abssm{f_{\vtheta}(\vx)} \le m \normMsm{\vtheta}^2$, $\abssm{f_{\vtheta}(\vx) - f_{\tildevtheta}(\vx)} \le m\normMsm{\vtheta - \tildevtheta}\left(\normMsm{\vtheta} + \normMsm{\tildevtheta}\right)$.
\end{lemma}
\begin{proof}
	The proof is straightforward by definition of $f_{\vtheta}(\vx)$ and $\normMsm{\vtheta}$. For the first inequality,
	\[
	\abssm{f_{\vtheta}(\vx)}
	\le \sum_{k=1}^{m} \abssm{a_k \phi(\vw_k^{\top} \vx)}
	\le \sum_{k=1}^{m} \abssm{a_k} \cdot \abssm{\vw_k^{\top} \vx}
	\le \sum_{k=1}^{m} \abssm{a_k} \cdot \normtwosm{\vw_k} \le m\normMsm{\vtheta}^2.
	\]
	For the second inequality,
	\begin{align*}
	\abssm{f_{\vtheta}(\vx) - f_{\tildevtheta}(\vx)}
	&\le \sum_{k=1}^{m} \abssm{a_k \phi(\vw_k^{\top} \vx) - \tildea_k \phi(\tildevw_k^{\top} \vx)} \\
	&\le \sum_{k=1}^{m} \abssm{a_k \phi(\vw_k^{\top} \vx) - a_k \phi(\tildevw_k^{\top} \vx)} + \abssm{a_k \phi(\tildevw_k^{\top} \vx) - \tildea_k \phi(\tildevw_k^{\top} \vx)} \\
	&\le \sum_{k=1}^{m} \abssm{a_k} \cdot \normtwosm{\vw_k - \tildevw_k} + \abssm{a_k - \tildea_k} \cdot \normtwosm{\tildevw_k} \\
	&\le m\normMsm{\vtheta - \tildevtheta}\left(\normMsm{\vtheta} + \normMsm{\tildevtheta}\right),
	\end{align*}
	which completes the proof.
\end{proof}

We have the following bound for the difference between $\cpartial \Loss(\vtheta)$ and $\cpartial \tildeLoss(\vtheta)$.
\begin{lemma} \label{lm:diff-Loss-tildeLoss-partial} Assume that
	$\normtwosm{\vx_i} \le 1$ for all $i \in [n]$. For any $\vtheta = (\vw_1,
	\dots, \vw_m, a_1, \dots, a_m) \in \R^D$, we have the following bounds for
	the partial derivatives of $\Loss(\vtheta)-\tildeLoss(\vtheta)$:
	\begin{align*}
		\normtwo{\frac{\cpartial(\Loss(\vtheta)-\tildeLoss(\vtheta))}{\partial \vw_k}} &\subseteq \left(-\infty, m\normMsm{\vtheta}^2 \abssm{a_k}\right], & \abs{\frac{\partial(\Loss(\vtheta)-\tildeLoss(\vtheta))}{\partial a_k}} &\le m\normMsm{\vtheta}^2 \normtwosm{\vw_k}.
	\end{align*}
	for all $k \in [m]$.
\end{lemma}
\begin{proof}
	We only need to prove the following bounds for gradients at any $\vtheta \in \OmegaS$, i.e., $\dotp{\vw_k}{\vx_i} \ne 0$ for all $i \in [n], k\in[m]$. For the general case where $\vtheta$ can be non-differentiable, we can prove the same bounds for Clarke's sub-differential at every point $\vtheta \in \R^D$ by taking limits in $\OmegaS$ through \eqref{eq:def-clarke2}.
	\begin{align*}
	\normtwo{\frac{\partial(\Loss(\vtheta)-\tildeLoss(\vtheta))}{\partial \vw_k}} &\le m\normMsm{\vtheta}^2 \abssm{a_k}, & \abs{\frac{\partial(\Loss(\vtheta)-\tildeLoss(\vtheta))}{\partial a_k}} &\le m\normMsm{\vtheta}^2 \normtwosm{\vw_k}.
	\end{align*}
	By Taylor expansion, we have
	\[
		\ell(y_i f_{\vtheta}(\vx_i)) = \ell(0) + \ell'(0) y_if_{\vtheta}(\vx_i) + \int_{0}^{y_if_{\vtheta}(\vx_i)} \ell''(z) (y_if_{\vtheta}(\vx_i) - z) \dd z.
	\]
	Taking average over $i \in [n]$ gives
	\begin{align*}
	\Loss(\vtheta) &= \ell(0) + \frac{1}{n} \sum_{i \in [n]} \ell'(0) y_if_{\vtheta}(\vx_i) + \frac{1}{n} \sum_{i \in [n]}\int_{0}^{y_if_{\vtheta}(\vx_i)} \ell''(z) (y_if_{\vtheta}(\vx_i) - z) \dd z \\
	&= \tildeLoss(\vtheta) + \frac{1}{n} \sum_{i \in [n]}\int_{0}^{y_if_{\vtheta}(\vx_i)} \ell''(z) (y_if_{\vtheta}(\vx_i) - z) \dd z.
	\end{align*}
	By Leibniz integral rule,
	\begin{align*}
	\nabla_{\vtheta}\left(\Loss(\vtheta)-\tildeLoss(\vtheta)\right) &= \nabla_{\vtheta}\left(\frac{1}{n} \sum_{i \in [n]}\int_{0}^{y_if_{\vtheta}(\vx_i)} \ell''(z) (y_if_{\vtheta}(\vx_i) - z) \dd z\right)\\
	&= -\frac{1}{n} \sum_{i \in [n]}\int_{0}^{y_if_{\vtheta}(\vx_i)} \ell''(z) y_i \nabla_{\vtheta}(f_{\vtheta}(\vx_i)) \dd z \\
	&= -\frac{1}{n} \sum_{i \in [n]} \left(\int_{0}^{y_if_{\vtheta}(\vx_i)} \ell''(z) \dd z \right) y_i \nabla_{\vtheta}(f_{\vtheta}(\vx_i)).
	\end{align*}
	Since $\ell''(z) \le 1$, there exists $\delta_i \in [-\abssm{f_{\vtheta}(\vx_i)}, \abssm{f_{\vtheta}(\vx_i)}]$ for all $i \in [n]$ such that
	\begin{equation} \label{eq:gradient_coupling}
	\nabla_{\vtheta}\left(\Loss(\vtheta)-\tildeLoss(\vtheta)\right) = \frac{1}{n} \sum_{i \in [n]} \delta_i \nabla_{\vtheta}(f_{\vtheta}(\vx_i)).
	\end{equation}
	Writing the formula with respect to $\vw_k, a_k$, we have
	\begin{align*}
	\normtwo{\frac{\partial(\Loss(\vtheta)-\tildeLoss(\vtheta))}{\partial \vw_k}}
	&\le \frac{1}{n} \sum_{i \in [n]} \abssm{\delta_i} \cdot \normtwo{\frac{\partial f_{\vtheta}(\vx_i)}{\partial \vw_k}}
	\le \frac{1}{n} \sum_{i \in [n]} \abssm{f_{\vtheta}(\vx_i)} \cdot \normtwo{a_k \phi'(\dotp{\vw_k}{\vx_i}) \vx_i}. \\
	\abs{\frac{\partial(\Loss(\vtheta)-\tildeLoss(\vtheta))}{\partial a_k}}
	&\le \frac{1}{n} \sum_{i \in [n]} \abssm{\delta_i} \cdot \normtwo{\frac{\partial f_{\vtheta}(\vx_i)}{\partial a_k}}
	\le \frac{1}{n} \sum_{i \in [n]} \abssm{f_{\vtheta}(\vx_i)} \cdot \abs{\phi(\dotp{\vw_k}{\vx_i})}.
	\end{align*}
	By \Cref{lm:f-output-ub}, $\abssm{f_{\vtheta}(\vx_i)} \le m \normMsm{\vtheta}^2$. Since Leaky ReLU is $1$-Lipschitz and $\normtwosm{\vx_i} \le 1$, we have $\normtwo{a_k \phi'(\dotp{\vw_k}{\vx_i}) \vx_i} \le \abssm{a_k}$, $\abs{\phi(\dotp{\vw_k}{\vx_i})} \le \normtwosm{\vw_k}$. Then we have
	\begin{align*}
	\normtwo{\frac{\partial(\Loss(\vtheta)-\tildeLoss(\vtheta))}{\partial \vw_k}}
	&\le \frac{1}{n} \sum_{i \in [n]} m \normMsm{\vtheta}^2 \cdot \abssm{a_k} = m \normMsm{\vtheta}^2 \cdot \abssm{a_k}, \\
	\abs{\frac{\partial(\Loss(\vtheta)-\tildeLoss(\vtheta))}{\partial a_k}}
	&\le \frac{1}{n} \sum_{i \in [n]} m \normMsm{\vtheta}^2 \cdot \normtwosm{\vw_k} = m \normMsm{\vtheta}^2 \cdot \normtwosm{\vw_k},
	\end{align*}
	which completes the proof for $\vtheta \in \OmegaS$ and thus the same bounds hold for the general case.
\end{proof}
\Cref{lm:diff-Loss-tildeLoss-partial} is a lemma for bounding the partial subderivatives. For the full subgradient, we have the following lemma.
\begin{lemma} \label{lm:diff-Loss-tildeLoss}
	Assume that $\normtwosm{\vx_i} \le 1$ for all $i \in [n]$. For any $\vtheta \in \R^D$, we have
	\[
		\forall \vg \in \cpartial\!\left(\Loss(\vtheta) - \tildeLoss(\vtheta)\right): \qquad \normM{\vg} \le m \normMsm{\vtheta}^3.
	\]
\end{lemma}
\begin{proof}
	Note that $\abssm{a_k} \le \normMsm{\vtheta}$ and $\normtwosm{\vw_k} \le \normMsm{\vtheta}$. Combining this with \Cref{lm:diff-Loss-tildeLoss-partial} gives $\normMsm{\cpartial \Loss(\vtheta)} \subseteq (-\infty, m\normMsm{\vtheta}^3]$.
\end{proof}

When $\tildeLoss(\vtheta)$ is smooth, we have the following direct corollary.
\begin{corollary} \label{cor:diff-Loss-tildeLoss-smooth}
	Assume that $\normtwosm{\vx_i} \le 1$ for all $i \in [n]$. If $\tildeLoss$ is continuously differentiable at $\vtheta \in \R^D$, then we have
	\[
	\forall \vg \in \left(\cpartial \Loss(\vtheta) - \nabla\tildeLoss(\vtheta)\right): \qquad \normM{\vg} \le m \normMsm{\vtheta}^3.
	\]
\end{corollary}
Note that $\cpartial(\Loss(\vtheta) - \tildeLoss(\vtheta)) \ne \cpartial \Loss(\vtheta) - \nabla\tildeLoss(\vtheta)$ because the exact sum rule does not hold for Clarke's subdifferential when $\tildeLoss(\vtheta)$ is not smooth. In the non-smooth case, we have the following lemma:
\begin{lemma}\label{lem:non_sym_eps}
	Assume that $\normtwosm{\vx_i} \le 1$ for all $i \in [n]$. For any $\eps>0$ and $\normM{\vtheta}\le \sqrt{\frac{\eps}{2m}}$, we have
	\[
		\forall k\in [m],\quad   \frac{\cpartial \Loss(\vtheta)}{\partial \vw_k} \subseteq \left\{ -\frac{a_k}{2n}\sum_{i=1}^n (1+\eps_i)\alpha_i y_i\vx_i : \alpha_i\in \cder{\phi}(\vw_k^{\top} \vx_i), \eps_i \in [-\eps,\eps], \forall i\in[n]\right\}.
	\]
\end{lemma}
\begin{proof}
	If $\vtheta \in \OmegaS$, by \eqref{eq:gradient_coupling}, there exists $\delta_i \in [-\abssm{f_{\vtheta}(\vx_i)}, \abssm{f_{\vtheta}(\vx_i)}]$ for all $i \in [n]$ such that
	\[
		\nabla \Loss(\vtheta) = \nabla\tildeLoss(\vtheta) + \frac{1}{n} \sum_{i \in [n]} \delta_i \nabla_{\vtheta}(f_{\vtheta}(\vx_i)).
	\]
	Writing it with respect to $\vw_k$, we have
	\begin{align*}
		\frac{\partial \Loss}{\partial \vw_k}
		&= -a_k \frac{\partial G(\vw_k)}{\partial \vw_k} + \frac{1}{n} \sum_{i \in [n]} \delta_i \frac{\partial f_{\vtheta}(\vx_i)}{\partial \vw_k} \\
		&= -\frac{a_k}{2n} \sum_{i \in [n]} y_i \phi'(\vw_k^\top \vx_i) \vx_i + \frac{1}{n} \sum_{i \in [n]} \delta_i a_k \phi'(\vw_k^{\top} \vx_i) \vx_i \\
		&= -\frac{a_k}{2n} \sum_{i \in [n]} y_i(1 - 2y_i\delta_i) \phi'(\vw_k^{\top} \vx_i) \vx_i.
	\end{align*}
	Regarding Clarke's subdifferential at a point $\vtheta \in \R^D$, we can take limits in a neighborhood of $\vtheta$ in $\OmegaS$ through \eqref{eq:def-clarke2}, then
	\[
		\frac{\cpartial \Loss}{\partial \vw_k} \subseteq \left\{-\frac{a_k}{2n} \sum_{i \in [n]} y_i(1+\eps_i) \alpha_i \vx_i : \alpha_i \in \cder{\phi}(\vw_k^{\top} \vx_i), \eps_i \in [-2\abssm{f_{\vtheta}(\vx_i)}, 2\abssm{f_{\vtheta}(\vx_i)}], \forall i \in [n]\right\}.
	\]
	We conclude the proof by noticing that $[-2\abssm{f_{\vtheta}(\vx_i)}, 2\abssm{f_{\vtheta}(\vx_i)}] \subseteq [-\eps, \eps]$ by \Cref{lm:f-output-ub}.
\end{proof}

\subsection{Basic Properties of Gradient Flow}

The following lemma is a simple corollary from \citet{davis2018stochastic}. 
\begin{lemma} \label{lm:descent}
	For gradient flow  $\vtheta(t)$ on a two-layer Leaky ReLU network with logistic loss, we have
	\[
	\frac{\dd \vtheta(t)}{\dd t} = -\barcpartial \Loss(\vtheta(t)), \qquad \frac{\dd \Loss(\vtheta(t))}{\dd t} = -\normtwo{\frac{\dd \vtheta(t)}{\dd t}}^2
	\]
	for a.e. $t \ge 0$.
\end{lemma}

The following lemma is from \citet{du2018algorithmic}. We provide a simple proof here for completeness.
\begin{lemma} \label{lm:weight-balance}
	For gradient flow $\vtheta(t) = (\vw_1(t), \dots, \vw_m(t), a_1(t), \dots, a_m(t))$ on a two-layer Leaky ReLU network with logistic loss, the following holds for all $t \ge 0$,
	\[
		\frac{1}{2}\frac{\dd \normtwosm{\vw_k}^2}{\dd t} = \frac{1}{2}\frac{\dd \abssm{a_k}^2}{\dd t} = -\frac{1}{n} \sum_{i=1}^{n} \ell'(q_i(\vtheta)) y_i a_k \phi(\vw_k^{\top} \vx_i),
	\]
	where $q_i(\vtheta) := y_i f_{\vtheta}(\vx_i)$. Therefore, $\frac{\dd}{\dd t}(\normtwosm{\vw_k}^2 - \abssm{a_k}^2) = 0$ for all $t \ge 0$.
\end{lemma}
\begin{proof}
	By \eqref{eq:grad-f}, we have the following for any $\vtheta \in \OmegaS$,
	\begin{align*}
		a_k \cdot \frac{\partial f_{\vtheta}(\vx)}{\partial a_k} &= a_k\phi(\vw_k^{\top} \vx_i), & \dotp{\vw_k}{\frac{\partial f_{\vtheta}(\vx)}{\partial \vw_k}} &= a_k \phi'(\vw_k^{\top} \vx_i) \vw_k^{\top}\vx_i.
	\end{align*}
	By $1$-homogeneity of $\phi$ and \Cref{thm:homo-euler}, we have $\phi'(\vw_k^{\top} \vx_i) \vw_k^{\top}\vx_i = \phi(\vw_k^{\top} \vx_i)$, which implies that $\dotp{\vw_k}{\frac{\partial f_{\vtheta}(\vx)}{\partial \vw_k}} = a_k \phi(\vw_k^{\top} \vx_i)$.
	
	For any $\vtheta \in \R^D$, we can take limits in $\OmegaS$ through \eqref{eq:def-clarke2} to show that the same equation holds in general.
	\begin{align*}
	a_k \cdot \frac{\cpartial f_{\vtheta}(\vx)}{\partial a_k} = \dotp{\vw_k}{\frac{\cpartial f_{\vtheta}(\vx)}{\partial \vw_k}} &= \left\{a_k \phi(\vw_k^{\top} \vx_i)\right\}.
	\end{align*}
	By chain rule, for a.e. $t \ge 0$ we have
	\begin{align*}
		\frac{1}{2}\frac{\dd \abssm{a_k}^2}{\dd t} = \frac{\dd a_k}{\dd t} \cdot a_k &\in -\frac{1}{n} \sum_{i=1}^{n} \ell'(q_i(\vtheta)) y_i \frac{\cpartial f_{\vtheta}(\vx_i)}{\partial a_k} \cdot a_k. \\
		\frac{1}{2}\frac{\dd \normtwosm{\vw_k}^2}{\dd t} = \dotp{\frac{\dd \vw_k}{\dd t}}{\vw_k} &\in -\frac{1}{n} \sum_{i=1}^{n} \ell'(q_i(\vtheta))  y_i\dotp{\frac{\cpartial f_{\vtheta}(\vx_i)}{\partial \vw_k}}{\vw_k}.
	\end{align*}
	Therefore we have
	\[
		\frac{1}{2}\frac{\dd \abssm{a_k}^2}{\dd t} = \frac{1}{2}\frac{\dd \normtwosm{\vw_k}^2}{\dd t} = -\frac{1}{n} \sum_{i=1}^{n} \ell'(q_i(\vtheta))  y_i a_k \phi(\vw_k^{\top} \vx_i),
	\]
	for a.e. $t \ge 0$. Note that $-\frac{1}{n} \sum_{i=1}^{n} \ell'(q_i(\vtheta)) y_i a_k \phi(\vw_k^{\top} \vx_i)$ is continuous in $\vtheta$ and thus continuous in time $t$. This means we can further deduce that this equation holds for all $t \ge 0$. This automatically proves that $\frac{\dd}{\dd t}(\normtwosm{\vw_k}^2 - \abssm{a_k}^2) = 0$.
\end{proof}

The following lemma shows that if a neuron has zero weights, then it stays with
zero weights forever. Conversely, this also implies that the weights stay
non-zero if they are initially non-zero.
\begin{lemma} \label{lm:weight-zero}
	If $a_k(t_0) = 0$ and $\vw_k(t_0) = \vzero$ at some time $t_0 \ge 0$, then $a_k(t) = 0$ and $\vw_k(t) = \vzero$ for all $t \ge 0$.
\end{lemma}
\begin{proof}
	By \Cref{lm:weight-balance}, we know that $\normtwosm{\vw_k} = \abssm{a_k}$ hold for all $t \ge 0$. Also, we have $\frac{1}{2}\abs{\frac{\dd \normtwosm{\vw_k}^2}{\dd t}} = \frac{1}{2}\abs{\frac{\dd \abssm{a_k}^2}{\dd t}} \le C \cdot \abssm{a_k} \normtwosm{\vw_k} = C \normtwosm{\vw_k}^2$, where $C > 0$ is some constant. Then
	\[
		\normtwosm{\vw_k(t)}^2 \le \normtwosm{\vw_k(t_0)}^2 + \int_{t_0}^{t} 2C \normtwosm{\vw_k(\tau)}^2 \dd \tau.
	\]
	By \Gronwall's inequality \eqref{eq:gron-2} this implies that $\normtwosm{\vw_k(t)} = 0$ for all $t \ge t_0$. Similarly,
	\[
		\normtwosm{\vw_k(t)}^2 \le \normtwosm{\vw_k(t_0)}^2 + \int_{t}^{t_0} 2C \normtwosm{\vw_k(\tau)}^2 \dd \tau.
	\]
	By \Gronwall's inequality \eqref{eq:gron-2} again, $\normtwosm{\vw_k(t)} = 0$ for all $t \le t_0$, which completes the proof.
\end{proof}

A direct corollary of \Cref{lm:weight-balance} and \Cref{lm:weight-zero} is the
following characterization in the case where the weights are initially balanced.
\begin{corollary} \label{cor:weight-equal}
	If $\abs{a_k} = \normtwosm{\vw_k}$ initially for $t = 0$, then this equation holds for all $t \ge 0$. Moreover,
	\begin{enumerate}
		\item If $a_k(0) = \normtwosm{\vw_k(0)}$, then $a_k(t) = \normtwosm{\vw_k(t)}$ for all $t \ge 0$;
		\item If $a_k(0) = -\normtwosm{\vw_k(0)}$, then $a_k(t) = -\normtwosm{\vw_k(t)}$ for all $t \ge 0$.
	\end{enumerate}
	
\end{corollary}

\subsection{A Useful Theorem for Loss Convergence}
In this section we prove a useful theorem for loss convergence, which will be
used later in our analysis for both symmetric and non-symmetric datasets.
\begin{theorem} \label{thm:loss-convergence}
	Under \Cref{ass:lin}, for any linear seprator $\vwopt$ of the data with positive linear margin (e.g. $y_i\dotp{\vwopt}{x_i}\geq\gamma^*>0$ for all $i\in [n]$), if initially there exists $k \in [m]$ such that
	\[
	\sgn(a_k(0))\dotp{\vw_k(0)}{\vwopt} > 0, \qquad \dotp{\vw_k(0)}{\vwopt}^2 > \normtwosm{\vw_k(0)}^2 - \abs{a_k(0)}^2,
	\]
	then $a_k(t)\neq 0$ for all $t>0$, and $\Loss(\vtheta(t)) \to 0$ and $\normtwosm{\vtheta(t)} \to +\infty$ as $t \to +\infty$.
\end{theorem}
Before proving \Cref{thm:loss-convergence}, we first prove a lemma on gradient lower bounds.
\begin{lemma} \label{lm:dotp-vw-vwopt}
	For a.e.~$t \ge 0$,
	\begin{align*}
	\dotp{\sgn(a_k) \frac{\dd \vw_k}{\dd t}}{\vwopt} &\ge \abssm{a_k} \alphaLK \gammaopt \cdot \frac{1 - \exp(-n\Loss)}{n}.
	\end{align*}
\end{lemma}
\begin{proof}
	By \eqref{eq:grad-Loss}, there exist $\hki[1](t), \dots, \hki[n](t) \in [\alphaLK, 1]$ such that
	\[
		\frac{\dd \vw_k}{\dd t} = \frac{a_k}{n}\sum_{i \in [n]} g_i(\vtheta(t)) \hki(t) y_i\vx_i
	\]
	where $g_i(\vtheta(t))=-\ell'(y_if_{\vtheta(t)}(\vx_i))>0$. Then we have
	\[
		\dotp{\sgn(a_k) \frac{\dd \vw_k}{\dd t}}{\vwopt} \ge \frac{\abssm{a_k}}{n}\sum_{i \in [n]} g_i(\vtheta(t)) \alphaLK \gammaopt.
	\]
	Note that $-\ell'(q) = \frac{1}{1 + e^q} = 1 - \frac{1}{1+e^{-q}} = 1 - \exp(-\ell(q))$ for all $q$. So we have the following lower bound for $\sum_{i \in [n]} g_i(\vtheta(t))$:
	\begin{align*}
	\sum_{i \in [n]} g_i(\vtheta(t)) = \sum_{i \in [n]} -\ell'(y_if_{\vtheta}(\vx_i)) &= \sum_{i \in [n]} \left(1-\exp(-\ell(y_if_{\vtheta}(\vx_i)))\right) \\
	&\ge \max_{i \in [n]} \left(1-\exp(-\ell(y_if_{\vtheta}(\vx_i)))\right) \\
	&\ge 1 - \exp\left(-\max_{i \in [n]}\ell(y_if_{\vtheta}(\vx_i))\right) \\
	&\ge 1 - \exp(-n\Loss).
	\end{align*}
	Therefore,
	\[
		\dotp{\sgn(a_k)\frac{\dd \vw_k}{\dd t}}{\vwopt} \ge \abssm{a_k} \alphaLK \gammaopt \cdot \frac{1 - \exp(-n\Loss)}{n},
	\]
	which completes the proof.
\end{proof}

\begin{proof}[Proof for \Cref{thm:loss-convergence}]
	We only need to show that there exists $t_0$ such that $\Loss(\vtheta(t_0))
	< \frac{\ln 2}{n}$, then we can apply \Cref{thm:converge-kkt-margin} to show
	that $\Loss(\vtheta(t)) \to 0$. Assume to the contrary that
	$\Loss(\vtheta(t)) \ge \frac{\ln 2}{n}$ for all $t \ge 0$. By
	\Cref{lm:dotp-vw-vwopt},
	\[
		\sgn(a_k)\dotp{\frac{\dd \vw_k}{\dd t}}{\vwopt}  \ge \abssm{a_k} \cdot \frac{\alphaLK \gammaopt}{2n}.
	\]
	Let $c := \dotp{\vw_k(0)}{\vwopt}^2 - \normtwosm{\vw_k(0)}^2 + \abs{a_k(0)}^2>0$. First we show that $\sgn(a_k(t))= \sgn(a_k(0))$ for all $t>0$. Otherwise let $\tmpts :=\inf\{t:\sgn(a_k(t))\neq\sgn(a_k(0))\}$, and since $a_k(t)$ is continuous, $a_k(\tmpts)=0$. We know for $t\in [0,\tmpts]$, $\sgn(a_k(0))\frac{\dd}{\dd t} \dotp{\vw_k(t)}{\vwopt}>0$,  and
	\begin{align*}
		\abs{a_k(\tmpts)}^2 = \abs{a_k(0)}^2 - \normtwosm{\vw_k(0)}^2 + \normtwosm{\vw_k(\tmpts)}^2 &\ge \abs{a_k(0)}^2 - \normtwosm{\vw_k(0)}^2 + \dotp{\vw_k(\tmpts)}{\vwopt}^2 \\
		&> \abs{a_k(0)}^2 - \normtwosm{\vw_k(0)}^2 + \dotp{\vw_k(0)}{\vwopt}^2 = c > 0.
	\end{align*}
	This contradicts to the fact that $a_k(\tmpts)=0$, and thus $\sgn(a_k(t))$ does not change during all time. Therefore for any $t>0$, $a_k(t)\neq 0$. Then for all $t>0$,
	\begin{align*}
		\abs{a_k(t)}^2 = \abs{a_k(0)}^2 - \normtwosm{\vw_k(0)}^2 + \normtwosm{\vw_k(t)}^2 &\ge \abs{a_k(0)}^2 - \normtwosm{\vw_k(0)}^2 + \dotp{\vw_k(t)}{\vwopt}^2 \\
		&> \abs{a_k(0)}^2 - \normtwosm{\vw_k(0)}^2 + \dotp{\vw_k(0)}{\vwopt}^2 = c.
	\end{align*}
	\Cref{lm:descent} ensures that $-\frac{\dd \Loss}{\dd t} = \normtwo{\frac{\dd\vtheta}{\dd t}}^2$ for a.e.~$t\ge 0$. Then we have
	\[
	-\frac{\dd \Loss}{\dd t} \ge \normtwo{\frac{\dd\vw_k}{\dd t}}^2 \ge \dotp{\frac{\dd\vw_k}{\dd t}}{\vwopt}^2 \ge \abssm{a_k}^2 \left(\frac{\alphaLK \gammaopt}{2n}\right)^2 \ge c^2 \cdot \left(\frac{\alphaLK \gammaopt}{2n}\right)^2.
	\]
	Then we can conclude that
	\[
		\Loss(\vtheta(0)) - \Loss(\vtheta(t)) \ge c^2 \left(\frac{\alphaLK \gammaopt}{2n}\right)^2t.
	\]
	Integrating on $t$ from $0$ to $+\infty$, we can see that the LHS is upper bounded by $\Loss(\vtheta(0)) - \frac{\ln 2}{n}$ while the RHS is unbounded, which leads to a contradiction. Therefore, there exist time $t_0$ such that $\Loss(\vtheta(t_0)) < \frac{\ln 2}{n}$, and thus $\Loss(\vtheta(t)) \to 0$ as $t \to +\infty$.
\end{proof}

\section{Proofs for Linear Maximality for the Symmetric Case}
For linearly separable and symmetric data, we show that all global-max-margin directions represent linear functions in \Cref{thm:maxmar-linear}. We give a proof here.

\begin{proof}[Proof for \Cref{thm:maxmar-linear}]
    Let $\vthetaopt= (\vw_1, \dots, \vw_m, a_1, \dots, a_m) \in \sphS^{D-1}$ be any
    global-max-margin direction with output
    margin $\qmin(\vthetaopt) = \gamma(\vthetaopt)$. As the dataset is symmetric,
    \[
        \gamma(\vthetaopt) = \min_{i \in [n]}\{y_if_{\vthetaopt}(\vx_i),-y_if_{\vthetaopt}(-\vx_i)\}.
    \]
    Now we define $A := \sqrt{\sum_{k \in [m]} a_k^2}$ and let $\vtheta'=(\vw'_1, \dots, \vw'_m, a'_1, \dots, a'_m)$ where
    \begin{align*}
        \vw'_1 &= \frac{1}{\sqrt{2} A}\sum_{k \in [m]} a_k \vw_k, & \vw'_2&=-\vw'_1, &  a'_1&=\frac{A}{\sqrt{2}}, & a'_2&=-a'_1,
    \end{align*}
    and $a'_k = 0, \vw'_k = \vzero$ for $k > 2$. We claim that $\gamma(\vtheta') \ge \gamma(\vthetaopt)$. First we prove that $q_i(\vtheta') \ge \gamma(\vthetaopt)$ by repeatedly applying $\phi(z) - \phi(-z) = (1+\alphaLK)z$.
    \begin{align*}
        q_i(\vtheta') = y_i f_{\vtheta'}(\vx_i) & =  y_i \left(a'_1\phi(\dotpsm{{\vw}'_1}{\vx_i})+a'_2\phi(\dotpsm{{\vw}'_2}{\vx_i}) \right)\\
        & = y_i (1+\alphaLK) a'_1 \dotp{{\vw}'_1}{\vx_i}\\
        & = \frac{y_i}{2} \sum_{k \in [m]} \dotp{ (1+\alphaLK) a_k\vw_k}{\vx_i} \\
        & = \frac{y_i}{2} \sum_{k \in [m]} \left(a_k\phi({{\vw}_k}^{\top} \vx_i)-a_k\phi(-{{\vw}_k}^{\top} \vx_i)\right) \\
        & = \frac{1}{2}(y_if_{\vthetaopt}(\vx_i)-y_if_{\vthetaopt}(-\vx_i))\geq \gamma(\vthetaopt).
    \end{align*}
    Meanwhile, by the Cauchy-Schwarz inequality,
    \[
    	\normtwosm{\vtheta'}^2= A^2 + \normtwo{\sum_{k \in [m]} \frac{a_k}{A} \vw_k}^2\le A^2 + \sum_{k \in [m]} \left(\frac{a_k}{A}\right)^2 \cdot \sum_{k \in [m]} \normtwosm{\vw_k}^2= A^2 + \sum_{k \in [m]} \normtwosm{\vw_k}^2 = \normtwosm{\vthetaopt}^2.
    \]
    Thus $\gamma(\vtheta') = \frac{\qmin(\vtheta')}{\normtwosm{\vtheta'}^2} \ge \gamma(\vthetaopt)$. As $\vthetaopt$ is already a global-max-margin direction, equalities should hold in all the inequalities above, so
    \[
    	\min_{i \in [n]}\{y_if_{\vthetaopt}(\vx_i)-y_if_{\vthetaopt}(-\vx_i)\}=2\gamma(\vthetaopt), \qquad 
    	\normtwo{\sum_{k \in [m]} \frac{a_k}{A} \vw_k}^2= \sum_{k \in [m]} \left(\frac{a_k}{A}\right)^2 \cdot \sum_{k \in [m]} \normtwosm{\vw_k}^2.
    \]
    Then we know the following:
    \begin{itemize}
        \item There is $\vc\in \R^d$ that $\vw_k=a_k\vc$ for all $k$;
        \item There is $j \in [n]$ that $y_jf_{\vthetaopt}(\vx_j)=-y_jf_{\vthetaopt}(-\vx_j)=\gamma(\vthetaopt)$.
    \end{itemize}
    Note that $\phi(z) + \phi(-z) = (1-\alphaLK) \abssm{z}$. Then we have
    \[
    	0 = f_{\vthetaopt}(\vx_j)+f_{\vthetaopt}(-\vx_j)=\sum_{k \in [m]} a_k\left(\phi(a_k\vc^\top \vx_j)+\phi(-a_k\vc^\top \vx_j)\right)=\sum_{k=1}^m (1-\alphaLK) a_k \abssm{a_k \vc^\top \vx_j}.
    \]
    Certainly $\vc^\top \vx_j\neq 0$ as otherwise the margin would be zero. Then $\sum_{k \in [m]} a_k|a_k|=0$,
    which means $\sum_{k:a_k\geq 0}a^2_k = \sum_{k:a_k<0}a^2_k = \frac{1}{2}A^2$, and therefore
    \begin{align*}
    	f_{\vthetaopt}(\vx)=\sum_{k=1}^m a_k\phi(a_k\vc^\top \vx)&=\sum_{k=1}^m a_k|a_k|\phi(\sgn(a_k)\vc^\top \vx)\\
    	&=\frac{1}{2}A^2(\phi(\vc^\top \vx)-\phi(-\vc^\top \vx))=\frac{1}{2}A^2(1+\alphaLK)\vc^\top \vx
    \end{align*}
    is a linear function in $\vx$.

    Finally, let $\gamma_{\vwopt}=\min_{i\in[n]}y_i \dotp{\vwopt}{\vx_i}$ be the
    maximum linear margin, where $\vwopt \in \sphS^{d-1}$ is the max-margin
    linear separator. As $\normtwo{\vthetaopt}^2=1=(1+\normtwo{\vc}^2)A^2$,
    \begin{align*}
        \gamma(\vthetaopt) = \frac{1}{2}A^2(1+\alphaLK)\min_{i\in[n]}y_i
    \vc^\top \vx_i &\leq
    \frac{1}{2}A^2(1+\alphaLK)\normtwo{\vc}\gamma_{\vwopt} \\
    &=\frac{\normtwo{\vc}}{2(1+\normtwo{\vc}^2)}(1+\alphaLK)\gamma_{\vwopt}\leq
    \frac{1}{4}(1+\alphaLK)\gamma_{\vwopt}.
    \end{align*}
    By choosing $\vc=\vwopt$ with $A=\frac{1}{\sqrt{2}}$, the
    network is able to attain the margin
    $\frac{1}{4}(1+\alphaLK)\gamma_{\vwopt}$. As $\vthetaopt$ is already a
    global-max-margin direction, we know again that the equalities must hold.
    Therefore we know
    \begin{itemize}
        \item $\min_{i\in[n]}y_i \vc^\top \vx_i=\normtwo{\vc}\gamma_{\vwopt}$;
        \item $\frac{\normtwo{\vc}}{1+\normtwo{\vc}^2}=\frac{1}{2}$.
    \end{itemize}
    Then we know $\vc=\vwopt$ due to the uniqueness of the max-margin linear separator, and thus $A=\frac{1}{\sqrt{2}}$. Therefore the function is $f_{\vthetaopt}(\vx) = \frac{1 + \alphaLK}{4} \dotp{\vwopt}{\vx}$.
\end{proof}

\section{Proofs for Phase I}
In the subsequent sections we first show the proofs for the symmetric datasets under \Cref{ass:sym}. Additional proofs for the non-symmetric counterparts are provided in \Cref{sec:proof-nonsym}.

As we have illustrated in \Cref{sec:sketch-phase-i}, we have $G(\vw) = \dotp{\vw}{\scaledvmu}$ under \Cref{ass:sym}. Then we have
\[
	\tildeLoss(\vtheta) := \ell(0) - \sum_{k \in [m]} a_k G(\vw_k) =  \ell(0) - \sum_{k \in [m]} a_k \dotp{\vw_k}{\scaledvmu}.
\]
This means the dynamics of $\tildevtheta(t) = (\tilde{\vw}_1(t), \dots, \tilde{\vw}_m(t), \tilde{a}_1(t), \dots, \tilde{a}_m(t)) = \phisg(\tildevtheta_0, t)$ can be described by linear ODE:
\[
	\frac{\dd \tilde{\vw}_k}{\dd t} = a_k \scaledvmu, \qquad \frac{\dd \tilde{a}_k}{\dd t} = \dotp{\tilde{\vw}_k}{\scaledvmu}.
\]

\begin{lemma} \label{lm:phisg-growth}
	Let $\tildevtheta(t) = \phisg(\tildevtheta_0, t)$. Then
	\[
		\normMsm{\tildevtheta(t)} \le \exp(t\lambda_0)\normMsm{\tildevtheta_0}.
	\]
\end{lemma}
\begin{proof}
	By definition and Cauchy-Schwartz inequality,
	\[
		\normtwo{\frac{\dd \tilde{\vw}_k}{\dd t}} \le \normtwosm{\tilde{a}_k \scaledvmu} \le \lambda_0 \abssm{\tilde{a}_k}, \qquad
		\abs{\frac{\dd \tilde{a}_k}{\dd t}} \le \abssm{\tilde{\vw}_k^{\top} \scaledvmu} \le \lambda_0 \normtwosm{\tilde{\vw}_k}.
	\]
	So we have $\normMsm{\tilde{\vtheta}(t)} \le \normMsm{\vtheta_0} + \int_{0}^{t} \lambda_0 \normMsm{\tilde{\vtheta}(\tau)} \dd\tau$. Then we can finish the proof by \Gronwall's inequality~\eqref{eq:gron-1}.
\end{proof}

\begin{lemma} \label{lm:phase1-main}
	For initial point $\vtheta_0 \ne \vzero$, we have
	\[
	\normM{\vtheta(t) - \phisg(\vtheta_0, t)} \le \frac{4m\normMsm{\vtheta_0}^3}{\lambda_0} \exp(3\lambda_0 t),
	\]
	for all $t \le \frac{1}{\lambda_0} \ln \frac{\sqrt{\lambda_0/4}}{\sqrt{m}\normMsm{\vtheta_0}}$.
\end{lemma}
\begin{proof}
	Let $\tilde{\vtheta}(t) = \phisg(\vtheta_0, t)$. By \Cref{cor:diff-Loss-tildeLoss-smooth}, the following holds for a.e. $t \ge 0$,
	\begin{align*}
		\normM{\frac{\dd \vtheta}{\dd t} - \frac{\dd \tildevtheta}{\dd t}} &\le \sup\left\{\normMsm{\vdelta  - \nabla \tildeLoss(\vtheta)} : \vdelta \in  \cpartial \Loss(\vtheta) \right\} + \normMsm{\nabla \tildeLoss(\vtheta) - \nabla \tildeLoss(\tildevtheta)} \\
		&\le m\normMsm{\vtheta(t)}^3 + \lambda_0\normMsm{\vtheta - \tilde{\vtheta}}.
	\end{align*}
	Taking integral, we have
	\[
	\normMsm{\vtheta(t) - \tilde{\vtheta}(t)} \le \int_{0}^t \left(m\normMsm{\vtheta(\tau)}^3 + \lambda_0\normMsm{\vtheta(\tau) - \tilde{\vtheta}(\tau)} \right) \dd \tau.
	\]
	Let $t_0 := \inf\{t \ge 0 : \normMsm{\vtheta(t)} \ge 2\normMsm{\vtheta_0} \exp(\lambda_0 t) \}$. Then for all $0 \le t \le t_0$ (or for all $t \ge 0$ if $t_0 = +\infty$),
	\begin{align*}
	\normMsm{\vtheta(t) - \tilde{\vtheta}(t)}
	&\le \int_{0}^t \left(8m\normMsm{\vtheta_0}^3 \exp(3\lambda_0 \tau) + \lambda_0\normMsm{\vtheta(\tau) - \tilde{\vtheta}(\tau)} \right) \dd \tau \\
	&\le \frac{8m\normMsm{\vtheta_0}^3}{3 \lambda_0} \exp(3\lambda_0 t) + \lambda_0 \int_{0}^t \normMsm{\vtheta(\tau) - \tilde{\vtheta}(\tau)} \dd \tau.
	\end{align*}
	By \Gronwall's inequality~\eqref{eq:gron-2},
	\begin{align*}
	\normMsm{\vtheta(t) - \tilde{\vtheta}(t)} &\le \frac{8m\normMsm{\vtheta_0}^3}{3 \lambda_0} \left( \exp(3\lambda_0 t) + \lambda_0 \int_{0}^t \exp(3\lambda_0 \tau) \exp(\lambda_0(t-\tau)) \dd \tau\right) \\
	&= \frac{8m\normMsm{\vtheta_0}^3}{3 \lambda_0} \left( \exp(3\lambda_0 t) + \frac{1}{2}\exp(3\lambda_0t)\right) = \frac{4m\normMsm{\vtheta_0}^3}{\lambda_0} \exp(3\lambda_0 t),
	\end{align*}
	If $t_0 < \frac{1}{2\lambda_0} \ln \frac{\lambda_0}{4m\normMsm{\vtheta_0}^2}$, then
	\begin{align*}
		\normMsm{\vtheta(t)}
		&\le \normMsm{\tilde{\vtheta}(t)} + \frac{4m\normMsm{\vtheta_0}^3}{\lambda_0} \exp(3\lambda_0 t) \\
		&\le \normMsm{\tilde{\vtheta}(t)} + \frac{4m\normMsm{\vtheta_0}^2}{\lambda_0} \exp(2\lambda_0 t_0) \cdot \normMsm{\vtheta_0} \exp(\lambda_0 t) \\
		&< \normMsm{\tilde{\vtheta}(t)} + \normMsm{\vtheta_0} \exp(\lambda_0 t).
	\end{align*}
	By \Cref{lm:phisg-growth}, $\normMsm{\tilde{\vtheta}(t)} \le \normMsm{\vtheta_0} \exp(\lambda_0 t)$. So $\normMsm{\vtheta(t)} < 2\normMsm{\vtheta_0} \exp(\lambda_0 t)$ for all $0 \le t \le t_0$, which contradicts to the definition of $t_0$. Therefore, $t_0 \ge \frac{1}{2\lambda_0} \ln \frac{\lambda_0}{4m\normMsm{\vtheta_0}^2} = \frac{1}{\lambda_0} \ln \frac{\sqrt{\lambda_0/4}}{\sqrt{m}\normMsm{\vtheta_0}}$.
\end{proof}

\begin{proof}[Proof for \Cref{lm:phase1-diff}]
	Let $\tilde{\vtheta}(t) = (\tilde{\vw}_1(t), \dots, \tilde{\vw}_m(t), \tilde{a}_1(t), \dots, \tilde{a}_m(t)) = \phisg(\vtheta_0, t)$. Then
	\[
		[\tilde{\vw}_k(t), \tilde{a}_k(t)]^{\top} = \exp(T_1(r)\mMmu) [\tilde{\vw}_k(0), \tilde{a}_k(0)]^{\top} = \exp(T_1(r)\mMmu) [\sigmainit\barvw_k, \sigmainit \bara_k]^{\top},
	\]
	where $\mMmu$ is defined in \Cref{sec:sketch-phase-i},
	\[
		\mMmu := \begin{bmatrix}
		\vzero & \scaledvmu \\
		\scaledvmu^{\top} & 0
		\end{bmatrix}.
	\]
	Let $\barvmu_2 := \frac{1}{\sqrt{2}}[\barvmu, 1]^\top$ be the top eigenvector of $\mMmu$, which is associated with eigenvalue $\lambda_0$. All the other eigenvalues of $\mMmu$ are no greater than $0$. Note that
	\[
		\exp(T_1(r)\lambda_0) \barvmu_2 \barvmu_2^{\top} \begin{bmatrix}
		\sigmainit\barvw_k \\
		\sigmainit \bara_k
		\end{bmatrix} = 
		\frac{r}{\sqrt{m} \normMsm{\vtheta_0}} \left(\frac{\sigmainit}{\sqrt{2}}(\barvmu^{\top} \barvw_k + \bara_k)\right) \barvmu_2
		= \sqrt{2} r \barb_k \barvmu_2 = r \barb_k\begin{bmatrix}
		\barvmu \\
		1
		\end{bmatrix}.
	\]
	So we have
	\begin{align*}
		\normtwo{\begin{bmatrix}
			\tilde{\vw}_k(T_1(r)) \\
			\tilde{a}_k(T_1(r))
			\end{bmatrix} - r \barb_k\begin{bmatrix}
			\barvmu \\
			1
			\end{bmatrix}}
		&= \normtwo{\left(\exp(T_1(r)\mMmu) - \exp(T_1(r)\lambda_0) \barvmu_2 \barvmu_2^{\top}\right) \begin{bmatrix}
			\sigmainit\barvw_k \\
			\sigmainit \bara_k
		\end{bmatrix}
		} \\
		&\le \sigmainit \normtwo{\begin{bmatrix}
			\barvw_k \\
			\bara_k
			\end{bmatrix}} \le \sqrt{2} \sigmainit \normMsm{\barvtheta_0}.
	\end{align*}
	With probability $1$, $\barvtheta_0 \ne \vzero$. For $r \le \sqrt{\lambda_0/4}$, we have $T_1(r) = \frac{1}{\lambda_0} \ln \frac{r}{\sqrt{m} \normMsm{\vtheta_0}} \le \frac{1}{\lambda_0} \ln \frac{\sqrt{\lambda_0/4}}{\sqrt{m}\normMsm{\vtheta_0}}$. Then by \Cref{lm:phase1-main},
	\[
		\normMsm{\vtheta(T_1(r)) - \tilde{\vtheta}(T_1(r))} \le \frac{4m\normMsm{\vtheta_0}^3}{\lambda_0} \exp(3\lambda_0 T_1(r)) = \frac{4r^3}{\lambda_0 \sqrt{m}}.
	\]
	By triangle inequality, we have
	\begin{align*}
		\normM{\begin{bmatrix}
				\vw_k(T_1(r)) \\
				a_k(T_1(r))
				\end{bmatrix} - r \barb_k\!\!\begin{bmatrix}
				\barvmu \\
				1
				\end{bmatrix}}
		&\le \normM{\begin{bmatrix}
				\vw_k(T_1(r)) \\
				a_k(T_1(r))
				\end{bmatrix} - \begin{bmatrix}
				\tilde{\vw}_k(T_1(r)) \\
				\tilde{a}_k(T_1(r))
				\end{bmatrix}} + \normM{\begin{bmatrix}
				\tilde{\vw}_k(T_1(r)) \\
				\tilde{a}_k(T_1(r))
				\end{bmatrix} - r \barb_k\!\!\begin{bmatrix}
				\barvmu \\
				1
				\end{bmatrix}} \\
		& \le \frac{4r^3}{\lambda_0 \sqrt{m}} + \sqrt{2}\sigmainit \normMsm{\barvtheta_0} \le \frac{Cr^3}{\sqrt{m}},
	\end{align*}
	for some universal constant $C$, where the last step is due to our choice of $\sigmainit \le \frac{r^3}{\sqrt{m}\normM{\barvtheta_0}}$.
\end{proof}

\section{Proofs for Phase II}

\subsection{Proof for Exact Embedding}
To prove \Cref{lm:two-neuron-embedding}, we start from the following lemma.
\begin{lemma} \label{lm:two-neuron-gf-embed-gf}
	Given $\hatvtheta_0 := (\hat{\vw}_1, \hat{\vw}_2, \hat{a}_1, \hat{a}_2)$ with $\hat{a}_1 > 0$ and $\hat{a}_2 < 0$, then $\vtheta(t)=\pi_{\vb}(\phitheta(\hatvtheta_0, t))$ is a gradient flow trajectory on $\Loss(\vtheta)$ starting from $\vtheta(0)=\pi_{\vb}(\hatvtheta_0)$.
\end{lemma}
First we notice the following fact.
\begin{lemma} \label{lm:grad_embed}
For any $\hatvtheta$ and $\vg \in \cpartial \Loss(\hatvtheta)$, $\pi_{\vb}(\vg)\in \cpartial \Loss(\pi_{\vb}(\hatvtheta))$.
\end{lemma}
Below we use $\pi_{\vb}(S)=\{\pi_{\vb}(s): \vs\in S\}$ to denote the embedding of a parameter set.
\begin{proof}
For every $\hatvtheta=(\hat{\vw}_1, \hat{\vw}_2, \hat{a}_1, \hat{a}_2) \in \OmegaS$ (i.e., no activation function has zero input), let $\vtheta=\pi_{\vb}(\hatvtheta)=(\vw_1,\dots,\vw_m, a_1,\dots,a_m)$, and clearly $\vtheta\in \OmegaS$. Then $\cpartial \Loss(\hatvtheta)=\{\nabla \Loss(\hatvtheta)\}$ and $ \cpartial\Loss(\vtheta)=\{\nabla \Loss(\vtheta)\}$ are the usual differentials. In this case, we can apply the chain rule as \[\nabla \Loss(\hatvtheta)=\frac{1}{n}\sum\limits_{i\in [n]}y_i\ell'(y_if_{\hatvtheta}(\vx_i)) \frac{\partial f_{\hatvtheta}(\vx_i)}{\partial \hatvtheta},\]
\[\nabla \Loss(\vtheta)=\frac{1}{n}\sum\limits_{i\in [n]}y_i\ell'(y_if_{\vtheta}(\vx_i)) \frac{\partial f_{\vtheta}(\vx_i)}{\partial \vtheta}.\]
Notice that the embedding preserves the function value,
\begin{align*}
		f_{\vtheta}(\vx_i) =\sum\limits_{j=1}^m a_j\phi(\vw_j^\top \vx_i) 
		&=\sum\limits_{j:b_j>0} \frac{b^2_j}{b^2_{+}}\hata_1\phi(\hatvw_1^\top \vx_i)+\sum\limits_{j:b_j<0} \frac{b^2_j}{b^2_{-}}\hata_2\phi(\hatvw_2^\top \vx_i) \\
		&=\hata_1\phi(\hatvw_1^\top \vx_i)+\hata_2\phi(\hatvw_2^\top \vx_i)=f_{\hatvtheta}(\vx_i);
\end{align*}
and the also preserves the gradient
	\begin{align*}
		\frac{\partial f_{\vtheta}(\vx_i)}{\partial \vw_k} &=a_k{\phi'}(\vw_k^\top \vx_i)\vx_i=\begin{cases}
		\frac{b_k}{b_+}\hata_1 {\phi'}(\hatvw_1^\top \vx_i)\vx_i & \quad \text{if } b_k>0\\
		\frac{b_k}{b_-}\hata_2 {\phi'}(\hatvw_2^\top \vx_i)\vx_i & \quad \text{if } b_k<0
		\end{cases}, \\
		\frac{\partial f_{\vtheta}(\vx_i)}{\partial a_k} &=\phi(\vw_k^\top \vx_i)=
		\begin{cases}
		\frac{b_k}{b_+}\phi(\hatvw_1^\top \vx_i) & \quad \text{if } b_k>0\\
		\frac{b_k}{b_-}\phi(\hatvw_2^\top \vx_i) & \quad \text{if } b_k<0\\
		\end{cases},
	\end{align*}
	so $\frac{\partial f_{\vtheta}(\vx_i)}{\partial \vtheta}=\pi_{\vb}\left(\frac{\partial f_{\hatvtheta}(\vx_i)}{\partial\hatvtheta}\right)$. Then from the chain rule above we can see $\nabla \Loss(\vtheta)=\pi_{\vb}(\nabla \Loss(\hatvtheta))$, and we proved the lemma in this case.
	
	In the general case, by the definition of Clarke's subdifferential, 
	\[\cpartial \Loss(\vtheta):= \conv\left\{ \lim_{n \to \infty} \nabla
	\Loss(\vtheta_n) : \Loss \text{ differentiable at }\vtheta_n, \lim_{n \to
	\infty}\vtheta_n = \vtheta \right\}.\]
	For any $\hatvtheta_n\to\hatvtheta$
	with $\hatvtheta_n\in\OmegaS$,
	$\pi_{\vb}(\hatvtheta_n)\to\pi_{\vb}(\hatvtheta)$, and
	\[
		\lim_{n \to \infty} \nabla \Loss(\pi_{\vb}(\hatvtheta_n))=\lim_{n \to \infty} \pi_{\vb}(\nabla \Loss(\hatvtheta_n))=\pi_{\vb}\left(\lim_{n \to \infty}\nabla\Loss(\hatvtheta_n)\right).
	\] Taking the convex hull, it follows that
	$\pi_{\vb}(\cpartial \Loss(\hatvtheta))\subseteq \cpartial
	\Loss(\pi_{\vb}(\hatvtheta))$, and we finished the proof.
\end{proof}
\begin{proof}[Proof for \Cref{lm:two-neuron-gf-embed-gf}]
	 For notations we write $\hatvtheta(t):=\phitheta(\hatvtheta_0, t)$ and
	 $\vtheta(t) = \pi_{\vb}(\hatvtheta(t))$.  Then $\frac{\dd}{\dd
	 t}\hatvtheta(t)\in -\cpartial\Loss(\hatvtheta(t))$ for a.e.~$t$. At these
	 $t$, $\frac{\dd}{\dd t}\vtheta(t)=\pi_{\vb}(\frac{\dd}{\dd t}\hatvtheta(t))
	 \in \pi_{\vb}(-\cpartial\Loss(\hatvtheta(t)))$. From \Cref{lm:grad_embed}
	 we know $ \pi_{\vb}(\cpartial\Loss(\hatvtheta(t)))\subseteq \cpartial
	 \Loss(\vtheta(t))$. Then $\frac{\dd}{\dd t}\vtheta(t)\in -\cpartial
	 \Loss(\vtheta(t))$ for a.e.~$t$, and therefore $\vtheta(t)$ is indeed a
	 gradient flow trajectory.
\end{proof}
\begin{proof}[Proof for \Cref{lm:two-neuron-embedding}]
	By  \Cref{lm:two-neuron-gf-embed-gf}, $\pi_{\vb}(\phitheta(\hatvtheta_0,
	t))$ is indeed a gradient flow trajectory. Then, as
	$\pi_{\vb}(\phitheta(\hatvtheta_0, 0))=\pi_{\vb}(\hatvtheta_0)$, as well as
	the fact that $\hatvtheta_0$ and $\pi_{\vb}(\hatvtheta_0)$ are non-branching
	starting points, the gradient flow trajectory is unique and therefore
	$\pi_{\vb}(\phitheta(\hatvtheta_0, t))=\phitheta(\pi_{\vb}(\hatvtheta_0),t)$
	for all $t\geq 0$.
\end{proof}

\subsection{A General Theorem for Limiting Trajectory Near Zero}
Before analyzing Phase II, we first give a characterization for gradient flow on Leaky ReLU networks with logistic loss, starting near $r\hatvtheta$, where $\hatvtheta$ is a well-aligned parameter vector defined below. We only assume that the inputs are bounded $\normtwosm{\vx_i} \le 1$ and $\lambda := \max\{\abssm{G(\vw)} : \vw \in \sphS^{d-1}\} > 0$. Theorems in the section will be used again in the non-symmetric case.

\begin{definition}[Well-aligned Parameter Vector] \label{def:well-aligned}
	We say that $\hatvtheta := (\hatvw_1, \dots, \hatvw_m, \hata_1, \dots, \hata_m)$ is a \textit{well-aligned parameter vector} if it satisfies the following for some $1 \le p \le m$:
	\begin{enumerate}
		\item For $1 \le k \le p$, $\frac{\hatvw_k}{\normtwosm{\hatvw_k}}$ attains the maximum value of $\abssm{G(\vw)}$ on $\sphS^{d-1}$, i.e., $\abs{G(\frac{\hatvw_k}{\normtwosm{\hatvw_k}})} = \lambda$;
		\item For $1 \le k \le p$, $\hata_k = \sgn(G(\hatvw_k)) \normtwosm{\hatvw_k}$;
		\item For $1 \le k \le p$, $\dotp{\hatvw_k}{\vx_i} \ne 0$ for all $i \in [n]$;
		\item For $p + 1 \le k \le m$, $\hatvw_k = \vzero$, $\hata_k = 0$.
	\end{enumerate}
\end{definition}
Our analysis for Phase I shows that weight vectors approximately align to either of $\barvmu$ or $-\barvmu$, and both of them are maximizers of $\abssm{G(\vw)}$. Therefore, gradient flow goes near a well-aligned parameter vector (with $p = m$) at the end of Phase I.

The following is the main theorem of this subsection.

\begin{theorem} \label{thm:phase-2-general}
	Let $\hatvtheta$ be a well-aligned parameter vector. Let $\hatgap := \min_{k\in [p], i \in [n]} \frac{\abssm{\dotp{\hatvw_k}{\vx_i}}}{\normMsm{\hatvtheta}} > 0$. Define $T_2(r) := \frac{1}{\lambda} \ln \frac{1}{r}$ and let $t_0$ be the following time constant
	\begin{equation} \label{eq:phase-2-general-def-t0}
		t_0 := \frac{1}{2\lambda} \ln \frac{\lambda \hatgap}{16 m \normMsm{\hatvtheta}^2}.
	\end{equation}
	Then for all $t \in (-\infty, t_0]$, the following is true:
	\begin{enumerate}
		\item $\lim_{r \to 0} \phitheta(r\hatvtheta, T_2(r) + t)$ exists. This limit is independent of the choice of $\phitheta$ when the gradient flow may not be unique.
		
		\item $\lim_{r \to 0} \phitheta(r\hatvtheta, T_2(r) + t)$ lies near $e^{\lambda t}\hatvtheta$:
			\[
				\normM{\lim_{r \to 0} \phitheta\left(r\hatvtheta, T_2(r) + t\right) - e^{\lambda t}\hatvtheta} \le \frac{4m \normMsm{\hatvtheta}^3}{\lambda} e^{3\lambda t}.
			\]
		\item Let $\vtheta_1, \vtheta_2, \dots$ be a series of parameters converging to $\vzero$, $r_1, r_2, \dots$ be a series of positive real numbers converging to $0$. If $\normtwosm{\vtheta_s - r_s \hatvtheta} \le C r_s^{1+\kappa}$ for some $C > 0, \kappa > 0$, then
		\[
			\lim_{s \to \infty} \phitheta\left(\vtheta_s, T_2(r_s) + t\right) = \lim_{r \to 0} \phitheta(r\hatvtheta, T_2(r) + t).
		\]
	\end{enumerate}
\end{theorem}

Now we prove \Cref{thm:phase-2-general}. Throughout this subsection, we fix a well-aligned parameter vector $\hatvtheta  := (\hatvw_1, \dots, \hatvw_m, \hata_1, \dots, \hata_m)$ with constant $p \in [m]$. We also use $t_0$ and $T_2(r)$ to denote the same constant $t_0$ defined by \eqref{eq:phase-2-general-def-t0} and the same function $T_2(r) := \frac{1}{\lambda} \ln \frac{1}{r}$ as in \Cref{thm:phase-2-general}.

For any parameter $\vtheta = (\vw_1, \dots, \vw_m, a_1, \dots, a_m)$, we use $\normPsm{\vtheta}$, $\normRsm{\vtheta}$ to denote the following semi-norms respectively,
\[
	\normPsm{\vtheta} := \max_{k \in [p]} \left\{\max\{\normtwosm{\vw_k}, \abssm{a_k} \} \right\}, \qquad \normRsm{\vtheta} := \max_{p < k \le m} \left\{\max\{\normtwosm{\vw_k}, \abssm{a_k} \} \right\}.
\]
The M-norm can be expressed in terms of P-norm and R-norm: $\normMsm{\vtheta} = \max\{\normPsm{\vtheta}, \normRsm{\vtheta}\}$. Also note that Condition 4 in \Cref{def:well-aligned} is now equivalent to $\normRsm{\hatvtheta} = 0$.

For $k \in [p]$, define $\hatsetw_k := \{ \vw \in \R^d : \dotp{\hatvw_k}{\vx_i}
\cdot \dotp{\vw}{\vx_i} > 0, \forall i \in [n] \}$ to be the set of weights that
share the same activation pattern as $\hatvw_k$.
\begin{lemma} \label{lm:phase-2-general-diff-to-hatvtheta}
	If $r > 0$ is small enough and the initial point $\vtheta_0$ of gradient flow satisfies  $\normMsm{\vtheta_0 - r\hatvtheta} \le Cr^{1 + \kappa}$ for some $C > 0, \kappa > 0$, then for any $-T_2(r) \le t \le t_0$, the following four properties hold:
	\begin{enumerate}
		\item For all $k \in [p]$, $\vw_k(T_2(r) + t) \in \hatsetw_k$;
		\item $\normMsm{\phitheta(\vtheta_0,T_2(r) + t)} \le 2 e^{\lambda t}\normMsm{\hatvtheta}$;
		\item $\normPsm{\phitheta\left(\vtheta_0, T_2(r) + t\right) - e^{\lambda t}\hatvtheta} \le Cr^{\kappa} e^{\lambda t} + \frac{4m\normMsm{\hatvtheta}^3}{\lambda}  e^{3\lambda t}$;
		\item $\normRsm{\phitheta\left(\vtheta_0, T_2(r) + t\right) - e^{\lambda t}\hatvtheta} \le 2Cr^{\kappa} e^{\lambda t} $.
	\end{enumerate}
\end{lemma}
\begin{proof}
	Let $\vtheta(t) := \phitheta(\vtheta_0, t)$ be gradient flow on $\Loss$ starting from $\vtheta_0$, and $\tildevtheta(t) := re^{\lambda t}\hatvtheta$.
	Let $t_1, t_2$ be the following time constants and define $t_{\max} := \min\{t_0, t_1, t_2\}$:
	\begin{align*}
		t_1 &:= \inf\{ t \ge 0 : \exists k \in [p], \vw_k(t) \notin \hatsetw_k \} - T_2(r), \\
		t_2 &:= \inf\{t \ge 0: \normMsm{\vtheta(t)} \ge 2 r e^{\lambda t} \normMsm{\hatvtheta} \} - T_2(r).
	\end{align*}
	We also define $r_{\max} := \left(\frac{\normMsm{\hatvtheta} \hatgap}{8C}\right)^{1/\kappa}$. We only consider the dynamics for $r \le r_{\max}$, $t < T_2(r) + t_{\max}$. Our goal is to show that 
	\[
		\normPsm{\vtheta(t) - \tildevtheta(t)} \le Cr^{1+\kappa} e^{\lambda t} + \frac{4m\normMsm{\hatvtheta}^3}{\lambda}  r^3e^{3\lambda t}, \qquad \normRsm{\vtheta(t) - \tildevtheta(t)} \le 2Cr^{1+\kappa} e^{\lambda t}
	\]
	within the time interval $[0, T_2(r) + t_{\max})$ (and thus it also holds for $[0, T_2(r) + t_{\max}]$ by continuity), and to show that $t_0$ is actually equal to $t_{\max}$, i.e., $t_0$ is the minimum among $t_0, t_1, t_2$. It is easy to see that proving these suffice to deduce the original lemma statement, given the translation of time $e^{\lambda T_2(r)}=\frac{1}{r}$.
	
	For $k \in [m]$, by \Cref{lm:diff-Loss-tildeLoss-partial} we have
	\begin{equation} \label{eq:phase-2-general-diff-to-hatvtheta-tilde}
		\normtwo{\frac{\dd \vw_k}{\dd t} - a_k \cpartial G(\vw_k)} \subseteq (-\infty, m \normMsm{\vtheta}^2 \abssm{a_k} ], \qquad \abs{\frac{\dd a_k}{\dd t} - G(\vw_k)} \le m \normMsm{\vtheta}^2 \normtwosm{\vw_k}.
	\end{equation}
	For $\tildevtheta(t)$, a simple calculus shows that for all $t \ge 0$,
	\begin{align}
		\forall k \in [p]:& \qquad \frac{\dd \tildevw_k}{\dd t} = \lambda \tildea_k \frac{\hatvw_k}{\normtwosm{\hatvw_k}}, \qquad \frac{\dd \tildea_k}{\dd t} = \lambda \dotp{\frac{\hatvw_k}{\normtwosm{\hatvw_k}}}{\tildevw_k}. \label{eq:phase-2-general-diff-to-hatvtheta-a} \\
		\forall p < k \le m:& \qquad \abssm{\tildea_k} = \normtwosm{\tildevw_k} = 0. \label{eq:phase-2-general-diff-to-hatvtheta-b}
	\end{align}
	\myparagraph{Bounding $\normPsm{\vtheta(t) - \tildevtheta(t)}$.} For $k \in [p]$, $\cpartial G(\vw_k) = \{\nabla G(\vw_k)\} = \{\nabla G(\hatvw_k)\}$. Also note that $\nabla G(\hatvw_k) = \nabla G(\frac{\hatvw_k}{\normtwosm{\hatvw_k}}) = \lambda \frac{\hatvw_k}{\normtwosm{\hatvw_k}}$ by \Cref{lm:G-opt-grad}. Then $a_k \cpartial G(\vw_k) = \{\lambda a_k \frac{\hatvw_k}{\normtwosm{\hatvw_k}} \}$ and $G(\vw_k) = \lambda \dotp{\frac{\hatvw_k}{\normtwosm{\hatvw_k}}}{\vw_k}$. Combining these with \eqref{eq:phase-2-general-diff-to-hatvtheta-tilde} gives
	\begin{equation}
	\max\left\{\normtwo{\frac{\dd \vw_k}{\dd t} - \lambda a_k \frac{\hatvw_k}{\normtwosm{\hatvw_k}}}, \abs{\frac{\dd a_k}{\dd t} - \lambda \dotp{\frac{\hatvw_k}{\normtwosm{\hatvw_k}}}{\vw_k}} \right\} \le m \normMsm{\vtheta}^3.
	\end{equation}
	Then by \eqref{eq:phase-2-general-diff-to-hatvtheta-a} we have
	\begin{align*}
		\normP{\frac{\dd \vtheta}{\dd t} - \frac{\dd \tildevtheta}{\dd t}} &\le m\normMsm{\vtheta}^3 + 
		\max_{k \in [p]}\left\{
			\normtwo{\lambda (a_k - \tildea_k) \frac{\hatvw_k}{\normtwosm{\hatvw_k}}},
			\abs{\lambda \dotp{\frac{\hatvw_k}{\normtwosm{\hatvw_k}}}{\vw_k - \tildevw_k}}
		\right\} \\
		&\le m\normMsm{\vtheta}^3 + \lambda \normPsm{\vtheta - \tildevtheta}.
	\end{align*}
	Taking the integral gives $\normPsm{\vtheta(t) - \tildevtheta(t)} \le \normPsm{\vtheta(0) - \tildevtheta(0)} + \int_{0}^{t} (m\normMsm{\vtheta(\tau)}^3 + \lambda \normPsm{\vtheta(\tau) - \tildevtheta(\tau)}) \dd \tau$. Note that $t_{\max} \le t_2$. Then
	\begin{align*}
		\normPsm{\vtheta(t) - \tildevtheta(t)} &\le \normPsm{\vtheta(0) - \tildevtheta(0)} + \int_{0}^{t} \left(8m r^3 e^{3\lambda t}\normMsm{\hatvtheta}^3 + \lambda \normPsm{\vtheta(\tau) - \tildevtheta(\tau)}\right) \dd \tau \\
		&\le Cr^{1+\kappa} + \frac{8}{3 \lambda}m r^3e^{3\lambda t} \normMsm{\hatvtheta}^3 + \lambda \int_{0}^{t} \normPsm{\vtheta(\tau) - \tildevtheta(\tau)} \dd \tau.
	\end{align*}
	By \Gronwall's inequality \eqref{eq:gron-1}, we have
	\begin{align*}
		\normPsm{\vtheta(t) - \tildevtheta(t)} &\le Cr^{1+\kappa} + \frac{8}{3 \lambda}m r^3e^{3\lambda t} \normMsm{\hatvtheta}^3 + \int_{0}^{t} \left(Cr^{1+\kappa} + \frac{8}{3 \lambda}m r^3e^{3\lambda \tau} \normMsm{\hatvtheta}^3 \right) \lambda e^{\lambda(t - \tau)}  \dd \tau \\
		&\le Cr^{1+\kappa} + Cr^{1+\kappa} (e^{\lambda t} - 1) + \frac{8}{3 \lambda}m r^3e^{3\lambda t} \normMsm{\hatvtheta}^3 + \frac{8}{3\lambda}m r^3  \cdot \frac{e^{\lambda t}}{2}(e^{2 \lambda t} - 1) \normMsm{\hatvtheta}^3 \\
		&\le Cr^{1+\kappa} e^{\lambda t} + \frac{8}{3 \lambda}m r^3e^{3\lambda t}  (1 + 1/2) \normMsm{\hatvtheta}^3.
	\end{align*}
	Therefore we can conclude that
	\begin{equation} \label{eq:phase-2-general-diff-to-hatvtheta-aa}
		\normPsm{\vtheta(t) - \tildevtheta(t)} \le Cr^{1+\kappa} e^{\lambda t} + \frac{4m\normMsm{\hatvtheta}^3}{\lambda} r^3  e^{3\lambda t}.
	\end{equation}
	\myparagraph{Bounding $\normRsm{\vtheta(t) - \tildevtheta(t)}$.} For $p < k \le m$, we can combine \Cref{thm:homo-euler} and \eqref{eq:phase-2-general-diff-to-hatvtheta-tilde} to give the following bound for the norm growth:
	\begin{align*}
	\frac{1}{2}\frac{\dd \normtwosm{\vw_k}^2}{\dd t} = \frac{1}{2}\frac{\dd \abssm{a_k}^2}{\dd t} &\le a_k G(\vw_k) + \abssm{a_k} \cdot m \normMsm{\vtheta}^2 \normtwosm{\vw_k}.
	\end{align*}
	This implies
	\begin{equation}
	\frac{1}{2}\frac{\dd \abssm{a_k}^2}{\dd t} = \frac{1}{2}\frac{\dd \normtwosm{\vw_k}^2}{\dd t} \le \normRsm{\vtheta}^2 (\lambda + m \normMsm{\vtheta}^2).
	\end{equation}
	Taking the integral gives $\normRsm{\vtheta(t)}^2 \le \normRsm{\vtheta(0)}^2 + \int_{0}^{t} 2\normRsm{\vtheta(\tau)}^2 (\lambda + m \normMsm{\vtheta(\tau)}^2) \dd \tau$.
	Note that $t_{\max} \le t_2$ and $\normRsm{\vtheta(0)} \le Cr^{1+\kappa}$. Then
	\[
		\normRsm{\vtheta(t)}^2 \le C^2 r^{2(1+\kappa)} + \int_{0}^{t} 2\normRsm{\vtheta(\tau)}^2 (\lambda + 4mr^2e^{2\lambda\tau} \normMsm{\hatvtheta}^2) \dd \tau
	\]
	By \Gronwall's inequality \eqref{eq:gron-2}, we have
	\begin{align*}
		\normRsm{\vtheta(t)}^2 &\le C^2 r^{2(1+\kappa)} \exp\left(\int_{0}^{t} 2(\lambda + 4mr^2e^{2\lambda\tau}\normMsm{\hatvtheta}^2) \dd \tau\right) \\
		&\le C^2 r^{2(1+\kappa)} \exp\left(2\lambda t + \frac{4m\normMsm{\hatvtheta}^2}{\lambda} r^2e^{2\lambda t}\right). 
	\end{align*}
	Taking the square root gives
	\[
		\normRsm{\vtheta(t)} \le C r^{1+\kappa} \exp\left(\lambda t + \frac{2m\normMsm{\hatvtheta}^2}{\lambda} r^2e^{2\lambda t} \right).
	\]
	For $t \le T(r) + t_{\max} \le T(r) + t_0$, we can use the the definition \eqref{eq:phase-2-general-def-t0} of $t_0$ to deduce that $\frac{2m\normMsm{\hatvtheta}^2}{\lambda} r^2e^{2\lambda t}  \le \frac{2m\normMsm{\hatvtheta}^2}{\lambda} e^{2\lambda t_0} = \hatgap/8 \le 1/8$. Therefore, we have
	\begin{equation} \label{eq:phase-2-general-diff-to-hatvtheta-bb}
		\normRsm{\vtheta(t) - \tildevtheta(t)} = \normRsm{\vtheta(t)} \le Cr^{1+\kappa} e^{\lambda t + 1/8} < Cr^{1+\kappa} e^{\lambda t + \ln 2} = 2Cr^{1+\kappa} e^{\lambda t}.
	\end{equation}
	
	\myparagraph{Bounding $t_{\max}$.} To prove the lemma, now we only need to show that $t_{\max} = t_0$. Combining \eqref{eq:phase-2-general-diff-to-hatvtheta-aa} and \eqref{eq:phase-2-general-diff-to-hatvtheta-bb}, we have
	for $t\leq T_2(r)+t_{\max}$,\[
		\normMsm{\vtheta(t) - \tildevtheta(t)} \le 2Cr^{1+\kappa} e^{\lambda t} + \frac{4m\normMsm{\hatvtheta}^3}{\lambda} r^3  e^{3\lambda t}.
	\]
	Since $r \le r_{\max}$, $2Cr^{\kappa} \le \frac{1}{4} \normMsm{\hatvtheta} \hatgap$. By definition \eqref{eq:phase-2-general-def-t0} of $t_0$, $\frac{4m\normMsm{\hatvtheta}^3}{\lambda} r^2 e^{2\lambda t} \le \frac{4m\normMsm{\hatvtheta}^3}{\lambda} e^{2\lambda t_0} \le \frac{1}{4} \normMsm{\hatvtheta}\hatgap$. Then we have $2Cr^{\kappa} + \frac{4m\normMsm{\hatvtheta}^3}{\lambda} r^2 e^{2\lambda t} \le \frac{1}{2}\normMsm{\hatvtheta}\hatgap$ and thus
	\begin{equation} \label{eq:phase-2-general-diff-to-hatvtheta-qqq}
		\normMsm{\vtheta(t) - \tildevtheta(t)} \le re^{\lambda t} \left(2Cr^{\kappa} + \frac{4m\normMsm{\hatvtheta}^3}{\lambda} r^2  e^{2\lambda t}\right) \le \frac{1}{2}re^{\lambda t} \normMsm{\hatvtheta}\hatgap.
	\end{equation}
	For all time $0 \le t < T_2(r) + t_{\max}$, we can use \eqref{eq:phase-2-general-diff-to-hatvtheta-qqq} to deduce
	\begin{align*}
		\sgn(\dotp{\hatvw_k}{\vw_i}) \dotp{\vw(t)}{\vx_i} &\ge \sgn(\dotp{\hatvw_k}{\vx_i}) \dotp{r e^{\lambda t} \hatvw_k}{\vx_i} - \frac{1}{2}re^{\lambda t} \normMsm{\hatvtheta}\hatgap \\
		&= r e^{\lambda t} \left(\abssm{\dotp{\hatvw_k}{\vx_i}} - \frac{1}{2}\normMsm{\hatvtheta}\hatgap\right) \\
		&\ge re^{\lambda t} \normMsm{\hatvtheta}\hatgap / 2  > 0,
	\end{align*}
	which implies $t_1 > t_{\max}$.
	
	For norm growth, we can again use \eqref{eq:phase-2-general-diff-to-hatvtheta-qqq} to deduce
	\begin{align*}
		\normMsm{\vtheta(t)} \le \normMsm{\tildevtheta(t)} + \frac{1}{2}re^{\lambda t} \normMsm{\hatvtheta}\hatgap &= r e^{\lambda t}\left(\normMsm{\hatvtheta} + \frac{1}{2} \normMsm{\hatvtheta} \hatgap \right) \\
		&\le \frac{3}{2}r e^{\lambda t}\normMsm{\hatvtheta} < 2re^{\lambda t} \normMsm{\hatvtheta},
	\end{align*}
	which implies $t_2 > t_{\max}$.
	
	Now we have $t_1 > t_{\max}, t_2 > t_{\max}$. Recall that $t_{\max} := \min\{t_0, t_1, t_2\}$ by definition. Then $t_{\max} = t_0$ must hold, which completes the proof.
\end{proof}

\begin{lemma} \label{lm:phase-2-general-diff}
	If $r > 0$ is small enough and the initial point $\vtheta_0$ of gradient flow satisfies  $\normMsm{\vtheta_0 - r\hatvtheta} \le Cr^{1 + \kappa}$ for some $C > 0, \kappa > 0$, then for all $t \in [-T_2(r), t_0]$,
	\[
		\normMsm{\phitheta\left(\vtheta_0, T_2(r) + t\right) - \phitheta(r\hatvtheta, T_2(r) + t)} \le 4C r^{\kappa} e^{\lambda t}.
	\]
\end{lemma}
\begin{proof}
	Let $\vtheta(t) := \phitheta(\vtheta_0, t)$ and $\tildevtheta(t) := \phitheta(r\hatvtheta, t)$ be gradient flows starting from $\vtheta_0$ and $r\hatvtheta$. For notation simplicity, let $h_{ki} = y_i\phi'(\hatvw_k^{\top} \vx_i)$. Let $g_i := -\ell'(y_i f_{\vtheta}(\vx_i))$, $\tildeg_i := -\ell'(y_i f_{\tildevtheta}(\vx_i))$.
	
	By \Cref{lm:phase-2-general-diff-to-hatvtheta}, we can make $r$ to be small enough so that the four properties hold for both $\vtheta(T_2(r) + t)$ and $\tildevtheta(T_2(r) + t)$ when $t \le t_0$.
	
	\myparagraph{Bounding the Difference for $1 \le k \le p$.} 
	For all $t \le t_0$ and $k \in [p]$, we know that $\phi'(\vw_k^{\top} \vx_i) = \phi'(\tildevw_k^{\top} \vx_i) = h_{ki}$, and thus for $\vw_k, \tildevw_k$ we have
	\begin{align*}
		\normtwo{\frac{\dd \vw_k}{\dd t} - \frac{\dd \tildevw_k}{\dd t}} &= \normtwo{\frac{a_k}{n}\sum_{i=1}^{n} g_i h_{ki} \vx_i - \frac{\tildea_k}{n}\sum_{i=1}^{n} \tildeg_i h_{ki} \vx_i} \\
		&\le \abssm{a_k - \tildea_k} \cdot \underbrace{\normtwo{\frac{1}{n}\sum_{i=1}^{n} \tildeg_i h_{ki} \vx_i}}_{\Lambda(t)} + \abssm{a_k} \cdot \underbrace{\normtwo{\frac{1}{n}\sum_{i=1}^{n} (g_i - \tildeg_i) h_{ki} \vx_i}}_{\Delta(t)} \\
		&=: \Lambda(t) \cdot \abssm{a_k - \tildea_k} + \abssm{a_k} \cdot \Delta(t).
	\end{align*}
	and for $a_k, \tildea_k$ we have
	\begin{align*}
		\normtwo{\frac{\dd a_k}{\dd t} - \frac{\dd \tildea_k}{\dd t}} &= \abs{\frac{1}{n}\sum_{i=1}^{n} g_i \phi(\vw_k^{\top} \vx_i) - \frac{1}{n}\sum_{i=1}^{n} \tildeg_i \phi(\tildevw_k^{\top} \vx_i)} \\
		&= \abs{\frac{1}{n}\sum_{i=1}^{n} g_i h_{ki} \vw_k^{\top} \vx_i - \frac{1}{n}\sum_{i=1}^{n} \tildeg_i h_{ki} \tildevw_k^{\top} \vx_i} \\
		&= \normtwosm{\vw_k - \tildevw_k} \cdot \normtwo{\frac{1}{n}\sum_{i=1}^{n} \tildeg_i h_{ki} \vx_i} + \normtwosm{\vw_k} \cdot \normtwo{\frac{1}{n}\sum_{i=1}^{n} (g_i - \tildeg_i) h_{ki} \vx_i} \\
		&= \Lambda(t) \cdot \normtwosm{\vw_k - \tildevw_k} + \normtwosm{\vw_k} \cdot \Delta(t).
	\end{align*}
	Therefore, $\normP{\frac{\dd \vtheta}{\dd t} - \frac{\dd \tildevtheta}{\dd t}} \le \Lambda(t) \cdot \normMsm{\vtheta - \tildevtheta} + \normMsm{\vtheta} \cdot \Delta(t)$. Now we turn to bound $\Lambda(t)$ and $\Delta(t)$.
	
	By Lipschitzness of $\ell'$ and \Cref{lm:f-output-ub}, we have
	\[
		\abssm{-\ell'(0) - \tildeg_i} \le m\normMsm{\tildevtheta}^2, \qquad \abssm{g_i - \tildeg_i} \le m \normMsm{\vtheta - \tildevtheta} \left( \normMsm{\vtheta} + \normMsm{\tildevtheta}\right).
	\]
	For $\Lambda(t)$, by triangle inequality and \Cref{lm:G-opt-grad} we have
	\[
		\Lambda(t) \le \normtwo{\frac{-\ell'(0)}{n}\sum_{i=1}^{n} h_{ki} \vx_i} + m\normMsm{\tildevtheta}^2 = \normtwosm{\nabla G(\hatvw_k)} + m\normMsm{\tildevtheta}^2 = \lambda + m\normMsm{\tildevtheta}^2,
	\]
	For $\Delta(t)$, we use triangle inequality again to give the following bound:
	\[
		\Delta(t) \le \frac{1}{n}\sum_{i=1}^{n} \abssm{g_i - \tildeg_i} \le m \normMsm{\vtheta - \tildevtheta} \left( \normMsm{\vtheta} + \normMsm{\tildevtheta}\right).
	\]
	Therefore, we can conclude that
	\begin{align*}
		\normP{\frac{\dd \vtheta}{\dd t} - \frac{\dd \tildevtheta}{\dd t}} &\le (\lambda + m\normMsm{\tildevtheta}^2) \cdot \normMsm{\vtheta - \tildevtheta} + \normMsm{\vtheta} \cdot m \normMsm{\vtheta - \tildevtheta} \left( \normMsm{\vtheta} + \normMsm{\tildevtheta}\right) \\
		&\le \left(\lambda + 3m \max\{ \normMsm{\vtheta}, \normMsm{\tildevtheta}\}^2 \right)\normMsm{\vtheta - \tildevtheta} \\
		&\le \left(\lambda + 12mr^2 e^{2\lambda t} \normMsm{\hatvtheta}^2\right)\normMsm{\vtheta - \tildevtheta},
	\end{align*}
	where the last inequality uses the 2nd property in \Cref{lm:phase-2-general-diff-to-hatvtheta}. Note that $\normPsm{\vtheta_0 - r\hatvtheta} \le Cr^{1+\kappa}$. So we can write it into the integral form:
	\begin{equation} \label{eq:norm-P-vtheta-diff-integral}
		\normPsm{\vtheta(t) - \tildevtheta(t)} \le Cr^{1+\kappa} + \int_0^{t} \left(\lambda + 12mr^2 e^{2\lambda \tau} \normMsm{\hatvtheta}^2 \right)\normMsm{\vtheta(\tau) - \tildevtheta(\tau)} \dd \tau.
	\end{equation}	
	\myparagraph{Bounding the Difference for $p < k \le m$.} By \Cref{lm:weight-zero}, $\normRsm{\tildevtheta(t)} = 0$ for all $t \ge 0$, so $\normRsm{\vtheta - \tildevtheta} = \normRsm{\vtheta}$. By the 4th property in \Cref{lm:phase-2-general-diff-to-hatvtheta}, we then have
	\[
		\normRsm{\vtheta(t) - \tildevtheta(t)} = \normRsm{\vtheta(t)} = \normRsm{\vtheta(t) - re^{\lambda t} \hatvtheta} \le 2Cr^{1+\kappa}e^{\lambda t}.
	\]
	So we can verify that $\normRsm{\vtheta(t) - \tildevtheta(t)}$ satisfies the following inequality:
	\begin{equation} \label{eq:norm-R-vtheta-diff-integral}
		\normRsm{\vtheta(t) - \tildevtheta(t)} \le 2Cr^{1+\kappa} + \int_0^{t} \lambda \normRsm{\vtheta(\tau) - \tildevtheta(\tau)} \dd \tau.
	\end{equation}
	
	\myparagraph{Bounding the Difference for All.} Combining \Cref{lm:phase-2-general-diff-to-hatvtheta} and \Cref{lm:phase-2-general-diff-to-hatvtheta}, we have the following inequality for $\normMsm{\vtheta(t) - \tildevtheta(t)}$:
	\[
		\normMsm{\vtheta(t) - \tildevtheta(t)} \le 2Cr^{1+\kappa} + \int_0^{t} \left(\lambda + 12mr^2 e^{2\lambda \tau} \normMsm{\hatvtheta}^2 \right)\normMsm{\vtheta(\tau) - \tildevtheta(\tau)} \dd \tau.
	\]
	By \Gronwall's inequality \eqref{eq:gron-2},
	\begin{align*}
	\normMsm{\vtheta(t) - \tildevtheta(t)} &\le 2Cr^{1+\kappa} \exp\left(\int_{0}^{t} \left(\lambda + 12mr^2 e^{2\lambda \tau} \normMsm{\hatvtheta}^2 \right) \dd \tau\right) \\
	&\le 2Cr^{1+\kappa} \exp\left(\lambda t + \frac{6m\normMsm{\hatvtheta}^2}{\lambda}r^2 e^{2\lambda t} \right).
	\end{align*}
	By definition \eqref{eq:phase-2-general-def-t0} of $t_0$, we have $\frac{6m\normMsm{\hatvtheta}^2}{\lambda}r^2 e^{2\lambda t} \le \frac{6m\normMsm{\hatvtheta}^2}{\lambda} e^{2\lambda t_0} = \frac{3\hatgap}{8} \le 3/8 < \ln 2$.
	Therefore we have the following bound for $\normMsm{\vtheta(t) - \tildevtheta(t)}$:
	\[
		\normMsm{\vtheta(t) - \tildevtheta(t)} \le 2Cr^{1+\kappa} e^{\lambda t + \ln 2} = 4C r^{1+\kappa} e^{\lambda t}.
	\]
	At time $T_2(r) + t \in [0, T_2(r) + t_0]$, this bound can be rewritten as
	\[
		\normMsm{\vtheta(T_2(r) + t) - \tildevtheta(T_2(r) + t)} \le 4C r^{\kappa} e^{\lambda t},
	\]
	which completes the proof.
\end{proof}

\begin{proof}[Proof for \Cref{thm:phase-2-general}]
	First we show that $\lim_{r \to 0} \phitheta(r\hatvtheta, T_2(r) + t)$ exists. We consider the case of $r \le r_{\max}$, where $r_{\max}$ is chosen to be small enough so that the properties in \Cref{lm:phase-2-general-diff-to-hatvtheta} hold. For any $r' < r$, by \Cref{lm:phase-2-general-diff-to-hatvtheta} we have
	\[
		\normM{\phitheta\left(r'\hatvtheta, T_2(r') + \frac{1}{\lambda}\ln r\right) - r\hatvtheta} \le \frac{4m\normMsm{\hatvtheta}^3}{\lambda} r^3 \le C'r^{1+\kappa'},
	\]
	where $C' = \frac{4m\normMsm{\hatvtheta}^3}{\lambda}$, $\kappa' = 2$.
	Applying \Cref{lm:phase-2-general-diff}, we then have
	\[
		\normM{\phitheta\left(\phitheta\left(r'\hatvtheta, T_2(r') + \frac{1}{\lambda}\ln r\right), T_2(r) + t\right) - \phitheta\left(r\hatvtheta, T_2(r) + t\right)} \le 4C'r^{\kappa'}e^{\lambda t}.
	\]
	Note that $T_2(r') + \frac{1}{\lambda}\ln r + T_2(r) + t = T_2(r') + t$. So this proves
	\[
		\normMsm{\phitheta(r'\hatvtheta, T_2(r') + t) - \phitheta(r\hatvtheta, T_2(r) + t)} \le 4C'r^{\kappa'}e^{\lambda t}.
	\]
	For any fixed $t \le t_0$, the RHS converges to $0$ as $r \to 0$, which implies Cauchy convergence of the limit $\lim_{r \to 0} \phitheta(r\hatvtheta, T_2(r) + t)$ and thus the limit exists. By the 1st property in \Cref{lm:phase-2-general-diff-to-hatvtheta}, we know that there is no activation pattern switch in the time interval $t \in [0, T_2(r) + t_0]$ if $r$ is small enough. This means $\Loss$ is locally smooth near the trajectory of $\phitheta(r\hatvtheta, T_2(r) + t)$ and thus the trajectory is unique. Therefore, the limit $\lim_{r \to 0} \phitheta(r\hatvtheta, T_2(r) + t)$ is uniquely defined.
	
	By \Cref{lm:phase-2-general-diff-to-hatvtheta},
	\[
		\normMsm{\phitheta(r\hatvtheta, T_2(r) + t) - e^{\lambda t} \hatvtheta} \le  \frac{4m\normMsm{\hatvtheta}^3}{\lambda}  e^{3\lambda t}.
	\]
	Taking $r \to 0$ on both sides gives the range of the limit $\lim_{r \to 0} \phitheta(r\hatvtheta, T_2(r) + t)$:
	\[
		\normM{\lim_{r \to 0} \phitheta(r\hatvtheta, T_2(r) + t) - e^{\lambda t} \hatvtheta} \le \frac{4m\normMsm{\hatvtheta}^3}{\lambda}  e^{3\lambda t}.
	\]
	For $s \to \infty$, by \Cref{lm:phase-2-general-diff}, we have
	\[
		\lim_{s \to \infty}\normM{\phitheta\left(\hatvtheta_{s}, T_2(r_s) + t\right) - \phitheta\left(r_s\hatvtheta, T_2(r_s) + t\right)} = 0.
	\]
	So $\lim_{s \to \infty} \phitheta(\vtheta_s, T_2(r_s) + t) = \lim_{r \to 0} \phitheta(r\hatvtheta, T_2(r) + t)$ is proved.
\end{proof}

\subsection{Proof for Approximate Embedding}

To analyze Phase II, we need to deal with approximate embedding instead of the exact one. For this, we further divide Phase II into Phase II.1 and II.2 and analyze them in order. At the end of this subsection we will prove \Cref{lm:phase2-main}.

\subsubsection{Proofs for Phase II.1}

Given the discussions in the previous sections, we are ready to present proofs for the phase II dynamics (\Cref{lm:phase2-main}) here. 

We subdivide the dynamics of Phase II into Phase II.1 and Phase II.2. At the end of Phase I, we show that the parameters grow to norm $O(r)$ in time $T_1(r)$. In Phase II.1, we extend the dynamic to time $T_1(r)+T_2(r)$ so that the parameters grow into constant norms (irrelevant to $r$ and $\sigmainit$). Then, when the initialization scale becomes sufficiently small, at the end of Phase II.1 the parameters become sufficiently close to the embedded parameters from two neurons at constant norms, so the subsequent dynamics is a good approximate embedding until the norm of the parameters grow sufficiently large to ensure directional convergence in Phase III. Here we show the results in Phase II.1.
\begin{lemma} \label{lm:phase2-1-good}
	For $m \ge 2$, with probability $1-2^{-(m-1)}$ over the random draw of $\barvtheta_0 \sim \Dinit(1)$, the vector $\barvb \in \R^m$ with entries $\barb_k := \frac{\dotpsm{\barvw_k}{\barvmu} + \bara_k}{2\sqrt{m} \normM{\barvtheta_0}}$ defined as in \Cref{lm:phase1-diff} is a good embedding vector.
\end{lemma}
\begin{proof}
	Note that $\barvb$ is a good embedding vector iff $\barvb' = 2\sqrt{m}\normMsm{\barvtheta_0}\barvb$ is a good embedding vector. Recall that $\barvw_k \simiid \Normal(\vzero, \mI)$, $\bara_k \simiid \Normal(0, \cAinit^2)$. By the property of Gaussian variables,
	\[
		\barb'_k = \dotpsm{\barvw_k}{\barvmu} + \bara_k \sim \Normal(0, 1+\cAinit^2).
	\]
	Thus $\barvb' \sim \Normal(\vzero, (1+\cAinit^2)\mI)$. Since it is a continuous probability distribution, $\barvb' \ne \vzero$ with probability $1$. By symmetry and independence, we know that $\Pr[\forall k \in [m]: \barb'_k > 0] = 2^{-m}$ and  $\Pr[\forall k \in [m]: \barb'_k < 0] = 2^{-m}$. So $\barb'$ is a good embedding vector with probability $1 - 2^{-m} - 2^{-m} = 1-2^{-(m-1)}$, and so is $\barvb$.
\end{proof}

\begin{lemma} \label{lm:phase2-1-main}
	Let $T_2(r) := \frac{1}{\lambda_0}\ln \frac{1}{r}$, then $T_{12} := T_1(r) + T_2(r) = \frac{1}{\lambda_0} \ln \frac{1}{\sqrt{m} \sigmainit \normM{\barvtheta_0}}$ regardless the choice of $r$.
	For random draw of $\barvtheta_0 \sim \Dinit(1)$, if $\barvb \in \R^m$ defined as in \Cref{lm:phase1-diff} is a good embedding vector, then there exists $t_0 \in \R$ such that the following holds:
	\begin{enumerate}
		\item For the two-neuron dynamics starting with rescaled initialization in the direction of $\hatvtheta := (\brmsplus[\barb], \brmsplus[\barb]\barvmu, \brmsminus[\barb], \brmsminus[\barb]\barvmu)$, for all $t \in (-\infty, t_0]$, the limit $\tildevtheta(t) := \lim_{r \to 0} \phitheta(r\hatvtheta, T_2(r) + t)$ exists and lies near $e^{\lambda_0 t}\hatvtheta$:
		\[
			\normM{\tildevtheta(t) - e^{\lambda_0 t}\hatvtheta} \le \frac{4m\normMsm{\hatvtheta}^3}{\lambda_0} e^{3\lambda_0 t} = O(e^{3\lambda_0 t}).
		\]
		\item For the $m$-neuron dynamics $\vtheta(t)$, the following holds for all $t \in (-\infty, t_0]$,
		\[
		\lim_{\sigmainit \to 0} \vtheta\left(T_{12} + t\right) = \pi_{\barvb}(\tildevtheta(t)).
		\]
	\end{enumerate}
\end{lemma}

\begin{proof}
	Under \Cref{ass:sym,ass:dotp-mu-x-non-zero}, the maximum value of $\abssm{G(\vw)}$ on $\sphS^{d-1}$ is $\lambda_0$ and is attained at $\pm \barvmu$. Given a good embedding vector $\barvb$, both $\hatvtheta$ and $\hatvtheta_{\pi} := \pi_{\barvb}(\hatvtheta)$ are well-aligned parameter vectors (\Cref{def:well-aligned}) for width-2 and width-$m$ Leaky ReLU nets respectively. By \Cref{thm:phase-2-general}, there exists $t_0 \in \R$ such that the following two limits exist for all $t \in (-\infty, t_0]$:
	\[
		\tildevtheta(t) := \lim_{r \to 0} \phitheta\left(r\hatvtheta, T_2(r) + t\right), \qquad \tildevtheta_{\pi}(t) := \lim_{r \to 0} \phitheta\left(r\hatvtheta_{\pi}, T_2(r) + t\right).
	\]
	Note that by \Cref{lm:two-neuron-gf-embed-gf}, we have $\pi_{\barvb}(\phitheta(r\hatvtheta, T_2(r) + t))$ is a trajectory of gradient flow starting from $r\hatvtheta_{\pi}$. The uniqueness of $\tildevtheta_{\pi}(t)$ (for all possible choices of $\phitheta$) shows that
	\[
		\pi_{\barvb}(\tildevtheta(t)) = \lim_{r \to 0} \pi_{\barvb}\left(\phitheta\left(r\hatvtheta, T_2(r) + t\right)\right) = \tildevtheta_{\pi}(t).
	\]
	By \Cref{lm:phase1-diff}, for $\sigmainit$ small enough, if we choose $r$ so that $\sigmainit = \frac{r^3}{\sqrt{m}\normM{\barvtheta_0}}$, then for some universal constant $C$ we have
	\[
		\normMsm{\vtheta(T_1(r)) - r\hatvtheta_{\pi}} \le \frac{Cr^3}{\sqrt{m}}.
	\]
	Applying \Cref{thm:phase-2-general} proves the following for all $t \in (-\infty, t_0]$:
	\[
		\lim_{\sigmainit \to 0} \phitheta(\vtheta(T_1(r)), T_2(r) + t) = \tildevtheta_{\pi}(t).
	\]
	Therefore $\lim_{\sigmainit \to 0} \vtheta\left(T_{12} + t\right) = \pi_{\barvb}(\tildevtheta(t))$.
\end{proof}

\subsection{Proofs for Phase II.2}

Next, at the end of Phase II.1, $\vtheta(T_{12}+t_0)$ has a constant norm.
Then we show the trajectory convergence with respect to the initialization scale
in Phase II.2.
\begin{lemma}\label{lm:phase-II-2}
    If $\pi_{\barvb}(\tildevtheta(t_0))$ is non-branching and $\lim_{\sigmainit \to 0} \vtheta\left(T_{12} + t_0\right) = \pi_{\barvb}(\tildevtheta(t_0))$ for some constant $t_0$, then for all $t>t_0$, $\lim_{\sigmainit \to 0} \vtheta\left(T_{12} + t\right) = \pi_{\barvb}(\tildevtheta(t))$.
\end{lemma}

We first start with a simple lemma on gradient upper bounds, and then show that the trajectory of gradient flow is Lipschitz with time.
\begin{lemma} \label{lm:grad-ub}
    For every $\vtheta \in \R^D$, $\normtwosm{\vg} \le \normtwosm{\vtheta}$ for all $\vg \in \cpartial \Loss(\vtheta)$.
\end{lemma}
\begin{proof}
    Note that $\abssm{\ell'(q)} \le 1$, $\abssm{\phi'(z)} \le 1$ $\normtwosm{\vx_i} \le 1$, $\abssm{y_i} \le 1$.
    For every $\vtheta \in \OmegaS$ (i.e., no activation function has zero input), by the chain rule \eqref{eq:grad-Loss}, we have
    \begin{align*}
        \normtwo{\frac{\partial \Loss(\vtheta)}{\partial \vw_k}} &= \normtwo{\frac{1}{n}\sum_{i \in [n]} \ell'(q_i(\vtheta)) y_i a_k \phi'(\vw_k^{\top} \vx_i) \vx_i} \le \frac{1}{n} \sum_{i \in [n]} \abssm{a_k} = \abssm{a_k}, \\
        \normtwo{\frac{\partial \Loss(\vtheta)}{\partial a_k}} &= \normtwo{\frac{1}{n}\sum_{i \in [n]} \ell'(q_i(\vtheta)) y_i \phi(\vw_k^{\top} \vx_i)} \le \frac{1}{n} \sum_{i \in [n]} \abssm{\vw_k^{\top}\vx_i} \le \normtwosm{\vw_k}.
    \end{align*}
    So $\normtwosm{\nabla \Loss(\vtheta)} \le \normtwosm{\vtheta}$. We can finish the proof for any $\vtheta \in \R^D$ by taking limits in $\OmegaS$.
\end{proof}

\begin{lemma}\label{lm:path-lip}
 The gradient flow trajectory $\vtheta(T_{12}+t)$,in the interval $t\in [t_0,\tmpts]$ for any $\tmpts>t_0$, is Lipschitz in $t$ with Lipschitz constant $\normtwo{\vtheta(T_{12}+t_0)}e^{(\tmpts-t_0)}$.
\end{lemma}
\begin{proof}
    By \Cref{lm:grad-ub}, $\normtwo{\frac{\dd }{\dd
    t}\vtheta(T_{12}+t)}\leq  \normtwo{\vtheta(T_{12}+t)}$.
    So $\frac{\dd
    \normtwo{\vtheta(T_{12}+t)}}{\dd t} \leq \normtwo{\vtheta(T_{12}+t)}$.
    By \Gronwall's inequality \eqref{eq:gron-2}, 
    $\normtwo{\vtheta(T_{12}+t)}\leq \normtwo{\vtheta(T_{12}+t_0)}e^{t-t_0}$.
    Then $\normtwo{\frac{\dd }{\dd
    t}\vtheta(T_{12}+t)}\leq  \normtwo{\vtheta(T_{12}+t)} \le \normtwosm{vtheta(T_{12} + t_0)} e^{\tmpts - t_0}$.
\end{proof}
Now we are ready to prove \Cref{lm:phase-II-2}.
\begin{proof}[Proof of \Cref{lm:phase-II-2}]
    When $\sigmainit\to 0$, as $\vtheta(T_{12} + t_0) \to
    \pi_{\barvb}(\tildevtheta(t_0))$, we can select a countable sequence
    $(\sigmainit)_i\to 0$ and trajectories $\vtheta_i(T_{12} + t)$ with
    initialization scale $(\sigmainit)_i$. We show that for every $t\geq t_0$,
    there must be $\vtheta_i(T_{12}+t)=\phitheta(\vtheta_i(T_{12}+t_0),t-t_0)\to
    \pi_{\barvb}(\tildevtheta(t))=\phitheta(\pi_{\barvb}(\tildevtheta(t_0)),t-t_0)$.
    
    If this does not hold for some $t = T$, then there must be a limit point $\vq_T$ of points in  $\{\phitheta(\vtheta_i(T_{12}+t_0), T-t_0)\}_{i=1}^{\infty}$ such that $\vq_T\neq \phitheta(\pi_{\barvb}(\tildevtheta(t_0)),T-t_0)$ and a converging subsequence in $\{\phitheta(\vtheta_i(T_{12}+t_0),T-t_0)\}_{i=1}^{\infty}$ to $\vq_T$. Thus WLOG we assume that the sequence is chosen so that
    \[
    	\lim_{i\to\infty}\phitheta(\vtheta_i(T_{12}+t_0), T-t_0)= \vq_T\neq \phitheta(\pi_{\barvb}(\tildevtheta(t_0)), T-t_0).
    \]
    We then show that there is a trajectory of the gradient flow that starts from $\pi_{\barvb}(\tildevtheta(t_0))$ and reaches $\vq_T$ at time $T - t_0$, thereby causing a contradiction to our assumption that $\pi_{\barvb}(\tildevtheta(t_0))$ is non-branching.

    For any pair of $L_0$-Lipschitz continuous functions $\vf, \vg: [t_0, T] \to \R^D$, define the $\normmax$-distance to be $\norminfsm{\vf - \vg} := \sup_{t \in [t_0, T]} \normtwosm{\vf(t) - \vg(t)}$. Note that the space of $L_0$-Lipschitz functions with bounded function values is compact under $\normmax$-distance, and therefore any sequence of functions in this space has a converging subsequence whose limit is also $L_0$-Lipschitz.
    
    Let $C := \sup_i\{\normtwo{\vtheta_i(T_{12}+t_0)}\}$, then as
    $\{\vtheta_i(T_{12}+t_0)\}$ is converging to
    $\pi_{\barvb}(\tildevtheta(t_0))\neq \infty$, $C<\infty$. By
    \Cref{lm:path-lip} we know each trajectory $\vtheta_i(T_{12}+t)$ is
    $(Ce^{T-t_0})$-Lipschitz for $t\in[t_0,T]$. This means we can find a
    subsequence $1 \le i_1 < i_2 < i_3 < \cdots$ that the trajectory
    $\{\vtheta_{i_j}(T_{12} + t)\}$ $\normmax$-converges on $[t_0, T]$ as
    $j\to\infty$. Then the pointwise limit $\vq(t) := \lim_{j\to\infty}
    \vtheta_{i_j}(T_{12}+t)$ exists for every $t \in [t_0, T]$.
    $\vq(t_0)=\pi_{\barvb}(\tildevtheta(t_0))$, $\vq(T)=\vq_T$.
    
    Finally we show that $\vq(t)$ is indeed a valid gradient flow trajectory.
    Notice that $\vq(t)$ is $(Ce^{T-t_0})$-Lipschitz, then by Rademacher theorem
    for $\vq(t)$ is differentiable for a.e.~$t \in [t_0, T]$. We are
    left to show $\vq'(t)\in\cpartial\Loss(\vq(t))$ whenever $\vq$ is differentiable at $t$.
    
    For any $\epsilon>0$ that $[t, t+\epsilon]\subseteq [t_0,T]$, we
    investigate the behaviour of $\vq(t)$ in the $\epsilon$-neighborhood of $t$.
    Let $\Omega_j$ be the set of $\tau \in [t_0, T]$ so that $\frac{\dd}{\dd \tau} \vtheta_{i_k}(T_{12}+\tau) \in -\cpartial \Loss(\vtheta_{i_k}(T_{12} + \tau))$.
    By definition of differential inclusion, $\Omega_j$ has full measure in $[t_0, T]$.
    Define $B_{j, \epsilon}$ be the following closed convex hull:
    \[
    B_{j,\epsilon}=\overline{\conv}\left\{\frac{\dd}{\dd
    \tau}\vtheta_{i_k}(T_{12}+\tau):k\geq
    j,\tau\in[t,t+\epsilon] \cap \Omega_j \right\}.
    \]
    It is easy to see that $B_{j,\epsilon}$ is monotonic with respect to $j$. Then we know
    that for any $j$,
    \[\frac{\vtheta_{i_j}(T_{12}+t+\epsilon)-\vtheta_{i_j}(T_{12}+t)}{\epsilon}=\frac{1}{\epsilon}\int_{t}^{t+\epsilon}
    \frac{\dd}{\dd\tau}\vtheta_{i_j}(T_{12}+\tau)\dd\tau\in B_{j,\epsilon},\]
    Then taking the limits $j\to\infty$, as all $B_{j,\epsilon}$ are closed, we
    know $\frac{\vq(t+\epsilon)-\vq(t)}{\epsilon}\in
    \lim_{j\to\infty}B_{j,\epsilon}$. 
    
    Now let $C_{j,\epsilon}, C_{\epsilon}$ be the following closed convex hull of subgradients:
    \[
        C_{j,\epsilon}=\overline{\conv}\left(\bigcup_{\substack{k\ge j \\ \tau\in[t,t+\epsilon]}} \cpartial\Loss(\vtheta_{i_k}(T_{12}+\tau))\right),
        \quad 
        C_{\epsilon}=\overline{\conv}\left(\bigcup_{\tau\in[t,t+\epsilon]}\cpartial\Loss(\vq(T_{12}+\tau))\right).
    \]
    Then we know $B_{j,\epsilon}\subseteq -C_{j,\epsilon}$ for all $j \ge 1$ and
    $\epsilon>0$. Notice that $C_{j,\epsilon}$ and $C_\epsilon$ are also
    monotonic with respect to $j$ and $\epsilon$ respectively so we can take the
    respective limit. As for $\tau\in[t,t+\epsilon]$,
    $\lim_{j\to\infty}\vtheta_{i_j}(T_{12}+\tau)=\vq(T_{12}+\tau)$, by the
    upper-semicontinuity of $\cpartial\Loss$,  $\lim_{j\to\infty}
    C_{j,\epsilon}\subseteq C_{\epsilon}$. Then
    $\frac{\vq(t+\epsilon)-\vq(t)}{\epsilon} \in
    \lim\limits_{j\to\infty}B_{j,\epsilon}\subseteq
    \lim\limits_{j\to\infty}C_{j,\epsilon}\subseteq C_\epsilon$.
    
    When $t \in [t_0, T)$ and $\vq(t)$ is differential at $t$, we can take the limit $\epsilon\to 0$, and by the upper-semicontinuity of
    $\cpartial\Loss$ again, we have
    \[\vq'(t)\in\lim_{\epsilon\to 0}C_\epsilon\subseteq
    \overline{\conv}(-\cpartial\Loss(\vq(T_{12}+t)))=-\cpartial\Loss(\vq(T_{12}+t))\]
    as $\cpartial\Loss(\vq(T_{12}+t))$ is closed convex for any $t$. Therefore
    $\vq(t)$ is indeed a gradient flow trajectory.
\end{proof}

\begin{proof}[Proof for \Cref{lm:phase2-main}]
	We can prove \Cref{lm:phase2-main} by combining \Cref{lm:phase2-1-good,lm:phase2-1-main,lm:phase-II-2} together. For $-\infty<t \le t_0$,  by \Cref{lm:phase2-1-main}, $\normM{\tildevtheta(t) - e^{\lambda_0 t}\hatvtheta} \le \frac{4m\normMsm{\hatvtheta}^3}{\lambda_0} e^{3\lambda_0 t}$. With 	$\hatvtheta := (\brmsplus[\barb], \brmsplus[\barb]\barvmu, \brmsminus[\barb], \brmsminus[\barb]\barvmu)$, by choosing a threshold $t_s<t_0$ small enough, we can have for any $t\leq t_s$,

	\begin{itemize}
	    \item $\tildea_1(t)\geq e^{\lambda_0 t}\brmsplus[\barb]-\frac{4m\normMsm{\hatvtheta}^3}{\lambda_0} e^{3\lambda_0t}>0$;
	    \item $\tildea_2(t)\leq e^{\lambda_0 t}\brmsminus[\barb]+\frac{4m\normMsm{\hatvtheta}^3}{\lambda_0} e^{3\lambda_0 t}<0$;    \item $\dotp{\tildevw_1(t)}{\vwopt}\geq e^{\lambda_0 t}\brmsplus[\barb]\dotp{\barvmu}{\vwopt}-\frac{4m\normMsm{\hatvtheta}^3}{\lambda_0} e^{3\lambda_0 t}>\frac{8m\normMsm{\hatvtheta}^3}{\lambda_0} e^{3\lambda_0 t}>0$;
	    \item
	    $\dotp{\tildevw_2(t)}{\vwopt}\leq e^{\lambda_0 t}\brmsminus[\barb]\dotp{\barvmu}{\vwopt}+\frac{4m\normMsm{\hatvtheta}^3}{\lambda_0} e^{3\lambda_0 t}<-\frac{8m\normMsm{\hatvtheta}^3}{\lambda_0} e^{3\lambda_0 t}<0$.
	\end{itemize}
	Then $\tildevtheta(t)\neq 0$ for all $t\leq t_s$. For $t>t_s$, we know $\tildevtheta(t)\neq 0$ by applying \Cref{thm:loss-convergence}. Finally by \Cref{lm:phase2-1-main,lm:phase-II-2} we know $\lim_{\sigmainit \to 0} \vtheta\left(T_{12} + t\right) = \pi_{\barvb}(\tildevtheta(t))$ for all $t$.
\end{proof}

\section{Proofs for Phase III}

\subsection{Two Neuron Case: Margin Maximization}

In this subsection we prove \Cref{thm:two-neuron-max-margin} for the symmetric
datasets. By \Cref{thm:loss-convergence} and \Cref{thm:converge-kkt-margin}, we
know that gradient flow must converge in a KKT-margin direction of width-$2$
two-layer Leaky ReLU network (\Cref{def:kkt-margin-lkrelu}). Thus we first give
some characterizations for KKT-margin directions by proving \Cref{lm:cross-ineq}
and \Cref{lm:two-neuron-kkt-equal}.

\begin{lemma} \label{lm:cross-ineq}
	Given $\vu_1, \vu_2 \in \R^d$, if $y_i(\phi(\dotpsm{\vu_1}{\vx_i}) - \phi(-\dotpsm{\vu_2}{\vx_i})) \ge 1$ for all $i \in [n]$, then 
	\[
		y_i (\hai \dotpsm{\vu_2}{\vx_i} + \hbi \dotpsm{\vu_1}{\vx_i}) \ge 1,
	\]
	for all Clarke's sub-differentials $\hai \in \cder{\phi}(\dotpsm{\vu_1}{\vx_i}), \hbi \in \cder{\phi}(-\dotpsm{\vu_2}{\vx_i})$.
\end{lemma}
\begin{proof}
	We prove by cases for any fixed $i \in [n]$. By \Cref{ass:sym} we have
	\[
		y_i(\phi(\dotpsm{\vu_1}{\vx_i}) - \phi(-\dotpsm{\vu_2}{\vx_i})) \ge 1, \qquad -y_i(\phi(-\dotpsm{\vu_1}{\vx_i}) - \phi(\dotpsm{\vu_2}{\vx_i})) \ge 1.
	\]	
	\myparagraph{Case 1.} Suppose that $\dotpsm{\vu_1}{\vx_i} \ne 0$, $\dotpsm{\vu_2}{\vx_i} \ne 0$. Then we have $\hai, \hbi \in \{\alphaLK, 1\}$. If $\hai = \hbi$, then $\hai \dotpsm{\vu_2}{\vx_i} = -\phi(-\dotpsm{\vu_2}{\vx_i})$, $\hbi \dotpsm{\vu_1}{\vx_i} = \phi(\dotpsm{\vu_1}{\vx_i})$, and thus we have
		\begin{align*}
		y_i (\hai \dotpsm{\vu_2}{\vx_i} + \hbi \dotpsm{\vu_1}{\vx_i}) &= y_i (-\phi(-\dotpsm{\vu_2}{\vx_i}) + \phi(\dotpsm{\vu_1}{\vx_i})) \\
		&= y_i (\phi(\dotpsm{\vu_1}{\vx_i})-\phi(-\dotpsm{\vu_2}{\vx_i})) \ge 1.
		\end{align*}
	Otherwise, $\hai \ne \hbi$, then we have $\hai \dotpsm{\vu_2}{\vx_i} = \phi(\dotpsm{\vu_2}{\vx_i})$, $\hbi \dotpsm{\vu_1}{\vx_i} = -\phi(-\dotpsm{\vu_1}{\vx_i})$, and thus
		\begin{align*}
		y_i (\hai \dotpsm{\vu_2}{\vx_i} + \hbi \dotpsm{\vu_1}{\vx_i}) &= y_i \left(\phi(\dotpsm{\vu_2}{\vx_{i}}) - \phi(-\dotpsm{\vu_1}{\vx_{i}})\right) \\
		&= -y_{i} \left(\phi(-\dotpsm{\vu_1}{\vx_{i}}) - \phi(\dotpsm{\vu_2}{\vx_{i}})\right) \ge 1.
		\end{align*}
	
	\myparagraph{Case 2.} Suppose that $\dotpsm{\vu_1}{\vx_i} = 0$ or $\dotpsm{\vu_2}{\vx_i} = 0$. WLOG we assume that $\dotpsm{\vu_1}{\vx_i} = 0$ (the case of $\dotpsm{\vu_2}{\vx_i} = 0$ can be proved similarly). Then we have
	\[
		- y_i\phi(-\dotpsm{\vu_2}{\vx_i})) \ge 1, \qquad y_i\phi(\dotpsm{\vu_2}{\vx_i})) \ge 1.
	\]
	If $\dotpsm{\vu_2}{\vx_i} = 0$, then the feasibility cannot be satisfied. So we must have $\dotpsm{\vu_2}{\vx_i} \ne 0$ and $\hbi \in \{\alphaLK, 1\}$. This implies that $y_i\dotpsm{\vu_2}{\vx_{i}} \ge \frac{1}{\alphaLK}$.
	
	Since $\dotpsm{\vu_1}{\vx_i} = 0$, we have $\hai \in [\alphaLK, 1]$. Therefore,
	\[
		y_i (\hai \dotpsm{\vu_2}{\vx_i} + \hbi \dotpsm{\vu_1}{\vx_i}) = y_i \hai \dotpsm{\vu_2}{\vx_i} \ge y_i \alphaLK \dotpsm{\vu_2}{\vx_i} \ge 1,
	\]
	which completes the proof.
\end{proof}

\begin{lemma} \label{lm:two-neuron-kkt-equal}
	If $(\vw_1, \vw_2, a_1, a_2)$ is along a KKT-margin direction of width-$2$ two-layer Leaky ReLU network and $a_1 > 0, a_2 < 0$, then $\vw_1 = -\vw_2$, $a_1 = -a_2 = \normtwosm{\vw_1}$.
\end{lemma}
\begin{proof}
	WLOG we assume that $\qmin(\vtheta) = 1$. By \Cref{def:kkt-margin-lkrelu}
	and \Cref{lm:a-w-equal}, there exist $\lambda_1, \dots, \lambda_n \ge 0$ and
	$\hai[1], \dots, \hai[n] \in \R, \hbi[1], \dots, \hbi[n] \in \R$ such that
	$\hai \in \cder{\phi}(\dotpsm{\vw_1}{\vx_i})$, $\hbi \in
	\cder{\phi}(\dotpsm{\vw_2}{\vx_i})$, and the following conditions hold:
	\begin{enumerate}
		\item $\vw_1 = a_1 \sum_{i \in [n]} \lambda_i y_i \hai[i] \vx_i$, $\vw_2 = a_2 \sum_{i \in [n]} \lambda_i y_i \hbi[i] \vx_i$;
		\item $a_1 = \normtwosm{\vw_1}$, $a_2 = -\normtwosm{\vw_2}$;
		\item For all $i \in [n]$, if $q_i(\vtheta) \ne 1$ then $\lambda_i = 0$.
	\end{enumerate}
	Let $\vu_1 = a_1 \vw_1$ and $\vu_2 = -a_2 \vw_2$. Let $\barvu_1 := \frac{\vu_1}{\normtwosm{\vu_1}}, \barvu_2 := -\frac{\vu_2}{\normtwosm{\vu_2}}$. Then the following conditions hold for all $i \in [n]$:
	\begin{align}
		\barvu_1 - \sum_{i=1}^{n} \lambda_i \hai y_i \vx_i &= 0, \label{eq:leaky-u1-stat} \\
		\barvu_2 - \sum_{i=1}^{n} \lambda_i \hbi y_i \vx_i &= 0, \label{eq:leaky-u2-stat} \\
		\lambda_i (1 - y_i(\phi(\dotpsm{\vu_1}{\vx_i}) - \phi(-\dotpsm{\vu_2}{\vx_i}))) &= 0. \label{eq:leaky-compl-slack}
	\end{align}
	
	By homogeneity, $\hai \cdot \dotpsm{\vu_1}{\vx_i} = \phi(\dotpsm{\vu_1}{\vx_i}), \hbi \cdot \dotpsm{\vu_2}{\vx_i} = -\phi(-\dotpsm{\vu_2}{\vx_i})$. Left-multiplying $(\vu_1)^\top$ or $(\vu_2)^{\top}$ on both sides of \eqref{eq:leaky-u1-stat}, we have
	\begin{align}
		\normtwosm{\vu_1} - \sum_{i=1}^{n}  \lambda_i y_i \phi(\dotpsm{\vu_1}{\vx_i}) &= 0, \label{eq:leaky-u1-stat-var1} \\
		\dotpsm{\barvu_1}{\barvu_2} \normtwosm{\vu_2} - \sum_{i=1}^{n} \lambda_i \hai y_i \dotpsm{\vu_2}{\vx_i} &= 0. \label{eq:leaky-u1-stat-var2}
	\end{align}
	Similarly, we have
	\begin{align}
		\normtwosm{\vu_2} + \sum_{i=1}^{n}  \lambda_i y_i \phi(-\dotpsm{\vu_2}{\vx_i}) &= 0, \label{eq:leaky-u2-stat-var1} \\
		\dotpsm{\barvu_1}{\barvu_2} \normtwosm{\vu_1} - \sum_{i=1}^{n} \lambda_i \hbi y_i \dotpsm{\vu_1}{\vx_i} &= 0. \label{eq:leaky-u2-stat-var2}
	\end{align}
	Combining \eqref{eq:leaky-u1-stat-var1} and \eqref{eq:leaky-u2-stat-var1}, we have
	\begin{align}
		\normtwosm{\vu_1} + \normtwosm{\vu_2}
		&= \sum_{i=1}^{n} \lambda_i y_i \phi(\dotpsm{\vu_1}{\vx_i}) - \sum_{i=1}^{n}  \lambda_i y_i \phi(-\dotpsm{\vu_2}{\vx_i}) \\
		&= \sum_{i=1}^{n} \lambda_i y_i \left(\phi(\dotpsm{\vu_1}{\vx_i}) -  \phi(-\dotpsm{\vu_2}{\vx_i})\right) \\
		&= \sum_{i=1}^{n} \lambda_i,
	\end{align}
	where the last equality is due to \eqref{eq:leaky-compl-slack}.
	
	Combining \eqref{eq:leaky-u1-stat-var2} and \eqref{eq:leaky-u2-stat-var2}, we have
	\begin{align}
		\dotpsm{\barvu_1}{\barvu_2}(\normtwosm{\vu_1} + \normtwosm{\vu_2}) &= \sum_{i=1}^{n} \lambda_i \hai y_i \dotpsm{\vu_2}{\vx_i} + \sum_{i=1}^{n} \lambda_i \hbi y_i \dotpsm{\vu_1}{\vx_i} \\
		&= \sum_{i=1}^{n} \lambda_i y_i \left(\hai \dotpsm{\vu_2}{\vx_i} + \hbi \dotpsm{\vu_1}{\vx_i}\right) \label{eq:leaky-comb-u1u2-stat-var2-ineq-prev} \\
		&\ge \sum_{i=1}^{n} \lambda_i, \label{eq:leaky-comb-u1u2-stat-var2-ineq}
	\end{align}
	where the last inequality is due to \Cref{lm:cross-ineq}. Since we have deduced that $\normtwosm{\vu_1} + \normtwosm{\vu_2} = \sum_{i=1}^{n} \lambda_i$, we further have
	\[
		\dotpsm{\barvu_1}{\barvu_2}(\normtwosm{\vu_1} + \normtwosm{\vu_2}) \ge \normtwosm{\vu_1} + \normtwosm{\vu_2}.
	\]
	Combining this with $\dotpsm{\barvu_1}{\barvu_2} \le \normtwosm{\barvu_1} \normtwosm{\barvu_2} \le 1$, we have $1 \le \dotpsm{\barvu_1}{\barvu_2} \le 1$. So all the inequalities become equalities, and thus $\barvu_1 = \barvu_2$. \eqref{eq:leaky-comb-u1u2-stat-var2-ineq-prev} also equals to \eqref{eq:leaky-comb-u1u2-stat-var2-ineq}, so
	\begin{equation} \label{eq:leaky-cross-eq}
		y_i \left(\hai \dotpsm{\vu_2}{\vx_i} + \hbi \dotpsm{\vu_1}{\vx_i}\right) = 1,
	\end{equation}
	whenever $\lambda_i \ne 0$.
	
	By \eqref{eq:leaky-compl-slack}, we have $y_i \left(\hai \dotpsm{\vu_1}{\vx_i} + \hbi \dotpsm{\vu_2}{\vx_i}\right) = 1$ whenever $\lambda_i \ne 0$. Combining this with \eqref{eq:leaky-cross-eq}, we have
	\[
		y_i (\hai - \hbi) \dotpsm{\vu_1}{\vx_i} = y_i (\hai - \hbi) \dotpsm{\vu_2}{\vx_i}.
	\]
	Then we prove that $\dotpsm{\vu_1}{\vx_i} = \dotpsm{\vu_2}{\vx_i}$ by discussing two cases:
	\begin{enumerate}
		\item If $\dotpsm{\vu_1}{\vx_i} = 0$, then $\dotpsm{\vu_2}{\vx_i} = 0$ since $\barvu_1 = \barvu_2$;
		\item Otherwise, we have $(\hai, \hbi) = (1, \alphaLK)$ or $(\alphaLK, 1)$ by symmetry, so $\hai \ne \hbi$ and thus $\dotpsm{\vu_1}{\vx_i} = \dotpsm{\vu_2}{\vx_i}$.
	\end{enumerate}
	This means $\vu_1$ and $\vu_2$ have the same projection onto the linear space spanned by $\{ \vx_i :  \lambda_i \ne 0\}$. By \eqref{eq:leaky-u1-stat} and \eqref{eq:leaky-u2-stat}, $\vu_1$ and $\vu_2$ are in the span of $\{ \vx_i : i \in [n], \lambda_i \ne 0\}$. Therefore, $\vu_1 = \vu_2$ and we can easily deduce that $\vw_1 = -\vw_2$, $a_1 = -a_2 = \normtwosm{\vw_1}$.
\end{proof}

\begin{lemma} \label{lm:two-neuron-kkt-all}
	If $\vtheta = (\vw_1, \vw_2, a_1, a_2)$ is along a KKT-margin direction of width-$2$ two-layer Leaky ReLU network and $\normtwo{\vtheta} = 1$, $a_1\ge 0$ and $a_2\le 0$, then one of the following three cases is true:
	\begin{enumerate}
		\item $\vtheta = \frac{1}{2}(\vwopt, -\vwopt, 1, -1)$;
		\item $\vtheta = \frac{1}{\sqrt{2}}(\vwopt, \vzero, 1, 0)$;
		\item $\vtheta = \frac{1}{\sqrt{2}}(\vzero, -\vwopt, 0, -1)$.
	\end{enumerate}
\end{lemma}
\begin{proof}
	Suppose $a_1>0$ and $a_2<0$, then by \Cref{lm:two-neuron-kkt-equal}, we know  $\vw_1 = -\vw_2$, $a_1 = -a_2 = \normtwosm{\vw_1}$. Since $q_i(\vtheta)>0,\forall i$, we know $\dotp{\vw_1}{\vx_i}\neq 0,\forall i$, which implies $q_i(\vtheta)$ is differentiable at $\vtheta$. Let $\vtheta' =(\vw_1,a_1)$ and $[\vtheta'; -\vtheta']= (\vw_1,-\vw_1, a_1, -a_1)$, we know $\vtheta'$ is along the KKT direction of the following optimization problem:
		\begin{align*}
		\min          & \quad f([\vtheta';-\vtheta']) \\
		\text{s.t.}   & \quad g_i([\vtheta';-\vtheta']) \le 0, \qquad \forall i \in [n],
	\end{align*}
	where $f([\vtheta';-\vtheta']) = \normtwo{[\vtheta';-\vtheta']}^2 = 2\normtwo{\vw_1}^2+2a_1^2$, and $g_i([\vtheta';-\vtheta']) = 1- q_i([\vtheta';-\vtheta']) = y_ia_i(\phi(\dotp{\vw_1}{\vx_i})- \phi(-\dotp{\vw_1}{\vx_i}))=a_1(1+\alphaLK)\dotp{\vw_1}{y_i\vx_i}$. With a standard analysis, we know $\vw_1$ be in the direction of the max-margin classifier of the original problem, $\vwopt$.
	
	Next we discuss the case where $a_2=0$ ($a_1=0$ follows the same analysis). When $a_2 =0$, since $q_i(\vtheta)>0$ for all $i$, we know $a_1y_i\dotp{\vx_i}{\vx_1}>0$ for all $i$. Thus $q_i(\vtheta)>q_{i+\frac{n}{2}}(\vtheta)=\alphaLK q_i(\vtheta)$, which means only the second half constraints might be active. This reduces the optimization problem to a standard linear-max-margin problem, and $\vw_1$ will be aligned with $\vwopt$.
\end{proof}

\begin{proof}[Proof for \Cref{thm:two-neuron-max-margin}]
	By \Cref{thm:loss-convergence} and \Cref{thm:converge-kkt-margin}, we know $\lim_{t \to +\infty} \frac{\vtheta(t)}{\normtwosm{\vtheta(t)}} $ must be along a KKT-margin direction. By \Cref{lm:two-neuron-kkt-all}, we know that there are only $3$ KKT-margin directions:
	\[
		\frac{1}{2}(\vwopt, -\vwopt, 1, -1), \quad\frac{1}{\sqrt{2}}(\vwopt, \bm{0}, 1, 0), \quad\frac{1}{\sqrt{2}}(\bm{0}, -\vwopt, 0, -1).
	\]
	Thus it suffices to show $\lim_{t \to +\infty} \frac{\vtheta(t)}{\normtwosm{\vtheta(t)}} \neq \frac{1}{\sqrt{2}}(\vwopt, \bm{0}, 1, 0)$. ($\lim_{t \to +\infty} \frac{\vtheta(t)}{\normtwosm{\vtheta(t)}} \neq \frac{1}{\sqrt{2}}(\vwopt, \bm{0}, 1, 0)$ would hold for the same reason.) 
	
	For convenience, we define $i' := i + n/2$ if $1 \le i \le n/2$ and $i' := i - n /2$ if $n/2 < i \le n$. By \Cref{ass:sym} we know that $\vx_{i'} = -\vx_i$ and $y_{i'} = -y_i$.
	
	We first define the angle between $\vwopt$ and $\vw_1(t)$ as $\beta_1(t):=\arccos \frac{\dotp{\vwopt}{\vw_1(t)}}{\normtwosm{\vw_1(t)}}$ and angle between $-\vwopt$ and $\vw_2(t)$ as $\beta_2(t):=\arccos \frac{\dotp{-\vwopt}{\vw_2(t)}}{\normtwosm{\vw_2(t)}}$. Since $\dotp{\vwopt}{\vw_1(0)} > 0$ and $\dotp{-\vwopt}{\vw_2(0)} > 0$, by \Cref{lm:dotp-vw-vwopt} we know that $\beta_1(t), \beta_2(t) \in [0, \pi/2)$ for all $t \ge 0$.
	
	We also define $\eps := \min_{i\in[n]} \left\{ \arcsin \frac{\dotp{y_i\vx_i}{\vwopt}}{\normtwosm{\vx_i}} \right\}$, which can be understood as the angle between $\vx_i$ and the decision boundary determined by the linear separator $\vwopt$.
	
	Below we will prove by contradiction.
	Suppose $\lim_{t \to +\infty} \frac{\vtheta(t)}{\normtwo{\vtheta(t)}} = \frac{1}{\sqrt{2}}(\vwopt, \vzero, 1, 0) =: \barvtheta_{\infty}$ holds. Then $\beta_1(t) \to 0$ and $\frac{\normtwosm{\vw_2(t)}}{\normtwosm{\vw_1(t)}} \to 0$ as $t \to +\infty$. Thus there must exist $T_1 > 0$ such that $\beta_1(t) \le \epsilon/2$.

	Note that $f_{\barvtheta_{\infty}}(\vx_i) = \frac{1}{2} \phi(\dotpsm{\vx_i}{\vwopt})$ for all $i \in [n]$. By symmetry, for $i \in [n/2]$ we have
	\[
		q_i(\barvtheta_{\infty}) - q_{i'}(\barvtheta_{\infty}) = f_{\barvtheta_{\infty}}(\vx_i) - f_{\barvtheta_{\infty}}(-\vx_i) = \frac{1}{2} \dotpsm{\vx_i}{\vwopt} - \frac{\gammaopt}{2}\dotpsm{\vx_i}{\vwopt} = \frac{1 - \gammaopt}{2} \dotpsm{\vx_i}{\vwopt} > 0.
	\]
	By \Cref{thm:loss-convergence}, we also know $\normtwosm{\vtheta(t)} \to \infty$, so $q_i(\vtheta) - q_{i'}(\vtheta) = \normtwosm{\vtheta(t)}(q_i(\barvtheta_{\infty}) - q_{i'}(\barvtheta_{\infty})) \to +\infty$ for all $i \in [n/2]$. Let $g_i(\vtheta) := -\ell'(q_i(\vtheta))$. Then
	\[
		\frac{g_i(\vtheta(t))}{g_{i'}(\vtheta(t))} \sim \frac{\exp(-g_i(\vtheta(t)))}{\exp(-g_{i'}(\vtheta(t)))} = e^{-(g_i(\vtheta(t))-g_{i'}(\vtheta(t)))} \to 0.
	\]
	Thus there must exist $T_2 > 0$ such that $\frac{g_i(\vtheta(t))}{g_{i'}(\vtheta(t))} \le \max\left\{\frac{\cos \eps-\alphaLK}{1-\alphaLK \cos \eps}, 1\right\}$ for all $i \in [n/2]$ and $t \ge T_2$.
	
	We will use these to show that $\frac{\dotp{\vw_{2}(t)}{-\vwopt}}{\dotp{\vw_{1}(t)}{\vwopt}}$ is non-decreasing for $t\ge T:=\max\{T_1,T_2\}$, which further implies $\frac{\normtwosm{\vw_{2}(t)}}{\normtwosm{\vw_{1}(t)}}$ is lower bounded by some constant. Thus it contradicts with the assumption of convergence. 
	
	By \Cref{cor:weight-equal}, we know that $a_1(t) = \normtwosm{\vw_1(t)}$ and $a_2(t) = -\normtwosm{\vw_2(t)}$ for all $t \ge 0$. Then for all $i \in [n]$, we have
	\[
		f_{\vtheta}(\vx_i) = a_1 \phi(\vw_1^{\top} \vx_i) + a_2 \phi(\vw_2^{\top} \vx_i) = \normtwosm{\vw_1} \phi(\vw_1^{\top} \vx_i) - \normtwosm{\vw_2} \phi(\vw_2^{\top} \vx_i).
	\]

	By \eqref{eq:grad-Loss}, if $\vw_1(t), \vw_2(t) \in \OmegaS$ then we have
	\begin{align*}
		\frac{\dd \vw_1}{\dd t} &= \frac{\normtwosm{\vw_1}}{n}\sum_{i \in [n]} g_i(\vtheta)  \phi'(\vw_1^{\top} \vx_i) y_i\vx_i, & -\frac{\dd \vw_2}{\dd t} &=  \frac{\normtwosm{\vw_2}}{n}\sum_{i \in [n]} g_i(\vtheta)  \phi'(\vw_2^{\top} \vx_i) y_i\vx_i.
	\end{align*}
	By symmetry, we can rewrite them as
	\begin{align}
	\frac{\dd \vw_1(t)}{\dd t} &= \frac{\normtwosm{\vw_1(t)}}{n}\sum_{i \in [n/2]} \sai(t) \vx_i, & -\frac{\dd \vw_2(t)}{\dd t} &=  \frac{\normtwosm{\vw_2(t)}}{n}\sum_{i \in [n/2]} \sbi(t) \vx_i. \label{eq:two-neuron-max-margin-sym-dyn}
	\end{align}
	where $\ski(t) := g_i(\vtheta(t)) \phi'(\vw_k^{\top}(t)\vx_i) + g_{i'}(\vtheta(t)) \phi'(-\vw_k^{\top}(t)\vx_i)$. Note that this only holds for $\vw_k(t) \in \OmegaS$. By taking limits through \eqref{eq:def-clarke2}, we know that for a.e.~$t \ge 0$, there exists $\ski(t)$ such that \eqref{eq:two-neuron-max-margin-sym-dyn} holds and
	\begin{equation} \label{eq:ski}
		\ski(t) \in \begin{cases}
		\{ g_i(\vtheta) + \alphaLK g_{i'}(\vtheta) \} & \qquad \text{if }\vw_k^{\top} \vx_i > 0; \\
		\{\alphaLK g_i(\vtheta) + g_{i'}(\vtheta)\} & \qquad \text{if }\vw_k^{\top} \vx_i < 0; \\
		\{\lambda g_i(\vtheta) + (1 + \alphaLK - \lambda) g_{i'}(\vtheta)  : \alphaLK \le \lambda \le 1 \} & \qquad \text{if }\vw_k^{\top} \vx_i = 0.
		\end{cases}
	\end{equation}
	
	By chain rule, for a.e.~$t \ge 0$ we have:
	\begin{align*}
	\frac{\dd}{\dd t}\ln \frac{\dotpsm{\vw_{2}}{-\vwopt}}{\dotpsm{\vw_{1}}{\vwopt}}  
	&= \frac{\dotpsm{\frac{\dd \vw_{2}}{\dd t}}{-\vwopt}}{\dotpsm{\vw_{2}}{-\vwopt}} - \frac{\dotpsm{\frac{\dd \vw_{1}}{\dd t}}{\vwopt}}{\dotpsm{\vw_{1}}{\vwopt}} \\
	&= \frac{\normtwosm{\vw_2}}{\dotpsm{\vw_{2}}{-\vwopt}} \cdot \frac{1}{n}\sum_{i \in [n/2]} \sbi \dotpsm{\vx_i}{\vwopt} - \frac{\normtwosm{\vw_1}}{\dotpsm{\vw_1}{\vwopt}} \cdot \frac{1}{n}\sum_{i \in [n/2]} \sai \dotpsm{\vx_i}{\vwopt} \\
	&= \frac{1}{n} \sum_{i \in [n/2]} \left( \frac{\sbi}{\cos \beta_2} - \frac{\sai}{\cos \beta_1} \right) \dotpsm{\vx_i}{\vwopt}.
	\end{align*}
	Now we are ready to prove $\frac{\dd}{\dd t}\ln \frac{\dotpsm{\vw_{2}}{-\vwopt}}{\dotpsm{\vw_{1}}{\vwopt}} \ge 0$ for $t \ge T$. For this, we only need to show that $\frac{\sbi}{\cos \beta_2} \ge \frac{\sai}{\cos \beta_1}$ in two cases.
	\begin{enumerate}
		\item[Case 1.] When $\beta_1 < \beta_2$, it suffices to show $\sai \le \sbi$. By our choice of $T_1$, we have $\beta_1 \le \eps / 2$, which implies $\vw_1^{\top} \vx_i > 0$ and $\sai = g_i(\vtheta) + \alphaLK g_{i'}(\vtheta)$ for all $i \in [n/2]$. Note that $g_i(\vtheta)) \le g_{i'}(\vtheta)$ according to our choice of $T_2$. Then for any  $\lambda \in [\alphaLK, 1]$ we have
		\[
			\sai = g_i(\vtheta) + \alphaLK g_{i'}(\vtheta) \le \lambda g_i(\vtheta) + (1+\alphaLK-\lambda) g_{i'}(\vtheta).
		\]
		By \eqref{eq:ski}, we therefore have $\sai \le \sbi$.
		\item[Case 2.] If $\beta_1 \ge \beta_2$, then by our choice of $T_1$ we have $\epsilon / 2 \ge \beta_1 \ge \beta_2$. Then for all $i \in[n/2]$, $\vw_2^{\top} \vx_i \le 0$. So we have 
		\begin{align*}
		\frac{\sai}{\sbi} = 
		\frac{g_i(\vtheta) + \alphaLK g_{i'}(\vtheta)}{\alphaLK g_i(\vtheta) + g_{i'}(\vtheta)} = \frac{\frac{g_i(\vtheta)}{g_{i'}(\vtheta)} + \alphaLK } {\alphaLK \frac{g_i(\vtheta)}{g_{i'}(\vtheta)} + 1} \le \cos \eps \le \cos \beta_1(t)\le \frac{\cos \beta_1(t)}{\cos\beta_2(t)}.
		\end{align*}
		Thus $\frac{\sbi}{\cos \beta_2} \ge \frac{\sai}{\cos \beta_1}$.
	\end{enumerate}
	
	Now we have shown that $\frac{\dotpsm{\vw_{2}(t)}{-\vwopt}}{\dotpsm{\vw_{1}(t)}{\vwopt}}\ge \frac{\dotpsm{\vw_{2}(T)}{-\vwopt}}{\dotpsm{\vw_{1}(T)}{\vwopt}}=: r_0$, where $r_0$ is a constant (ratio at time $T$). So for $t \ge T$,
	\begin{align}
	\frac{\normtwosm{\vw_2(t)}}{\normtwosm{\vw_1(t)}} = \frac{\dotp{\vw_{2}(t)}{-\vwopt} \cos \beta_1(t)}{\dotpsm{\vw_{1}(t)}{\vwopt} \cos \beta_2(t)} \ge \frac{\dotpsm{\vw_2(t)}{-\vwopt} \cos \eps}{\dotpsm{\vw_1(t)}{\vwopt}}\ge r_0 \cos \eps,
	\end{align}
	is lower bounded, which contradicts with $\lim_{t \to +\infty} \frac{\vtheta(t)}{\normtwo{\vtheta(t)}} = \frac{1}{\sqrt{2}}(\vwopt, \vzero, 1, 0)$.
\end{proof}

\subsection{Directional Convergence of \texorpdfstring{$L$}{L}-homogeneous Neural Nets}

In this section we consider general $L$-homogeneous neural nets with logistic
loss following the settings introduced in \Cref{sec:review-margin-maximization}.
We define $\alpha(\vtheta)$ and $\tildegamma(\vtheta)$ to be smoothed margin and
its normalized version following \citet{lyu2020gradient}.
\[
	\alpha(\vtheta) = \ell^{-1}(n\Loss(\vtheta)), \qquad \tildegamma(\vtheta) = \frac{\alpha(\vtheta)}{\normtwosm{\vtheta}^L}.
\]
Define $\zeta(t) := \int_{0}^{t} \normtwo{\frac{\dd}{\dd \tau} \frac{\vtheta(\tau)}{\normtwosm{\vtheta(\tau)}}} \dd \tau$ to be the length of the trajectory swept by $\vtheta/\normtwosm{\vtheta}$ from time $0$ to $t$. Define $\beta(t)$ to be the cosine of the angle between $\vtheta(t)$ and $\frac{\dd \vtheta(t)}{\dd t}$.
\[
	\beta(t) := \frac{\dotp{\frac{\dd \vtheta(t)}{\dd t}}{\vtheta(t)}}{\normtwo{\frac{\dd \vtheta(t)}{\dd t}} \cdot \normtwo{\vtheta(t)}}, \qquad \text{for a.e. } t \ge 0.
\]

\subsubsection{Lemmas from Previous Works}

We leverage the following two lemmas from \citet{ji2020directional} on
desingularizing function. Formally, we say that $\Psi: [0, \nu)$ is a
desingularizing function if $\Psi$ is continuous on $[0, \nu)$ with $\Psi(0) =
0$ and continuously differentiable on $(0, \nu)$ with $\Psi' > 0$.
\begin{lemma}[Lemma 3.6, \citealt{ji2020directional}] \label{lm:psi-a}
	Given a locally Lipschitz definable function $f$ with an open domain $D \subseteq \{\vtheta : \normtwosm{\vtheta} > 1\}$, for any $c, \eta > 0$, there exists $\nu > 0$ and a definable desingularizing function $\Psi$ on $[0, \nu)$ such that
	\[	
		\Psi'(f(\vtheta)) \cdot \normtwosm{\vtheta} \normtwo{\barcpartial f(\vtheta)} \ge 1,
	\]
	whenever $f(\vtheta) \in (0, \nu)$ and $\normtwo{\barcpartialperp f(\vtheta)} \ge c \normtwosm{\vtheta}^{\eta} \normtwo{\barcpartialr f(\vtheta)}$.
\end{lemma}

\begin{lemma}[Corollary of Lemma 3.7, \citealt{ji2020directional}] \label{lm:psi-b}
	Given a locally Lipschitz definable function $f$ with an open domain $D \subseteq \{\vtheta : \normtwosm{\vtheta} > 1\}$, for any $\lambda > 0$, there exists $\nu > 0$ and a definable desingularizing function $\Psi$ on $[0, \nu)$ such that
	\[
		\Psi'(f(\vtheta)) \cdot \normtwosm{\vtheta}^{1+\lambda} \normtwo{\barcpartial f(\vtheta)} \ge 1,
	\]
	whenever $f(\vtheta) \in (0, \nu)$.
\end{lemma}

For $\tildegamma(\vtheta)$, we have the following decomposition lemma from \citet{ji2020directional}.
\begin{lemma}[Lemma 3.4, \citealt{ji2020directional}] \label{lm:tildegamma-speed-decompose}
	If $\Loss(\vtheta(t)) < \ell(0)/n$ at time $t = t_0$, it holds for a.e.~$t \ge t_0$ that
	\[
		\frac{\dd \tildegamma(\vtheta(t))}{\dd t} = \normtwo{\barcpartialr \tildegamma(\vtheta(t))} \normtwo{\barcpartialr \Loss(\vtheta(t))} + \normtwo{\barcpartialperp \tildegamma(\vtheta(t))} \normtwo{\barcpartialperp \Loss(\vtheta(t))}.
	\]
\end{lemma}

For $a \in \R \cup \{+\infty, -\infty\}$, we say that $v$ is an asymptotic Clarke critical value of a locally Lipschitz function $f: \R^D \to \R$ if there exists a sequence of $(\vtheta_j, \vg_j)$, where $\vtheta_j \in \R^D$ and $\vg_j \in \cpartial f(\vtheta_j)$, such that $\lim_{j \to +\infty} f(\vtheta_j) = v$ and $\lim_{j \to +\infty} (1 + \normtwosm{\vtheta_j}) \normtwosm{\vg_j} = 0$.

\begin{lemma}[Corollary of Lemma B.10, \citealt{ji2020directional}] \label{lm:tilde-gamma-finite-critical}
	$\tildegamma(\vtheta)$ only has finitely many asymptotic Clarke critical values.
\end{lemma}

For $\beta(\vtheta)$, we have the following lemma from \citet{lyu2020gradient}.
\begin{lemma}[Lemma C.12, \citealt{lyu2020gradient}] \label{lm:beta-converge}
	If $\Loss(\vtheta(t)) < \ell(0)/n$ at time $t = t_0$, then there exists a sequence $t_1, t_2, \dots$ such that $t_j \to +\infty$ and $\beta(t_j) \to 1$ as $j \to +\infty$.
\end{lemma}

\subsubsection{Characterizing Margin Maximization with Asymptotic Clarke Critical Value}

Before proving \Cref{thm:phase-4-general}, we first prove the following theorem that characterizes margin maximization using asymptotic Clarke critical value.

\begin{theorem} \label{thm:converge-to-asymptotic-clarke-critical}
	For homogeneous nets, if $\Loss(\vtheta(0)) < \ell(0)/n$, then $\frac{\vtheta(t)}{\normtwosm{\vtheta(t)}}$ converges to some direction $\barvtheta$ and $\gamma(\barvtheta)$ is an asymptotic Clarke critical value of $\tildegamma$.
\end{theorem}
\begin{proof}
	Note that \Cref{thm:converge-kkt-margin} already implies that $\frac{\vtheta(t)}{\normtwosm{\vtheta(t)}}$ converges to some direction $\barvtheta$. We only need to show that $\gamma(\barvtheta)$ is an asymptotic Clarke critical value of $\tildegamma$.
	
	By \Cref{lm:beta-converge} and definition of $\beta$, there exists a sequence of $(\vtheta_j, \vh_j)$, where $\vh_j \in -\cpartial \Loss(\vtheta_j)$, such that $\tildegamma(\vtheta_j) \to \gamma(\barvtheta)$, $\normtwosm{\vtheta_j} \to +\infty$, $\frac{\vtheta_j}{\normtwosm{\vtheta_j}} \to \barvtheta$, $\frac{\dotp{\vh_j}{\vtheta_j}}{\normtwosm{\vh_j} \normtwosm{\vtheta_j}} \to 1$ as $j \to +\infty$. By chain rule, we know that
	\begin{align*}
		\cpartial \tildegamma(\vtheta) = \frac{\cpartial \alpha(\vtheta)}{\normtwosm{\vtheta}^L} + \frac{L\alpha(\vtheta)}{\normtwosm{\vtheta}^{L+2}} \vtheta &= \frac{n (\ell^{-1})'(n\Loss(\vtheta)) \cpartial \Loss(\vtheta)}{\normtwosm{\vtheta}^L} + \frac{L\alpha(\vtheta)}{\normtwosm{\vtheta}^{L+2}} \vtheta \\
		&= \frac{n \cpartial \Loss(\vtheta)}{\normtwosm{\vtheta}^L \ell'(\alpha(\vtheta))} + \frac{L\alpha(\vtheta)}{\normtwosm{\vtheta}^{L+2}} \vtheta.
	\end{align*}
	This means $\vg_j := \frac{L\alpha(\vtheta_j)}{\normtwosm{\vtheta_j}^{L+2}} \vtheta_j- \frac{n \vh_j}{\normtwosm{\vtheta_j}^L \ell'(\alpha(\vtheta_j))} \in \cpartial \tildegamma(\vtheta)$. By definition of asymptotic Clarke critical value, it suffices to show that $\normtwosm{\vtheta_j} \cdot \normtwosm{\vg_j} \to 0$ as $j \to +\infty$.
	
	By Lemma C.5 in \citep{ji2020directional}, $\abs{\frac{n\dotp{-\vh_j}{\vtheta_j}}{L \cdot \ell'(\alpha(\vtheta_j))} - \alpha(\vtheta)} \le 2 \ln n + 1$. So
	\[
		\lim_{j \to +\infty} \abs{\frac{n\dotp{-\vh_j}{\vtheta_j}}{L \cdot \normtwosm{\vtheta_j}^L \ell'(\alpha(\vtheta_j))} - \tildegamma(\vtheta_j)} = \lim_{j \to +\infty} \frac{1}{\normtwosm{\vtheta_j}^L} \abs{\frac{n\dotp{-\vh_j}{\vtheta_j}}{L \cdot \ell'(\alpha(\vtheta_j))} - \alpha(\vtheta_j)} = 0,
	\]
	which implies that $\lim_{j \to +\infty} \frac{n\dotp{-\vh_j}{\vtheta_j}}{L \cdot \normtwosm{\vtheta_j}^L \ell'(\alpha(\vtheta_j))} = \lim_{j \to +\infty} \tildegamma(\vtheta_j) = \gamma(\barvtheta)$. Now for the radial component of $\vg_j$ we have
	\[
		\normtwosm{\vtheta_j} \cdot \abs{\dotp{\vg_j}{\frac{\vtheta_j}{\normtwosm{\vtheta_j}}}} = \frac{n \dotp{\vh_j}{\vtheta_j}}{\normtwosm{\vtheta_j}^{L} \ell'(\alpha(\vtheta_j))} + \frac{L\alpha(\vtheta_j)}{\normtwosm{\vtheta_j}^{L}} \to -L\gamma(\barvtheta) + L\gamma(\barvtheta) = 0.
	\]
	And for the tangential component we have
	\begin{align*}
		\normtwosm{\vtheta_j} \cdot \normtwo{\left(\mI - \frac{\vtheta_j\vtheta_j^{\top}}{\normtwosm{\vtheta_j}^2}\right)\vg_j} &= \frac{n}{\normtwosm{\vtheta_j}^{L-1} \ell'(\alpha(\vtheta_j))} \normtwo{\left(\mI - \frac{\vtheta_j\vtheta_j^{\top}}{\normtwosm{\vtheta_j}^2}\right)\vh_j} \\
		& = \frac{n \normtwosm{\vh_j}}{\normtwosm{\vtheta_j}^{L-1} \ell'(\alpha(\vtheta_j))} \sqrt{1 - \frac{\dotpsm{\vtheta_j}{\vh_j}^2}{\normtwosm{\vtheta_j}^2\normtwosm{\vh_j}^2}} \\
		&= \frac{n\dotp{-\vh_j}{\vtheta_j}}{\normtwosm{\vtheta_j}^{L} \ell'(\alpha(\vtheta_j))} \frac{\normtwosm{\vh_j}\normtwosm{\vtheta_j}}{\dotp{-\vh_j}{\vtheta_j}} \sqrt{1 - \frac{\dotpsm{\vtheta_j}{\vh_j}^2}{\normtwosm{\vtheta_j}^2\normtwosm{\vh_j}^2}} \\
		&\to L\gamma(\barvtheta) \cdot 1 \cdot 0 = 0.
	\end{align*}
	Combining these proves that $\normtwosm{\vtheta_j} \cdot \normtwosm{\vg_j} \to 0$.
\end{proof}

\subsubsection{Proof for Theorem \ref{thm:phase-4-general}}

Given \Cref{lm:psi-a,lm:psi-b} from \citet{ji2020directional}, we have the following inequality around any direction.

\begin{lemma} \label{lm:phase-4-main-kl}
	Given any parameter direction $\barvthetaopt \in \sphS^{D-1}$, for any $\kappa \in (L/2, L)$, there exists $\nu > 0$ and a definable desingularizing function $\Psi$ on $[0, \nu)$ such that the following holds.
	\begin{enumerate}
		\item For any $\vtheta$, if $\gamma(\barvthetaopt) - \tildegamma(\vtheta) \in (0, \nu)$ and
		\begin{equation} \label{eq:phase-4-main-kl-a-cond}
		\normtwo{\barcpartialperp \tildegamma(\vtheta)} \ge \frac{\gamma(\barvthetaopt)}{4 \ln n + 2} \normtwosm{\vtheta}^{L-\kappa} \normtwo{\barcpartialr \tildegamma(\vtheta)},
		\end{equation}
		then
		\begin{equation} \label{eq:phase-4-main-kl-a}
		\Psi'(\gamma(\barvthetaopt) - \tildegamma(\vtheta)) \cdot \normtwosm{\vtheta} \normtwo{\barcpartial \tildegamma(\vtheta)} \ge 1.
		\end{equation}
		\item For any $\vtheta$, if $\gamma(\barvthetaopt) - \tildegamma(\vtheta) \in (0, \nu)$,
		\begin{equation} \label{eq:phase-4-main-kl-b}
		\Psi'(\gamma(\barvthetaopt) - \tildegamma(\vtheta)) \cdot \normtwosm{\vtheta}^{2\kappa - L + 1} \normtwo{\barcpartial \tildegamma(\vtheta)} \ge 1.
		\end{equation}
	\end{enumerate}
\end{lemma}
\begin{proof}
	Applying \Cref{lm:psi-a} with $f(\vtheta) = \gamma(\barvthetaopt) - \tildegamma(\vtheta), c = \frac{\gamma(\barvthetaopt)}{4 \ln n + 2}, \eta = L - \kappa$, we know that there exists $\nu_1 > 0$ and a definable desingularizing function $\Psi_1$ on $[0, \nu_1)$ such that Item 1 holds for $\Psi_1$, i.e.,
	\[
	\Psi_1'(\gammaopt - \tildegamma(\vtheta)) \cdot \normtwosm{\vtheta} \normtwo{\barcpartial \tildegamma(\vtheta)} \ge 1,
	\]
	whenever \eqref{eq:phase-4-main-kl-a-cond} holds.
	
	Applying \Cref{lm:psi-b} with $f(\vtheta) = \gamma(\barvthetaopt) - \tildegamma(\vtheta), \lambda = 2\kappa - L$, we know that there exists $\nu_2 > 0$ and a definable desingularizing function $\Psi_2$ on $[0, \nu_2)$ such that Item 2 holds for $\Psi_2$, i.e.,
	\[
	\Psi'_2(\gamma(\barvthetaopt) - \tildegamma(\vtheta)) \cdot \normtwosm{\vtheta}^{2\kappa - L + 1} \normtwo{\barcpartial \tildegamma(\vtheta)} \ge 1.
	\]
	Since $\Psi'_1(x) - \Psi'_2(x)$ is definable, there exists a sufficiently small constant $\nu > 0$ such that either $\Psi'_1(x) - \Psi'_2(x) \ge 0$ holds for all $x \in [0, \nu)$, or $\Psi'_1(x) - \Psi'_2(x) \le 0$ holds for all $x \in [0, \nu)$. This means either $\Psi'_1(x) \ge \Psi'_2(x)$ for all $x \in [0, \nu)$ or $\Psi'_2(x) \ge \Psi'_1(x)$ for all $x \in [0, \nu)$. Let $\Psi(x) = \Psi_1(x)$ in the former case and $\Psi(x) = \Psi_2(x)$ in the latter case. Then $\Psi'(x) \ge \Psi'_1(x)$ and $\Psi'(x) \ge \Psi'_2(x)$, and thus both Items 1 and 2 hold.
\end{proof}

Now we prove the following lemma, which will directly lead to
\Cref{thm:phase-4-general}. The core idea of the proof is essentially the same
as that for Lemma 3.3 in \citet{ji2020directional}. The key difference here is
that the desingularizing function $\Psi$ in their lemma has dependence on the
initial point, while our lemma does not have such dependence.
\begin{lemma} \label{lm:phase-4-main-dyn}
	Consider any $L$-homogeneous neural networks with definable output $f_{\vtheta}(\vx_i)$ and logistic loss. Given a local-max-margin direction $\barvthetaopt \in \sphS^{D-1}$, there is a desingularizing function on $[0, \nu)$ and two constants $\epsilon_0 > 0$, $\rho_0 \ge 1$ such that for any $\vtheta_0$ with norm $\normtwosm{\vtheta_0} \ge \rho_0$ and direction $\normtwo{\frac{\vtheta_0}{\normtwosm{\vtheta_0}} - \barvthetaopt} \le \epsilon_0$, the gradient flow $\vtheta(t)$ starting with $\vtheta_0$ satisfies
	\[
		\frac{\dd \zeta(t)}{\dd t} \le -c \frac{\dd \Psi(\gamma(\barvthetaopt) - \tildegamma(\vtheta(t)))}{\dd t}, \qquad \text{for a.e.~}t \in [0, T),
	\]
	where $T := \inf\{ t \ge 0 : \tildegamma(\vtheta(t)) \ge \gamma(\barvthetaopt) \} \in \R \cup \{+\infty\}$.
\end{lemma}
\begin{proof}
	Fix an arbitrary $\kappa \in (L/2, L)$. Let $\Psi$ be the desingularizing function on $[0, \nu)$ obtained from \Cref{lm:phase-4-main-kl}. WLOG, we can make $\nu < \gamma(\barvthetaopt) / 2$.
	
	Let $\tildegamma_{\inf}(\rho, \eps)$ be the following lower bound for the initial smoothed margin $\tildegamma(\vtheta_0)$:
	\begin{equation} \label{eq:def-tildegamma-inf}
	\tildegamma_{\inf}(\rho, \eps) := \inf\left\{\tildegamma(\vtheta) : \normtwosm{\vtheta} \ge \rho, \normtwo{\frac{\vtheta}{\normtwosm{\vtheta}} - \barvthetaopt} \le \epsilon \right\}.
	\end{equation}
	We set $\rho_0$ to be sufficiently large and $\epsilon_0$ to be sufficiently small so that $\tildegamma_{\inf}(\rho_0, \epsilon_0) > \frac{1}{2}\gamma(\barvthetaopt)$ and $\rho_0^{L-\kappa} \ge (4\ln n + 2)/\gamma(\barvthetaopt)$. By chain rule, it suffices to prove
	\begin{equation} \label{eq:phase-4-general-impl-kl}
	\frac{\dd \tildegamma(\vtheta(t))}{\dd t} \ge \frac{1}{c\Psi'(\gamma(\barvthetaopt) - \tildegamma(\vtheta(t)))} \frac{\dd \zeta(t)}{\dd t}, \qquad \text{for a.e.~}t \in [0, T),
	\end{equation}
	where $c = \max\left\{2, \frac{\gamma(\barvthetaopt)}{2 \ln n + 1}\right\}$.
	
	We consider two cases, where assume \eqref{eq:phase-4-main-kl-a-cond} is true in Case 1 and \eqref{eq:phase-4-main-kl-a-cond} is not true in Case 2. According to our choice of $\rho_0$ and the monotonicity of $\normtwosm{\vtheta(t)}$, we have $\normtwosm{\vtheta(t)}^{L-\kappa} \ge \rho_0^{L-\kappa} \ge \frac{4\ln n + 2}{\gamma(\barvthetaopt)}$, and thus $\frac{\gamma(\barvthetaopt)}{4\ln n + 2} \normtwosm{\vtheta(t)}^{L-\kappa} \ge 1$. This means
	\begin{equation} \label{eq:phase4-general-tmp1}
	\normtwo{\barcpartialr \tildegamma(\vtheta(t))} + \normtwo{\barcpartialperp \tildegamma(\vtheta(t))} \le 2\normtwo{\barcpartialr \tildegamma(\vtheta(t))}
	\end{equation}
	in Case 1, and
	\begin{equation} \label{eq:phase4-general-tmp2}
	\normtwo{\barcpartialr \tildegamma(\vtheta(t))} + \normtwo{\barcpartialperp \tildegamma(\vtheta(t))} < \frac{\gamma(\barvthetaopt)}{2\ln n + 1} \normtwosm{\vtheta(t)}^{L-\kappa} \normtwo{\barcpartialr \tildegamma(\vtheta(t))}
	\end{equation}
	in Case 2.
	
	\paragraph{Case 1.} For any $t \ge 0$, if $\frac{\dd \vtheta(t)}{\dd t} = -\barcpartial \Loss(\vtheta(t))$ and \eqref{eq:phase-4-main-kl-a-cond} hold for $\vtheta(t)$.
	By \Cref{lm:tildegamma-speed-decompose}, we have the following lower bound for $\frac{\dd \tildegamma(\vtheta(t))}{\dd t}$.
	\begin{equation} \label{eq:phase-4-general-case1-highlevel}
	\frac{\dd \tildegamma(\vtheta(t))}{\dd t} \ge \normtwo{\barcpartialperp \tildegamma(\vtheta(t))} \normtwo{\barcpartialperp \Loss(\vtheta(t))}.
	\end{equation}
	By triangle inequality and \eqref{eq:phase4-general-tmp1},
	\[
	\normtwosm{\barcpartial \tildegamma(\vtheta(t))} \le \normtwosm{\barcpartialr \tildegamma(\vtheta(t))} + \normtwosm{\barcpartialperp \tildegamma(\vtheta(t))} \le 2\normtwo{\barcpartialperp \tildegamma(\vtheta(t))}.
	\]
	So $\normtwo{\barcpartialperp \tildegamma(\vtheta(t))} \ge \frac{1}{2}\normtwosm{\barcpartial \tildegamma(\vtheta(t))}$. Combining this with \eqref{eq:phase-4-general-case1-highlevel} and noting that $\frac{\dd \zeta(t)}{\dd t} = \frac{1}{\normtwosm{\vtheta(t)}}\normtwo{\barcpartialperp \Loss(\vtheta(t))}$, we have
	\[
	\frac{\dd \tildegamma(\vtheta(t))}{\dd t} \ge \frac{1}{2}\normtwosm{\barcpartial \tildegamma(\vtheta(t))} \cdot \left(\normtwosm{\vtheta(t)} \cdot \frac{\dd \zeta(t)}{\dd t}\right).
	\]
	Applying \eqref{eq:phase-4-main-kl-a} gives
	\[
	\frac{\dd \tildegamma(\vtheta(t))}{\dd t} \ge \frac{1}{2\Psi'(\gamma(\barvthetaopt) - \tildegamma(\vtheta(t)))}\frac{\dd \zeta(t)}{\dd t} \ge \frac{1}{c\Psi'(\gamma(\barvthetaopt) - \tildegamma(\vtheta(t)))}\frac{\dd \zeta(t)}{\dd t}.
	\]
	\paragraph{Case 2.} For any $t \ge 0$, if $\frac{\dd \vtheta(t)}{\dd t} = -\barcpartial \Loss(\vtheta(t))$ and \eqref{eq:phase-4-main-kl-a-cond} does not hold for $\vtheta(t)$, i.e.,
	\begin{equation} \label{eq:phase-4-main-kl-a-cond-anti}
	\normtwo{\barcpartialperp \tildegamma(\vtheta(t))} < \frac{\gamma(\barvthetaopt)}{4 \ln n + 2} \normtwosm{\vtheta(t)}^{L-\kappa} \normtwo{\barcpartialr \tildegamma(\vtheta(t))},
	\end{equation}
	By \Cref{lm:tildegamma-speed-decompose}, we have the following lower bound for $\frac{\dd \tildegamma(\vtheta(t))}{\dd t}$.
	\begin{equation} \label{eq:phase-4-general-case2-highlevel}
	\frac{\dd \tildegamma(\vtheta(t))}{\dd t} \ge \normtwo{\barcpartialr \tildegamma(\vtheta(t))} \normtwo{\barcpartialr \Loss(\vtheta(t))}.
	\end{equation}
	We lower bound $\normtwo{\barcpartialr \tildegamma(\vtheta(t))}$ and $\normtwo{\barcpartialr \Loss(\vtheta(t))}$ respectively in order to apply KL inequality \eqref{eq:phase-4-main-kl-b}.
	
	\paragraph{Bounding $\normtwo{\barcpartialr \tildegamma(\vtheta(t))}$ in Case 2.}  By triangle inequality and \eqref{eq:phase4-general-tmp2},
	\begin{align*}
	\normtwosm{\barcpartial \tildegamma(\vtheta(t))} &\le \normtwosm{\barcpartialr \tildegamma(\vtheta(t))} + \normtwosm{\barcpartialperp \tildegamma(\vtheta(t))} \\
	&< \frac{\gamma(\barvthetaopt)}{2 \ln n + 1} \normtwosm{\vtheta(t)}^{L-\kappa} \normtwo{\barcpartialr \tildegamma(\vtheta(t))},
	\end{align*}
	which can be restated as
	\begin{equation} \label{eq:phase-4-general-case2-term1}
	\normtwo{\barcpartialr \tildegamma(\vtheta(t))} \ge \frac{2 \ln n + 1}{\gamma(\barvthetaopt)} \normtwosm{\vtheta(t)}^{\kappa-L} \normtwo{\barcpartial \tildegamma(\vtheta(t))}.
	\end{equation}
	
	\paragraph{Bounding $\normtwo{\barcpartialr \Loss(\vtheta(t))}$ in Case 2.}
	
	By Lemma C.3 in \citet{ji2020directional},
	\[
	\normtwo{\barcpartialr \tildegamma(\vtheta(t))} \le \frac{L \cdot (2 \ln n + 1)}{\normtwo{\vtheta(t)}^{L+1}}.
	\]
	Combining this with \eqref{eq:phase-4-main-kl-a-cond-anti},
	\[
	\normtwo{\barcpartialperp \tildegamma(\vtheta(t))} < \frac{\gamma(\barvthetaopt)}{4 \ln n + 2} \normtwosm{\vtheta(t)}^{L-\kappa} \cdot \frac{L \cdot (2 \ln n + 1)}{\normtwo{\vtheta(t)}^{L+1}} = \frac{\gamma(\barvthetaopt)}{2} L \normtwo{\vtheta(t)}^{-(1+\kappa)},
	\]
	which can be rewritten as
	\begin{equation} \label{eq:gamma-perp-greater}
	\frac{L \gamma(\barvthetaopt)}{2} > \normtwo{\barcpartialperp \tildegamma(\vtheta(t))}  \normtwo{\vtheta(t)}^{1+\kappa}.
	\end{equation}
	By the chain rule and Lemma C.5 in \citet{ji2020directional},
	\begin{equation} \label{eq:alpha-r-ge-gamma-tilde}
	\normtwo{\barcpartialr \alpha(\vtheta(t))} = \frac{L \cdot \dotp{\vtheta(t)}{\barcpartial \alpha(\vtheta(t))}}{\normtwosm{\vtheta(t)}} \ge \frac{L\alpha(\vtheta(t))}{\normtwosm{\vtheta(t)}} = L \tildegamma(\vtheta(t)) \normtwosm{\vtheta(t)}^{L-1}.
	\end{equation}	
	By the monotonicity of $\tildegamma(\vtheta(t))$ during training, $\tildegamma(\vtheta(t)) \ge \tildegamma(\vtheta(0))$. Also note that
	\[
	\tildegamma(\vtheta(0)) \ge \tildegamma_0 > \frac{1}{2} \gamma(\barvthetaopt),
	\]
	where the first inequality is by definition of $\tildegamma_0$, and the second inequality is due to our choice of $\rho_0, \epsilon_0$. So we can replace $\tildegamma(\vtheta(t))$ with $\frac{1}{2} \gamma(\barvthetaopt)$ in the RHS of \eqref{eq:alpha-r-ge-gamma-tilde} and obtain
	\[
	\normtwo{\barcpartialr \alpha(\vtheta(t))} \ge \frac{L \gamma(\barvthetaopt)}{2} \normtwosm{\vtheta(t)}^{L-1}.
	\]
	Combining this with \eqref{eq:gamma-perp-greater} and noting that $\barcpartialperp \alpha(\vtheta(t)) = \normtwo{\vtheta(t)}^L\barcpartialperp \tildegamma(\vtheta(t))$, we have
	\begin{equation} \label{eq:barcpartial-alpha-ineq}
	\normtwo{\barcpartialr \alpha(\vtheta(t))} \ge \normtwo{\vtheta(t)}^{L+\kappa} \normtwo{\barcpartialperp \tildegamma(\vtheta(t))} \ge \normtwo{\vtheta(t)}^{\kappa} \normtwo{\barcpartialperp \alpha(\vtheta(t))}.
	\end{equation}
	Recall that $\alpha(\vtheta) = \ell^{-1}(\Loss(\vtheta))$. By chain rule, $\barcpartial \alpha(\vtheta)$ is equal to the subgradient $\barcpartial \Loss(\vtheta)$ rescaled by some factor. Thus \eqref{eq:barcpartial-alpha-ineq} implies
	\begin{equation} \label{eq:phase-4-general-case2-term2}
	\normtwo{\barcpartialr \Loss(\vtheta(t))} \ge \normtwo{\vtheta(t)}^{\kappa} \normtwo{\barcpartialperp \Loss(\vtheta(t))}.
	\end{equation}
	
	\paragraph{Applying \eqref{eq:phase-4-main-kl-b} for Case 2.} Putting \eqref{eq:phase-4-general-case2-highlevel}, \eqref{eq:phase-4-general-case2-term1} and \eqref{eq:phase-4-general-case2-term2} together gives
	\begin{align*}
	\frac{\dd \tildegamma(\vtheta(t))}{\dd t} &\ge \left(\frac{2 \ln n + 1}{\gamma(\barvthetaopt)} \normtwosm{\vtheta(t)}^{\kappa-L} \normtwo{\barcpartial \tildegamma(\vtheta(t))}\right) \cdot \left(\normtwo{\vtheta(t)}^{\kappa} \normtwo{\barcpartialperp \Loss(\vtheta(t))}\right) \\
	&\ge \frac{2 \ln n + 1}{\gamma(\barvthetaopt)} \normtwosm{\vtheta(t)}^{2\kappa-L} \normtwo{\barcpartial \tildegamma(\vtheta(t))} \cdot \normtwo{\barcpartialperp \Loss(\vtheta(t))} \\
	&= \frac{2 \ln n + 1}{\gamma(\barvthetaopt)} \normtwosm{\vtheta(t)}^{2\kappa-L+1} \normtwo{\barcpartial \tildegamma(\vtheta(t))} \cdot \frac{\dd \zeta(t)}{\dd t},
	\end{align*}
	where the last equality is due to $\frac{\dd \zeta(t)}{\dd t} = \frac{1}{\normtwosm{\vtheta(t)}}\normtwo{\barcpartialperp \Loss(\vtheta(t))}$. Applying \eqref{eq:phase-4-main-kl-b} gives
	\[
	\frac{\dd \tildegamma(\vtheta(t))}{\dd t} \ge \frac{2 \ln n + 1}{\gamma(\barvthetaopt)} \cdot \frac{1}{\Psi'(\gamma(\barvthetaopt) - \tildegamma(\vtheta(t)))}\frac{\dd \zeta(t)}{\dd t} \ge \frac{1}{c\Psi'(\gamma(\barvthetaopt) - \tildegamma(\vtheta(t)))}\frac{\dd \zeta(t)}{\dd t}.
	\]
	
	\paragraph{Final Proof Step.} For a.e.~$t \ge 0$, $\vtheta(t)$ lies in either Case 1 or Case 2, so \eqref{eq:phase-4-general-impl-kl} holds, and we can rewrite it as
	\[
	c\Psi'(\gamma(\barvthetaopt) - \tildegamma(\vtheta(t))) \frac{\dd \tildegamma(\vtheta(t))}{\dd t} \ge \frac{\dd \zeta(t)}{\dd t}, \qquad \text{for a.e.~}t \in [0, T).
	\]
	By chain rule, the LHS is equal to $\frac{\dd}{\dd t}\left( c\Psi(\gamma(\barvthetaopt) - \tildegamma(\vtheta(t)))\right)$, which completes the proof.
\end{proof}

\begin{proof}[Proof for \Cref{thm:phase-4-general}]
	By \Cref{lm:phase-4-main-dyn}, we can choose $\epsilon_0, \rho_0$ such that 
	\[
	\frac{\dd \zeta(t)}{\dd t} \le -c \frac{\dd \Psi(\gamma(\barvthetaopt) - \tildegamma(\vtheta(t)))}{\dd t}, \qquad \text{for a.e.~}t \in [0, T),
	\]
	where $T := \inf\{ t \ge 0 : \tildegamma(\vtheta(t)) \ge \gamma(\barvthetaopt) \} \in \R \cup \{+\infty\}$. Then for all $t \in (0, T)$,
	\begin{equation} \label{eq:phase-4-general-key}
		\zeta(t) \le c\Psi(\gamma(\barvthetaopt) - \tildegamma(\vtheta_0)) \le \delta(\epsilon_0, \rho_0) := c\Psi(\gamma(\barvthetaopt) - \tildegamma_{\inf}(\rho_0, \epsilon_0)),
	\end{equation}
	where $\tildegamma_{\inf}$ is defined in \eqref{eq:def-tildegamma-inf}. We can choose $\epsilon_0$ small enough and $\rho_0$ large enough so that $\delta(\epsilon_0, \rho_0) > 0$ is as small as we want.
	
	If $T = +\infty$, then \eqref{eq:phase-4-general-key} implies that $\frac{\vtheta(t)}{\normtwosm{\vtheta(t)}}$ converges to some $\barvtheta$ as $t \to +\infty$, and $\normtwosm{\barvtheta - \barvthetaopt} \le \delta$ if $\delta(\epsilon_0, \rho_0) \le \delta$.
	
	If $T$ is finite, then by triangle inequality we have
	$\normtwo{\frac{\vtheta(T)}{\normtwosm{\vtheta(T)}} - \barvthetaopt} \le
	\epsilon_0 + \delta(\epsilon_0, \rho_0)$. Since $\barvthetaopt$ is a
	local-max-margin direction, when $\epsilon_0$ and $\delta(\epsilon_0,
	\rho_0)$ are sufficiently small, $\tildegamma(\vtheta) \le \gamma(\vtheta)
	\le \gamma(\barvthetaopt)$ holds for any $\vtheta$ satisfying
	$\normtwo{\frac{\vtheta}{\normtwosm{\vtheta}} - \barvthetaopt} \le
	2(\epsilon_0 + \delta(\epsilon_0, \rho_0))$. The definition of $T$ then
	implies that $\tildegamma(\vtheta(T)) = \gamma(\vtheta(T)) =
	\gamma(\barvthetaopt)$. By Lemma B.1 from \citet{lyu2020gradient},
	$\tildegamma(\vtheta(t))$ is non-decreasing over time, and if it stops
	increasing at some value, then the time derivative of
	$\frac{\vtheta(t)}{\normtwosm{\vtheta(t)}}$ must be zero. Thus we have
	$\tildegamma(\vtheta(t)) = \gamma(\barvthetaopt)$ and $\frac{\dd}{\dd t}
	\frac{\vtheta(t)}{\normtwosm{\vtheta(t)}} = 0$ for all $t \ge T$, which
	implies that $\frac{\vtheta(t)}{\normtwosm{\vtheta(t)}}$ converges to
	$\barvtheta := \frac{\vtheta(T)}{\normtwosm{\vtheta(T)}}$ as $t \to
	+\infty$. This again proves that $\normtwosm{\barvtheta - \barvthetaopt} \le
	\delta$.
	
	Now we only need to show that $\gamma(\barvtheta) = \gamma(\barvthetaopt)$. In the case where $T$ is finite, we have $\gamma(\barvtheta) = \gamma\left(\frac{\vtheta(T)}{\normtwosm{\vtheta(T)}}\right) = \gamma(\barvthetaopt)$. In the case where $T = +\infty$, $\gamma(\barvtheta)$ is a asymptotic Clarke critical value of $\tildegamma$ by \Cref{thm:converge-to-asymptotic-clarke-critical}. Since there are only finitely many asymptotic Clarke critical values (\Cref{lm:tilde-gamma-finite-critical}), we can make $\delta(\epsilon_0, \rho_0)$ to be small enough so that the only asymptotic Clarke critical value that can be achieved near $\barvthetaopt$ is $\gamma(\barvthetaopt)$ itself.
\end{proof}

\subsection{Proof for Theorem \ref{thm:sym_main}}

\begin{proof}
	By \Cref{lm:phase2-main}, $\lim_{\sigmainit \to 0} \vtheta\left(T_{12} + t\right) = \pi_{\barvb}(\tildevtheta(t))$.
	Using \Cref{lm:phase2-1-main} and noting that $\dotpsm{\vmu}{\vwopt} = \frac{1}{n}\sum_{i\in[n]} \dotpsm{y_i \vx_i}{\vwopt} > 0$,
	we know that there exists $t \le t_0$
	such that $\tildevtheta(t) = (\tildevw_1(t), \tildevw_2(t), \tildea_1(t), \tildea_2(t))$
	satisfies $\tildea_1(t) = \normtwosm{\tildevw_1(t)}$, $\tildea_2(t) = -\normtwosm{\tildevw_2(t)}$,
	$\dotp{\tildevw_1(t)}{\vwopt} > 0$ and $\dotp{\tildevw_2(t)}{\vwopt} < 0$.
	Then by \Cref{thm:two-neuron-max-margin},
	\[
		\lim_{t \to +\infty} \frac{\tildevtheta(t)}{\normtwosm{\tildevtheta(t)}} = \frac{1}{2}(\vwopt, -\vwopt, 1, -1) =: \tildevtheta_{\infty},
	\]
	which also implies that
	\begin{align*}
		\lim_{t \to +\infty} \lim_{\sigmainit \to 0} \frac{\vtheta\left(T_{12} + t\right)}{\normtwosm{\vtheta\left(T_{12} + t\right)}}
		&= \lim_{t \to +\infty} \frac{\pi_{\barvb}(\tildevtheta(t))}{\normtwosm{\pi_{\barvb}(\tildevtheta(t))}}
		= \pi_{\barvb}\left(\lim_{t \to +\infty}\frac{\tildevtheta(t)}{\normtwosm{\tildevtheta(t)}}\right)
		= \pi_{\barvb}(\tildevtheta_{\infty}). \\
		\lim_{t \to +\infty} \lim_{\sigmainit \to 0} \normtwosm{\vtheta\left(T_{12} + t\right)}
		&= \lim_{t \to +\infty} \normtwosm{\pi_{\barvb}(\tildevtheta(t))}
		= +\infty.
	\end{align*}
	This means that for any $\epsilon > 0$ and $\rho > 0$,
	we can choose a time $t_1 \in \R$
	such that
	$\normtwo{\frac{\vtheta(T_{12} + t_1)}{\normtwosm{\vtheta(T_{12} + t_1)}} - \pi_{\barvb}(\tildevtheta_{\infty})} \le \epsilon$
	and $\normtwo{\vtheta(T_{12} + t_1)} \ge \rho$ for any $\sigmainit$ small enough.
	By \Cref{thm:maxmar-linear}, $\pi_{\barvb}(\tildevtheta_{\infty})$ is a global-max-margin direction.
	Then \Cref{thm:phase-4-general} shows that there exists $\sigmainitmax$ such that for all $\sigmainit < \sigmainitmax$,
	$\frac{\vtheta(t)}{\normtwosm{\vtheta(t)}} \to \barvtheta$,
	where $\gamma(\barvtheta) = \gamma(\pi_{\barvb}(\tildevtheta_{\infty}))$
	and $\normtwosm{\barvtheta - \pi_{\barvb}(\tildevtheta_{\infty})} \le \delta$. 
	Therefore, $\barvtheta$ is a global-max-margin direction
	and $\finftime(\vx) = \frac{1+\alphaLK}{4}\dotp{\vwopt}{\vx}$ by \Cref{thm:maxmar-linear}.
\end{proof}

\section{Trajectory-based Analysis for Non-symmetric Case} \label{sec:proof-nonsym}
The proofs for the non-symmetric case follow similar manners from phase I to
phase III. The high-level idea is to show the following in the 3 phases:
\begin{enumerate}
	\item In Phase I, every weight vector $\vw_k$ in the first layer moves
	towards the direction of either $\vmuplus$ or $-\vmusubt$. At the end of Phase I the weight
	vectors towards $-\vmusubt$ have much smaller norms than those towards
	$\vmuplus$, thereby becoming negligible.
	\item In Phase II, we show that the dynamics of $\vtheta(t)$ is close to a
	one-neuron dynamic (after embedding) for a long time.
	\item In Phase III, we show that the one-neuron classifier converges to the
	max-margin solution among one-neuron neural nets (while the embedded classifier may have suboptimal margin among $m$-neuron neural nets), and the gradient flow
	$\vtheta(t)$ on the $m$-neuron neural net gets stuck at a KKT-direction near this embedded classifier.
\end{enumerate}

\subsection{Additional Notations}
In this section we highlight the additional notations that allow us to adapt the results from previous sections.
For $\delta\ge 0$, define $\cone[\delta]$ to be the convex cone containing all the unit weight vectors that have $\delta$ margin over the dataset $\{(\vx_i,y_i)\}_{i\in [n]}$.
\begin{align*}
	\cone[\delta] &:= \left\{\lambda\vw : \dotp{\vw}{y_i\vx_i}\ge \delta, \vw \in \sphS^{d-1}, \lambda > 0, \forall i \in [n]\right\}, \\
	\cone         := \cone[0] &:= \left\{\vw \in \R^{d} : \vw \ne \vzero, \dotp{\vw}{y_i\vx_i}\ge 0, \forall i \in [n]\right\}.
\end{align*}
For $0<\eps<1$, we define
\begin{align*}
	\hone[\eps] &:= \left\{\frac{1}{2n}\sum_{i\in [n]} (1+\eps_i) \alpha_i y_i\vx_i :  \alpha_i\in [\alphaLK,1], \eps_i \in [-\eps,\eps], \forall i \in [n]\right\}, \\
	\hone := \hone[0] &:= \left\{\frac{1}{2n}\sum_{i\in [n]} \alpha_i y_i\vx_i :  \alpha_i\in [\alphaLK,1], \forall i \in [n]\right\}.
\end{align*}
By \Cref{lem:non_sym_eps} we know $-\frac{\cpartial \Loss(\vtheta)}{\partial
\vw_k} \subseteq a_k\hone[\eps]$ if $\normMsm{\vtheta} \le
\sqrt{\frac{\eps}{2m}}$. Further we define
\[
	\kone[\eps] := \bigcup_{\lambda > 0}\lambda \hone[\eps] \quad\textrm{ and }\quad \kone := \kone[0] = \bigcup_{\lambda > 0}\lambda \hone[0]
\]
Then $-\frac{\cpartial \Loss(\vtheta)}{\partial \vw_k} \subseteq \sgn(a_k) \kone[\eps]$. For a set $S$,  we will use $\mathring{S}$ to denote the interior of $S$. 

Recall in \Cref{sec:nonsym-thm}, for every $\vx_i$, we define $\vxplus_i := \vx_i$ if $y_i = 1$ and $\vxplus_i := \alphaLK\vx_i$ if $y_i = -1$. Similarly, we define $\vxsubt_i := \alphaLK\vx_i$ if $y_i = 1$ and $\vxplus_i := \vx_i$ if $y_i = -1$. Then we define $\vmuplus$ to be the mean vector of $y_i \vxplus_i$, and $\vmusubt$ to be the mean vector of $y_i \vxsubt_i$, that is,
\begin{align*}
\vmuplus := \frac{1}{n}\sum_{i \in [n]} y_i \vxplus_i, \qquad \vmusubt := \frac{1}{n}\sum_{i \in [n]} y_i \vxsubt_i.
\end{align*}
We use $\barvmuplus:=\frac{\vmuplus}{\normtwosm{\vmuplus}}$,  $\barvmusubt:=\frac{\vmusubt}{\normtwosm{\vmusubt}}$ to denote $\vmuplus$, $\vmusubt$ after normalization.
Similar to $\kone[\eps]$, we define $\mup[\eps]$ and $\mun[\eps]$ as the perturbed versions of $\vmu^+$ and $\vmu^-$ in the sense that $\mup := \{\lambda \vmu^+ : \lambda > 0\}$ and $\mun := \{\lambda \vmu^- :  \lambda > 0\}$.
\begin{align*}
\mup[\eps] &= \left\{  \frac{\lambda}{2n}\sum_{i \in [n]} (1+\eps_i)y_i \vxplus_i : \eps_i\in [-\eps,\eps], \lambda > 0 \right\}, \\
\mun[\eps] &= \left\{  \frac{\lambda}{2n}\sum_{i \in [n]} (1+\eps_i)y_i \vxsubt_i : \eps_i\in [-\eps,\eps], \lambda > 0 \right\}.
\end{align*}

\subsection{More about Our Assumptions}

The following lemma shows that \Cref{ass:principal-dir} is a weaker assumption than \Cref{assump:cone}.

\begin{lemma}\label{lem:assump_cone_imply_principal_direction}
      \Cref{ass:principal-dir} implies \Cref{assump:cone}.
\end{lemma}
\begin{proof}
    Let $\vwgar$ be the principal direction defined in \Cref{ass:principal-dir}. We can decompose $\vmu = \vmu_{\perp} + \vmu_{\paral}$, where $\vmu_{\paral}$ is the along the direction of $\vwgar$ and $\vmu_{\perp}$ is orthogonal to $\vwgar$. \Cref{ass:principal-dir} implies that for all $i, j\in [n]$,
    \begin{align*}
    	-\dotpsm{y_i\vx_i}{y_j\vx_j} &= -\dotpsm{y_i\vx_i}{\vwgar} \cdot \dotpsm{y_j\vx_j}{\vwgar} - \dotp{\mPgar (y_i\vx_i)}{\mPgar (y_j\vx_j)} \\
    	&\le \normtwosm{\mPgar \vx_i} \normtwosm{\mPgar \vx_j}.
    \end{align*}
    Then for all $i \in [n]$, we have
    \begin{align*}
    	\frac{1}{n}\sum_{j\in [n]}\max\{-\dotpsm{y_i\vx_i}{y_j\vx_j}, 0\} &\le \frac{1}{n}\sum_{j \in [n]}\normtwosm{\mPgar \vx_i} \normtwosm{\mPgar \vx_j} \\
    	&\le \normtwosm{\mPgar \vx_i} \cdot \alphaLK\dotp{\vmu}{\vwgar} \frac{\gammagar}{\max_{j \in [n]} \normtwosm{\mPgar \vx_j}} \\
    	&\le \alphaLK\dotp{\vmu}{\vwgar}\gammagar.
    \end{align*}
    On the other hand, recall that $\gammagar := \min_{i \in [n]} y_i\dotp{\vwgar}{\vx_i}$, then we have
    \begin{align*}
    	\dotpsm{\vmu}{y_i\vx_i} &= \dotp{\vmu}{\vwgar} \dotp{y_i\vx_i}{\vwgar} + \dotp{\mPgar \vmu}{\mPgar (y_i\vx_i)} \\
    	&\ge \dotp{\vmu}{\vwgar} \dotp{y_i\vx_i}{\vwgar} - \frac{1}{n} \sum_{j \in [n]} \normtwosm{\mPgar \vx_j} \normtwosm{\mPgar \vx_i} \\
    	&\ge \dotp{\vmu}{\vwgar} \gammagar - \normtwosm{\mPgar \vx_i} \cdot \alphaLK\dotp{\vmu}{\vwgar} \frac{\gammagar}{\max_{j \in [n]} \normtwosm{\mPgar \vx_j}} \\
    	&\ge (1 - \alphaLK) \dotp{\vmu}{\vwgar} \gammagar.
    \end{align*}
    Combining these proves that $\dotpsm{\vmu}{y_i\vx_i} \ge \frac{1-\alphaLK}{n \cdot \alphaLK} \sum_{j\in [n]}\max\{-\dotpsm{y_i\vx_i}{y_j\vx_j}, 0\}$.
\end{proof}

\Cref{lem:assump_equiv} gives the main property we will use from  \Cref{assump:cone}, i.e. $\kone \subseteq \mathring{\cone}$.  
\begin{lemma}\label{lem:assump_equiv}
	For linearly separable dataset $\{(\vx_i, y_i)\}_{i \in [n]}$ and $\alphaLK\in (0,1]$, \Cref{assump:cone} is equivalent to $\kone \subseteq \mathring{\cone}$.
\end{lemma}
\begin{proof}
    By definition, we know
    \begin{align*}
    	\cone &= \{\vw \in \R^d : \vw \ne \vzero, \dotp{\vw}{y_i\vx_i}\ge 0, \forall i \in [n]\}, \\
    	\mathring{\cone} &= \{\vw \in \R^d: \dotp{\vw}{y_i\vx_i}> 0, \forall i \in [n]\}.
    \end{align*}
    and $\kone = \left\{\lambda\sum_{i\in [n]} \alpha_i y_i\vx_i: \alpha_i\in [\alphaLK,1], \lambda > 0\right\}$. For any $i \in [n]$, we have
    \begin{align*}
                            & \quad \dotp{\vmu}{y_i\vx_i} > \frac{1-\alphaLK}{n \cdot \alphaLK}\sum_{j\in [n]}\max\{-\dotp{y_i\vx_i}{y_j\vx_j}, 0\} \\
    	\Longleftrightarrow & \quad \frac{1}{n}\sum_{j\in [n]}\left( \frac{1-\alphaLK}{\alphaLK}\min\{\dotp{y_i\vx_i}{y_j\vx_j}, 0\}+ \dotp{y_i\vx_i}{y_j\vx_j}\right)> 0 \\
    	\Longleftrightarrow & \quad \sum_{j\in [n]} \underbrace{\left( (1-\alphaLK)\min\{\dotp{y_i\vx_i}{y_j\vx_j}, 0\}+ \alphaLK\dotp{y_i\vx_i}{y_j\vx_j}\right)}_{\Delta_{ij}} > 0
    \end{align*}
    Note that $\Delta_{ij} = \min_{\alpha_j \in [\alphaLK, 1]} \dotp{y_i\vx_i}{\alpha_j y_j \vx_j}$. So
    \[
    	\sum_{j \in [n]} \Delta_{ij} = \sum_{j \in [n]} \min_{\alpha_j \in [\alphaLK, 1]} \dotp{y_i\vx_i}{\alpha_j y_j \vx_j} = \min_{\valpha \in [\alphaLK, 1]^n}\dotp{y_i\vx_i}{\sum_{j \in [n]} \alpha_j y_j \vx_j}.
    \]
    Therefore we have the following equivalence:
    \begin{align}
    	\text{\Cref{assump:cone}} \quad &\Longleftrightarrow \quad \forall i \in [n]: \min_{\valpha \in [\alphaLK, 1]^n}\dotp{y_i\vx_i}{\sum_{j \in [n]} \alpha_j y_j \vx_j} > 0 \label{eq:ass-cone-equiv-1} \\
    	&\Longleftrightarrow \quad \forall \vw \in \kone, \forall i \in [n]: \dotp{y_i\vx_i}{\vw} > 0 \\
    	&\Longleftrightarrow \quad \kone \subseteq \mathring{\cone}.
    \end{align}
    which completes the proof.
\end{proof}

\Cref{lem:assump_equiv} shows that every direction in $\kone$ has non-zero margin. Below we let the $\delta$ be the minimum of the margin of unit-norm linear separators in $\kone$:
\[
    \delta := \min_{\vw\in \kone \cap \sphS^{d-1}}\min_{i\in [n]} \dotp{y_i\vx_i}{\vw}.
\]
By \eqref{eq:ass-cone-equiv-1} we have $\delta>0$, and thus $\kone \subseteq \cone[\delta]$.

\subsection{Phase I}
The overall result we will prove for phase I in the non-symmetric case is
\Cref{lem:non-sym-phase1-main}.  Compared to the symmetric case, even $G$
function is not linear anymore. Recall $G$ is defined as below:
\[
	G(\vw) := \frac{-\ell'(0)}{n}\sum_{i \in [n]} y_i \phi(\vw^{\top} \vx_i) = \frac{1}{2n}\sum_{i \in [n]} y_i \phi(\vw^{\top} \vx_i).
\]
It holds that $\forall \vw\in\R^d$, $\cpartial G(\vw)\subseteq \kone$. Moreover,
we have $\vw\in \ocone \Longrightarrow \cpartial G(\vw) = \{\vmuplus \}$ and
$\vw\in -\ocone \Longrightarrow\cpartial G(\vw) = \{ \vmusubt\}$. Thanks to
\Cref{assump:cone}, we can show each neuron $\vw_k(t)$ will eventually converge
to areas with fixed sign pattern $\pm \cone[\delta/3]$ and thus $G$ will become
linear. \Cref{lem:non_sym_first_converge_to_cone_G} states this idea more
formally. Its proof is a simpliciation to the realistic case,
\Cref{lem:non_sym_first_converge_to_cone_realistic}, and thus omitted. We will
not use \Cref{lem:non_sym_first_converge_to_cone_G} in the future.

\begin{lemma}\label{lem:non_sym_first_converge_to_cone_G}
For any dataset $\{(\vx_i,y_i)\}_{i\in [n]}$ satisfying \Cref{assump:cone}, suppose $\vw(0) \neq \lambda \vmu^-,\ \forall \lambda\ge0$, and  it holds that 
\begin{align}
    \frac{\dd \vw}{\dd t}\in \normtwo{\vw} \cdot\cpartial G(\vw),
\end{align}
then there exists $T_0>0$, such that $ \vw(T_0)\in \cone[\delta/2]$.
    
\end{lemma}

However, in the realistic setting, each $\vw_k$ is not following gradient flow of $G$ exactly --- there are tiny correlations between different $\vw_k$. And we will control those correlations by setting initialization very small. This yields \Cref{lem:non_sym_first_converge_to_cone_realistic}.

\begin{lemma}\label{lem:non_sym_first_converge_to_cone_realistic}
Under \Cref{assump:cone}, if $\barvtheta = (\barvw_1, \dots, \barvw_m, \bara_1, \dots, \bara_m)$ satisfies the following three conditions:
\begin{enumerate}
	\item For all $k \in [m]$, $\abssm{\bara_k} = \normtwosm{\barvw_k} \ne 0$;
	\item If $\bara_k > 0$, then $\barvw_k \neq \lambda \vmusubt$ for any $\lambda > 0$;
	\item If $\bara_k < 0$, then $\barvw_k \neq -\lambda \vmuplus$ for any $\lambda > 0$;
\end{enumerate}
then there exist $T_0,\sigmainitmax>0$, such that for any $\sigmainit<\sigmainitmax$, the gradient flow $\vtheta(t) = (\vw_1(t), \dots, \vw_m(t), a_1(t), \dots, a_m(t)) = \phitheta(\sigmainit \barvtheta, t)$ satisfies the following at time $T_0$,
\begin{equation} \label{eq:non_sym_first_converge_to_cone_realistic}
    \vw_k(T_0) \in \begin{cases}
    \cone[\delta/3],  &\quad\text{if } \bar{a}_k > 0, \\
    -\cone[\delta/3], &\quad\text{if } \bar{a}_k < 0.
    \end{cases}
\end{equation}
Moreover, there are constants $A,B>0$ such that $A\sigmainit\le \normtwo{\vw_k(T_0)}\le B\sigmainit$.
\end{lemma}
It is easy to see that the three conditions in
\Cref{lem:non_sym_first_converge_to_cone_realistic} hold with probability $1$
over the random draw of $\barvtheta_0 \sim \Dinit(1)$. Then after time $T_0$,
all the neurons $\vw_k$ are either in $\cone[\delta/3]$ or $-\cone[\delta/3]$,
and will not leave it until $T^\eps_{\sigmainit}$, which implies the sign
patterns $\sgn(\dotp{\vx_i}{\vw_k(t)})=s_ky_i$ is fixed for
$t\in[T_0,T^\eps_{\sigmainit}]$. Thus similar to the symmetric case,
$\vtheta(t)$ evolves approximately under power iteration and yields the
following lemma.
\begin{lemma}\label{lem:non-sym-phase1-main}
	Suppose that \Cref{assump:cone,ass:vmuplus-greater} hold.
	Let $T_1(\sigmainit, r) := \frac{1}{\lambda^+_0} \ln \frac{r}{\sqrt{m} \sigmainit}$.
	With probability $1$ over the random draw of
	$\barvtheta_0 = (\barvw_1, \dots, \barvw_m, \bara_1, \dots, \bara_m) \sim \Dinit(1)$,
	the prerequisites of \Cref{lem:non_sym_first_converge_to_cone_realistic} are satisfied.
	In this case, there exists a vector $\barvb(\sigmainit) \in \R^m$ for any $\sigmainit > 0$ such that the following statements hold:
	\begin{enumerate}
		\item There exist constants $C_1 > 0, C_2 > 0, T_0 \ge 0, r_{\max} > 0$ such that for $r \in (0, r_{\max}), \sigmainit \in (0, C_1 r^3)$, any neuron $(\vw_k, a_k)$ at time $T_0 + T_1(\sigmainit, r)$ can be decomposed into
		\begin{align*}
			\textrm{If $\bara_k >0$:}\quad \vw_k(T_0+T_1(\sigmainit, r)) &= r\barb_k(\sigmainit) \barvmuplus + \Delta\vw_k, \\
			a_k(T_0+T_1(\sigmainit, r)) &= r\barb_k(\sigmainit) + \Delta a_k,\\
			\textrm{If $\bara_k <0$:}\quad	\vw_k(T_0+T_1(\sigmainit, r)) &= r^{1-\kappa} \barb_k(\sigmainit) \barvmusubt + \Delta\vw_k, \\
			a_k(T_0+T_1(\sigmainit, r)) &= r^{1-\kappa} \barb_k(\sigmainit) + \Delta a_k,
		\end{align*}
		where the error term $\Delta \vtheta :=(\Delta \vw_1, \dots, \Delta
		\vw_m, \Delta a_1, \dots, \Delta a_m)$ is upper bounded by
		$\normMsm{\Delta \vtheta} \le C_2 r^3$ and $\kappa$ is the gap $1-
		\frac{\normtwosm{\vmusubt}}{\normtwosm{\vmuplus}} > 0$.
		\item There exist constants $\bar{A}, \bar{B} > 0$ such that
		$\abssm{\barb_k(\sigmainit)} \in [\bar{A}, \bar{B}]$ whenever $\bara_k
		>0$ and $\abssm{\barb_k(\sigmainit)} \in [\sigmainit^{\kappa}\bar{A},
		\sigmainit^{\kappa}\bar{B}]$ whenever $\bara_k < 0$.
	\end{enumerate}
\end{lemma}

As $\sigmainit \to 0$, $\abssm{\barb_k(\sigmainit)} \to 0$ for neurons with
$\bara_k < 0$, while $\abssm{\barb_k(\sigmainit)} \in [\bar{A}, \bar{B}]$
remains for neurons with $\bara_k > 0$. This means when the initialization scale
is small, only the neurons with $\bara_k > 0$ remain effective and the others
become negligible. Those effective neurons move their weight vectors towards the
direction of $\barvmuplus$, until the error term $\Delta\vtheta$ becomes large.

\subsubsection{Proof of Lemma \ref{lem:non_sym_first_converge_to_cone_realistic}}

\begin{proof}[Proof of \Cref{lem:non_sym_first_converge_to_cone_realistic}]
	Let $s_k := \sgn(\bara_k)$. By \Cref{cor:weight-equal}, $a_k(t) = s_k\normtwosm{\vw(t)}$ for all $t \ge 0$. Define $T^\eps_{\sigmainit}:=\inf\left\{t\ge 0 : \normM{\vtheta(t)}\ge \sqrt{\frac{\eps}{m}}\right\}$. By \Cref{lem:non_sym_eps}, we have $\forall t\le T^\eps_{\sigmainit}$, $-\frac{\cpartial \Loss(\vtheta(t))}{\partial \vw_k} \subseteq a_k(t)\hone[\eps]\subseteq s_k\kone[\eps]$.  Since $\kone\subseteq {\cone[\delta]}$, there exists $\eps_1>0$, such that for all $\eps<\eps_1$, $\kone[\eps] \subseteq {\cone[2\delta/3]}$.
	The high-level idea of the proof is that  suppose  $-\frac{\cpartial \Loss(\vtheta(t))}{\partial \vw_k} \subseteq s_k \cone[2\delta/3]$ holds for sufficiently long time $T_0$ , $\vw_k(t)$ will eventually end up in a  cone $s_k\cone[\delta/3]$ slightly wider than $s_k \cone[2\delta/3]$, as long as the total distance traveled is sufficiently long. On the other hand, we can make $\sigmainitmax$ sufficiently small, such that $T^\eps_{\sigmainit}\ge T_0$ for all $\sigmainit<\sigmainitmax$.
	
	By \Cref{lm:weight-balance} and Lipschitzness of $\ell$,
	\[
	\abs{\frac{1}{2}\frac{\dd \normtwosm{\vw_k}^2}{\dd t}} = \abs{\frac{1}{n} \sum_{i=1}^{n} \ell'(q_i(\vtheta)) y_i a_k \phi(\vw_k^{\top} \vx_i)} \le \frac{1}{n} \sum_{i=1}^{n} \abssm{a_k} \cdot \normtwosm{\vw_k} \le \normtwosm{\vw_k}^2.
	\]
	Then we have
	\begin{align}\label{eq:non_sym_first_phase_growth_control}
	\forall t\le T_{\sigmainit}^\eps, \quad \normtwosm{\vw_k(t)} \in [\normtwosm{\vw_k(0)}e^{-t}, \normtwosm{\vw_k(0)}e^t].
	\end{align}
	Thus for any $T_0\ge 0$, if $\sigmainit\le e^{-T_0}\frac{\sqrt{\frac{\eps}{m}}}{\normM{\barvtheta}}$, we have $T_0 \le T^\eps_{\sigmainit}$.
	
	In order to lower bound the total travel distance for each $\vw_k(t)$, it turns out that it suffices to lower bound the $\inf_{t\in[0,T_{\sigmainit}^\eps]} \normtwosm{\vw_k(t)}$ by $\bar{D}\sigmainit$, where $\bar{D} > 0$ is some constant. We will first show that we can guarantee the existence of such constant $\bar{D}$ by picking sufficiently small $\eps$. Then we will formally prove the original claim of \Cref{lem:non_sym_first_converge_to_cone_realistic}.
	
	\paragraph{Existence of $\bar{D}$.} By definitions of $\mup$ and $\mun$, it holds that $\forall k\in [m]$,
	\[\barvw_k \notin 
	\begin{cases}
	\mup,  &\mbox{if $\bara_k<0$;}\\
	-\mun, &\mbox{if $\bara_k>0$.}
	\end{cases}
	\]
	In other words
	\[
	\bar{d}:= \min \left\{\min_{k:\bara_k<0} \distsm{\barvw_k- \mup}{\vzero},\min_{k:\bara_k>0} \distsm{\barvw_k+ \mun}{\vzero} \right\} > 0.
	\]
	By the continuity of the distance function, there exists  $\eps_2>0$ such that $\forall \eps \in (0, \eps_2)$, it holds that 
	\[
	\min \left\{\min_{k:\bara_k<0} \distsm{\bar{\vw}_k - \mup[\eps]}{\vzero},\min_{k:\bara_k>0} \distsm{\bar{\vw}_k + \mun[\eps]}{\bm{0}} \right\} \ge \frac{\bar{d}}{2}.
	\]
	Now we take $\eps = \min\{\eps_1,\eps_2\}$. We will first show the existence of such $\bar{D}$ for $k \in [m]$ with $\bara_k>0$. And the same argument holds for $k$ with negative $\bara_k$. Let $t_k := \sup\{t\le T^\eps_{\sigmainit} : \vw_k(t)\in -\cone[\delta/3]\}$. We note that $\vw_k(t)\in -\cone[\delta/3]$ for all $t \le t_k$. Otherwise, $\vw_k(t') \notin -\cone[\delta/3]$ for some $t' <t_k$. On the one hand, we have $\vw_k(t_k)\in\vw_k(t')+ \kone[\eps]\subseteq \vw_k(t')+\cone[2\delta/3]$; on the other hand, we also know that $\vw_k(t_k)\in -\cone[\delta/3]$ by continuity of the trajectory of $\vw_k(t)$. This implies $-\cone[\delta/3] \cap (\cone[2\delta/3]+\vw_k(t'))\neq \varnothing$, and thus $\vw_k(t') \in -\cone[\delta/3]-\cone[2\delta/3] \subseteq -\cone[\delta/3]-\cone[\delta/3] \subseteq -\cone[\delta/3]$. Contradiction.
	
	Now we have $\vw_k(t)\in -\cone[\delta/3]$ for all $t \le t_k$, and this implies that $\dotp{\vw_k(t)}{\vx_i}<0$ for all $i \in [n]$. Then $ \frac{\dd \vw_k(t)}{\dd t} = -\nabla_{\vw_k} \Loss(\vtheta(t)) \in \mun[\eps]$. Therefore we have $\inf_{t\in[0,t_k]} \normtwosm{\vw_k(t)} \ge \distsm{\vw_k(0)+\mun[\eps]}{\vzero} = \sigmainit \distsm{\barvw_k+\mun[\eps]}{\vzero} \ge \frac{\bar{d}\sigmainit}{2}$. 
	
	Below we show the norm lower bound for any $t$ such that $t \in [t_k, T^\eps_{\sigmainit}]$. Let $\bar{d}'$ be the minimum distance between any point in $-\cone[2\delta/3]$ and any point on unit sphere but not in $-\cone[\delta/2]$, that is,
	\[
	\bar{d}':= \distsm{\sphS^{d-1} \setminus (-\cone[\delta/2])}{-\cone[2\delta/3]}
	\]
	We claim that $\bar{d}' > 0$. Otherwise there is a sequence of $\{\vw_j\}$ with unit norm and $\vw_j\notin-\cone[\delta/2]$ satisfying that $\lim_{n\to \infty} \dist{\vw_k}{-\cone[2\delta/3]} = 0$. Let $\barvw$ be a limit point, then $\barvw \in-\cone[2\delta/3]$ since $-\cone[2\delta/3]$ is closed. Since $-\cone[2\delta/3] \subseteq -\ocone[\delta/2]$, we further have $\barvw \in-\ocone[\delta/2]$, which contradicts with the definition of limit point.
	
	By the continuity of $\vw_k(t)$, we know $\vw_k(t_k)\notin -\cone[\delta/2]$. Thus for any $t \in [t_k, T^\eps_{\sigmainit}]$, we have $\vw_k(t)\in \vw_k(t_k)+ \cone[2\delta/3]$ and $\inf_{t \in [t_k, T^\eps_{\sigmainit}]} \normtwo{\vw_k(t)}\ge \distsm{\vzero}{\vw_k(t_k)+\cone[2\delta/3]} =\distsm{-\cone[2\delta/3]}{\vw_k(t_k)} = \normtwo{\vw_k(t_k)} \distsm{-\cone[2\delta/3]}{\frac{\vw_k(t_k)}{\normtwosm{\vw_k(t_k)}}}\ge \normtwosm{\vw_k(t_k)} \cdot \bar{d}'\ge \frac{\bar{d}\bar{d}'\sigmainit}{2}$. We can apply the same argument for those $k$ with $\bara_k < 0$, and finally we can conclude that $\normtwo{\vw_k(t)}\ge \bar{D}\sigmainit$ for all $t\in[0,T^\eps_{\sigmainit}]$ and $k\in [m]$, where $\bar{D}:=\max\{1,\bar{d}'\}\frac{\bar{d}}{2}$.

	\paragraph{Convergence to $\cone[\delta/3]$.} For $c \ge 0$ and $i\in[n]$ define $\Gamma^c_i(\vw): = \dotp{\vw}{y_i\vx_i}-c\normtwo{\vw}$. For all $k\in [m]$ and $t\le T^\eps_{\sigmainit}$, it holds that 
	\begin{align*}
	\frac{\dd \Gamma^{\delta/3}_i(s_k\vw_k)}{\dd t}
	&=\dotp{\frac{\dd \vw_k}{\dd t}}{s_ky_i\vx_i} - (\delta/3)\dotp{\frac{\dd \vw_k}{\dd t}}{\frac{\vw_k}{\normtwo{\vw_k}}}\\
	&\ge (2\delta/3)\normtwo{\frac{\dd \vw_k}{\dd t}} - (\delta/3)\normtwo{\frac{\dd \vw_k}{\dd t}}
	= (\delta/3)\normtwo{\frac{\dd \vw_k}{\dd t}},
	\end{align*}
	where the inequality is because
	$\frac{\dd \vw_k}{\dd t} \subseteq a_k\hone[\eps] \subseteq a_k\cone[2\delta/3]$ and $\dotpsm{\frac{\dd \vw_k}{\dd t}}{\frac{\vw_k}{\normtwo{\vw_k}}} \le \normtwo{\frac{\dd \vw_k}{\dd t}}$.
	
	Let $h_{\min} :=\inf_{\vw\in \hone[\eps]} \normtwosm{\vw} = \min_{\vw\in \hone[\eps]} \normtwosm{\vw} >0 $. Note that $\abssm{a_k(t)} = \normtwo{\vw_k(t)}\ge \bar{D}\sigmainit$. Using $\frac{\dd \vw_k}{\dd t} \subseteq a_k\hone[\eps]$ again we have  
	\[\frac{\dd \Gamma^{\delta/3}_i(s_k\vw_k(t))}{\dd t} \ge \frac{\delta h_{\min} \bar{D}\sigmainit}{3}.\]
	Thus if we pick
	\[
	T_0 := \max\left\{\frac{3}{\delta h_{\min} \bar{D}} \max_{i\in [n],k\in [m]}\{-\Gamma^{\delta/3}_i(s_k\barvw_k)\}, 0 \right\}
	\]
	and set $\sigmainitmax \le e^{-T_0} \frac{\sqrt{\frac{\eps}{m}}}{\normM{\barvtheta}}$ then it holds that $T_0\le T^\eps_{\sigmainit}$ for all $\sigmainit\le\sigmainitmax$ and that
	\begin{align*}
	\Gamma^{\delta/3}_i(s_k\vw_k(T_0)) 
	&\ge \frac{\delta h_{\min} \bar{D}\sigmainit T_0}{3} + \Gamma^{\delta/3}_i(s_k\vw_k(0)) \\
	&\ge \sigmainit\left(\frac{\delta h_{\min} \bar{D} T_0}{3} + \Gamma^{\delta/3}_i(s_k\bar{\vw}_k(0)) \right)\\
	&\ge  0,
	\end{align*} 
	which implies \eqref{eq:non_sym_first_converge_to_cone_realistic}.
	
	Finally, by \eqref{eq:non_sym_first_phase_growth_control}, it suffices to pick $A = e^{-T_0}\min_{k\in[m]}\normtwo{\bar{\vw}_k}$ and $B = e^{T_0}\max_{k\in[m]}\normtwo{\bar{\vw}_k}$.
\end{proof}

\subsubsection{Proof of Lemma \ref{lem:non-sym-phase1-main}}
Note that $G(\vw) = \dotp{\vw}{\vmuplus}$ for $\vw\in\cone[\delta/3]$ and  $G(\vw) = \dotp{\vw}{\vmusubt}$ for $\vw\in-\cone[\delta/3]$. Similar to the first-phase analysis to the symmetric case, we use $\phisg(\tildevtheta_0, t)$ to denote the trajectory of gradient flow on $\tildeLoss$:
\[
	\tildeLoss(\vtheta) := \ell(0) + \sum_{k \in [m]} a_k G(\vw_k) .
\]
Throughout this subsection, we will set $T_0$ and $\eps$ as defined in the proof of \Cref{lem:non_sym_first_converge_to_cone_realistic}, and therefore by \Cref{lem:non_sym_first_converge_to_cone_realistic}, we know there is $\sigmainitmax>0$, s.t. $a_k(T_0)\vw_k(T_0)\in \cone[\delta/3]$ for all $\sigmainit\le \sigmainitmax$.
This means the dynamics of $\tildevtheta(t) = (\tilde{\vw}_1(t), \dots, \tilde{\vw}_m(t), \tilde{a}_1(t), \dots, \tilde{a}_m(t)) = \phisg(\vtheta(T_0), t - T_0)$ can be described by linear ODE for $T_0 \le t\le T^\eps_{\sigmainit}$.
\begin{align*}
	 \textrm{If $\bara_k > 0$:} & \qquad \frac{\dd \tilde{\vw}_k}{\dd t} = \tilde{a}_k \vmuplus, \qquad \frac{\dd \tilde{a}_k}{\dd t} = \dotpsm{\tilde{\vw}_k}{\vmuplus};\\
	 \textrm{If $\bara_k < 0$:} & \qquad \frac{\dd \tilde{\vw}_k}{\dd t} = \tilde{a}_k \vmusubt, \qquad \frac{\dd \tilde{a}_k}{\dd t} = \dotpsm{\tilde{\vw}_k}{\vmusubt}.
\end{align*}

Let $\mM_+ := \begin{bmatrix}
	\vzero & \vmuplus \\
	(\vmuplus)^{\top} & 0
	\end{bmatrix}$ and $\mM_- := \begin{bmatrix}
	\vzero & \vmusubt \\
	(\vmusubt)^{\top} & 0
	\end{bmatrix}$. The largest eigenvalues for $\mM_+$ and $\mM_-$ are $\lambda_0^+ :=\normtwo{\vmuplus}$ and $\lambda_0^- :=\normtwo{\vmusubt}$ respectively. Then the above linear ODE can be solved as
\begin{align}
\textrm{If $\bara_k > 0$:} & \qquad \begin{bmatrix}
\tildevw_k(T_0 + t) \\
\tildea_k(T_0 + t)
\end{bmatrix} = \exp(t\mM_+)\begin{bmatrix}
\vw_k(T_0) \\
a_k(T_0)
\end{bmatrix}; \label{eq:mmplusexp} \\
\textrm{If $\bara_k < 0$:} & \qquad \begin{bmatrix}
\tildevw_k(T_0 + t) \\
\tildea_k(T_0 + t)
\end{bmatrix} = \exp(t\mM_-)\begin{bmatrix}
\vw_k(T_0) \\
a_k(T_0)
\end{bmatrix}. \label{eq:mmsubtexp}
\end{align}
\begin{lemma} \label{lm:non-sym-phisg-growth}
	Let $\tildevtheta(t) = \phisg(\vtheta(T_0), t-T_0)$. Then for all $T_0\le t\le T^\eps_{\sigmainit}$, it holds that 
	\[
		\normMsm{\tildevtheta(t)} \le \exp((t-T_0)\lambda^+_0)\normMsm{\tildevtheta(T_0)}.
	\]
\end{lemma}
\begin{proof}
	By \Cref{ass:vmuplus-greater}, we have $\lambda_0^+>\lambda_0^-$. 
	By definition and Cauchy-Schwartz inequality, 
	\[
		\normtwo{\frac{\dd \tilde{\vw}_k}{\dd t}} \le  \lambda^+_0 \abssm{\tilde{a}_k}, \qquad
		\abs{\frac{\dd \tilde{a}_k}{\dd t}} \le \lambda^+_0 \normtwosm{\tilde{\vw}_k}.
	\]
	So we have $\normMsm{\tilde{\vtheta}(t)} \le \normMsm{\vtheta(T_0)} + \int_{T_0}^{T^\eps_{\sigmainit}} \lambda^+_0 \normMsm{\tilde{\vtheta}(\tau)} \dd\tau$. Then we can finish the proof by \Gronwall's inequality \eqref{eq:gron-2}.
\end{proof}

\begin{lemma} \label{lm:non-sym-phase1-norm}
	For  $\vtheta(T_0)$ with $\abssm{a_k(T_0)} = \normtwo{\vw_k(T_0)}$ and $a_k(T_0)\vw_k(T_0)\in\cone[\delta/3]$, we have
	\[
	\normM{\vtheta(t) - \phisg(\vtheta(T_0), t-T_0)} \le \frac{4m\normMsm{\vtheta(T_0)}^3}{\lambda_0^+} \exp(3\lambda_0 (t-T_0)),
	\]
	for all $T_0\le t\le \frac{1}{\lambda_0^+} \ln \frac{\sqrt{\min\{\eps,\lambda_0^+\}}}{\sqrt{4m}\normMsm{\vtheta(T_0)}}$.
\end{lemma}

\begin{proof}
	Let $\tilde{\vtheta}(t) = \phisg(\vtheta(T_0), t-T_0)$. 
	Let
	\[
		t_0 := \min\{T^\eps_{\sigmainit}, \inf\{t \ge T_0 : \normMsm{\vtheta(t)} \ge 2\normMsm{\vtheta(T_0)} \exp(\lambda_0^+ (t-T_0)) \}.
	\]
	and it holds that $\forall T_0\le t\le t_0$, all neurons of $\tildevtheta(t),\vtheta(t)$ are either in $\cone[\delta/3]$ or $-\cone[\delta/3]$, thus $\tildevtheta(t),\vtheta(t)$ are in the same differentiable region of $\tildeLoss$.
	By \Cref{cor:diff-Loss-tildeLoss-smooth}, the following holds for a.e. $t \ge 0$,
	\begin{align*}
	\normM{\frac{\dd \vtheta}{\dd t} - \frac{\dd \tildevtheta}{\dd t}} &\le \sup\left\{\normMsm{\vdelta  - \nabla \tildeLoss(\vtheta)} : \vdelta \in  \cpartial \Loss(\vtheta) \right\} + \normMsm{\nabla \tildeLoss(\vtheta) - \nabla \tildeLoss(\tildevtheta)} \\
	&\le m\normMsm{\vtheta(t)}^3 + \lambda^+_0\normMsm{\vtheta - \tilde{\vtheta}}.
	\end{align*}
	Then we can argue as the proof for \Cref{lm:phase1-main} to show that
	\[
		\normMsm{\vtheta(t) - \tilde{\vtheta}(t)} \le \frac{4m\normMsm{\vtheta(T_0)}^3}{\lambda_0} \exp(3\lambda_0^+ (t-T_0))
	\]
	for all $t \in [T_0, t_0]$. If $t_0 < T_0 + \frac{1}{2\lambda_0^+} \ln \frac{\min\{\lambda_0^+,\eps\}}{4m\normMsm{\vtheta(T_0)}^2}$, then  for all $T_0\le t\le t_0 $, we have
	\[
		\normM{\vtheta(t)}\le \normM{\vtheta(T_0)}\sqrt{\frac{\min\{ \lambda_0^+,\eps\}}{4m\normM{\vtheta(T_0)}^2}}< \sqrt{\frac{\eps}{m}},
	\]
	which implies that $t_0 < T^\eps_{\sigmainit}$ by definition of $T^\eps_{\sigmainit}$. Moreover, \begin{align*}
		\normMsm{\vtheta(t)}
		&\le \normMsm{\tilde{\vtheta}(t)} + \frac{4m\normMsm{\vtheta(T_0)}^3}{\lambda_0^+} \exp(3\lambda_0^+ (t-T_0)) \\
		&\le \normMsm{\tilde{\vtheta}(t)} + \frac{4m\normMsm{\vtheta(T_0)}^2}{\lambda_0^+} \exp(2\lambda_0^+ (t_0-T_0)) \cdot \normMsm{\vtheta(T_0)} \exp(\lambda_0^+ (t-T_0)) \\
		&< \normMsm{\tilde{\vtheta}(t)} + \normMsm{\vtheta(T_0)} \exp(\lambda_0^+ (t-T_0)).
	\end{align*}
	By \Cref{lm:non-sym-phisg-growth}, $\normMsm{\tilde{\vtheta}(t)} \le \normMsm{\vtheta(T_0)} \exp(\lambda_0^+ (t-T_0))$. So $\normMsm{\vtheta(t)} < 2\normMsm{\vtheta(T_0)} \exp(\lambda_0^+ (t-T_0))$ for all $T_0 \le t \le t_0$, which contradicts to the definition of $t_0$. 
	Therefore, $t_0 \ge \frac{1}{2\lambda_0^+} \ln \frac{\min\{\eps,\lambda_0^+\}}{4m\normMsm{\vtheta(T_0)}^2} = \frac{1}{\lambda_0^+} \ln \frac{\sqrt{\min\{\eps,\lambda_0^+\}}}{\sqrt{4m}\normMsm{\vtheta(T_0)}}$.
\end{proof}

\begin{proof}[Proof for \Cref{lem:non-sym-phase1-main}]	
	Let $r_{\max} := \frac{\sqrt{\min\{\lambda_0^+, \eps\}}}{2}$ and $C_1 := \sigmainit r_{\max}^{-3}$. We only need to prove the statements for all $\sigmainit < \sigmainitmax = C_1 r_{\max}^3$.
	
	We fix a pair of $\sigmainit < \sigmainitmax$ and $r < r_{\max}$ satisfying $\sigmainit < C_1 r^3$. For convenience, we use $\barvb, T_1$ to denote $\barvb(\sigmainit), T_1(\sigmainit, r)$ for short.
	
	Let $\vtheta(t) = \phitheta(\sigmainit\barvtheta_0, t)$.
	It is easy to see that the prerequisites of \Cref{lem:non_sym_first_converge_to_cone_realistic} are satisfied with probability 1.
	Below we only focus on the case where the prerequisites of \Cref{lem:non_sym_first_converge_to_cone_realistic} are satisfied.
	Let $T_0, \sigmainitmax, A, B$ be the constants from \Cref{lem:non_sym_first_converge_to_cone_realistic}. Let $\tilde{\vtheta}(t) = \phisg(\vtheta(T_0), t-T_0)$.
	
	For $\bara_k > 0$, we define
	\[
		\barb_k := \frac{\dotpsm{\vw_k(T_0)}{\barvmuplus} + a_k(T_0)}{2\sqrt{m} \sigmainit},
	\]
	and for $\bara_k < 0$, we define
	\[
		\barb_k := \frac{\dotpsm{\vw_k(T_0)}{\barvmusubt} + a_k(T_0)}{2 (\sqrt{m} \sigmainit)^{1-\kappa}}.
	\]
	
	$r \le r_{\max}$ and $\sigmainit < \sigmainitmax$.
	
	\myparagraph{Proof for Item 1.}  By \Cref{lm:non-sym-phase1-norm}, we have
	\begin{equation} \label{eq:non-sym-phase1-main-nnn}
		\normMsm{\vtheta(T_0+T_1) - \tildevtheta(T_0+T_1)} \le \frac{4m\normMsm{\vtheta(T_0)}^3}{\lambda_0^+} \exp(3\lambda_0^+ T_1) = \frac{4\normMsm{\vtheta(T_0)}^3}{\lambda_0^+ \sqrt{m} \sigmainit^3} r^3.
	\end{equation}
	Now we turn to characterize $\tildevtheta(T_0+T_1)$. Note that $\zvmuplus := \frac{1}{\sqrt{2}}[\barvmuplus, 1]^\top$ and $\zvmusubt := \frac{1}{\sqrt{2}}[\barvmusubt, 1]^\top$ are the top eigenvectors of $\mM_+$ and $\mM_-$ respectively. Let $\kappa := 1- \frac{\normtwosm{\vmusubt}}{\normtwosm{\vmuplus}}$. Recall that $T_1 := \frac{1}{\lambda^+_0} \ln \frac{r}{\sqrt{m} \sigmainit}$. Then for $\bara_k > 0$, we have
	\[
		\exp(T_1 \lambda_0^+) \zvmuplus (\zvmuplus)^{\top} \begin{bmatrix}
		\vw_k(T_0) \\
		a_k(T_0)
		\end{bmatrix} = \left(\frac{r}{\sqrt{m} \sigmainit}\right) \zvmuplus (\zvmuplus)^{\top} \begin{bmatrix}
		\vw_k(T_0) \\
		a_k(T_0)
		\end{bmatrix} = r \barb_k \begin{bmatrix}
		\barvmuplus \\
		1
		\end{bmatrix},
	\]
	where the last equality is by definition of $\barb_k$. Similarly for $\bara_k > 0$, we have
	\[
		\exp(T_1 \lambda_0^-)\zvmusubt (\zvmusubt)^{\top} \begin{bmatrix}
		\vw_k(T_0) \\
		a_k(T_0)
		\end{bmatrix} = \left(\frac{r}{\sqrt{m} \sigmainit}\right)^{1-\kappa} \zvmusubt (\zvmusubt)^{\top} \begin{bmatrix}
		\vw_k(T_0) \\
		a_k(T_0)
		\end{bmatrix}
		= r^{1-\kappa} \barb_k \begin{bmatrix}
		\barvmusubt \\
		1
		\end{bmatrix}.
	\]
	Combining these with \eqref{eq:mmplusexp} and \eqref{eq:mmsubtexp}, then for $\bara_k > 0$ we have
	\begin{align*}
		\normtwo{
		\begin{bmatrix}
		\tildevw_k(T_0 + t) \\
		\tildea_k(T_0 + t)
		\end{bmatrix}
		- r \barb_k \begin{bmatrix}
		\barvmuplus \\
		1
		\end{bmatrix}} &\le \normtwo{\left(\exp(T_1\mM_+) - \exp(T_1\lambda_0^+) \zvmuplus (\zvmuplus)^{\top}\right) \begin{bmatrix}
		\vw_k(T_0) \\
		a_k(T_0)
		\end{bmatrix}
	} \\
	&\le \normtwo{\begin{bmatrix}
	\vw_k(T_0) \\
	a_k(T_0)
	\end{bmatrix}} \\
	&\le \sqrt{2} \normMsm{\vtheta(T_0)}.
	\end{align*}
	and for $\bara_k < 0$ we have
	\begin{align*}
	\normtwo{
		\begin{bmatrix}
		\tildevw_k(T_0 + t) \\
		\tildea_k(T_0 + t)
		\end{bmatrix}
		- r^{1-\kappa} \barb_k \begin{bmatrix}
		\barvmusubt \\
		1
		\end{bmatrix}} &\le \normtwo{\left(\exp(T_1\mM_-) - \exp(T_1\lambda_0^-) \zvmusubt (\zvmusubt)^{\top}\right) \begin{bmatrix}
		\vw_k(T_0) \\
		a_k(T_0)
		\end{bmatrix}
	} \\
	&\le \normtwo{\begin{bmatrix}
		\vw_k(T_0) \\
		a_k(T_0)
		\end{bmatrix}} \\
	&\le \sqrt{2} \normMsm{\vtheta(T_0)}.
	\end{align*}
	Then by definition of $\Delta \vtheta$ and \eqref{eq:non-sym-phase1-main-nnn}, we have
	\begin{align*}
		\normMsm{\Delta \vtheta} \le \frac{4\normMsm{\vtheta(T_0)}^3}{\lambda_0^+ \sqrt{m} \sigmainit^3} r^3 + \sqrt{2} \normMsm{\vtheta(T_0)}.
	\end{align*}
	Applying the upper bound $\normtwo{\vw_k(T_0)}\le B\sigmainit$ from \Cref{lem:non_sym_first_converge_to_cone_realistic}, we then have
	\[
		\normMsm{\Delta \vtheta} \le \frac{4B^3\sigmainit^3}{\lambda_0^+ \sqrt{m} \sigmainit^3} r^3 + \sqrt{2} B \sigmainit = \frac{4B^3}{\lambda_0^+ \sqrt{m}} r^3 + \sqrt{2} B \sigmainit,
	\]
	Finally, recalling that $\sigmainit \le C_1r^3$, we can conclude that $\normMsm{\Delta \vtheta} \le C_2r^3$, where $C_2 := \frac{4B^3}{\lambda_0^+ \sqrt{m}} + \sqrt{2} B C_1$.
	
	\myparagraph{Item 2.} Now it only remains to lower and upper bound $\abssm{\barb_k}$. By \Cref{lem:non_sym_first_converge_to_cone_realistic}, $a_k(T_0)\vw_k(T_0) \in \cone[\delta/3]$. Then $\sgn(a_k(T_0)) \dotpsm{\vw_k(T_0)}{\barvmuplus} \ge 0$ and thus
	\begin{align*}
		\text{if } \bara_k > 0: &~~ \dotpsm{\vw_k(T_0)}{\barvmuplus} + a_k(T_0) \in \big[\normtwosm{\vw_k(T_0)}, 2 \cdot \normtwosm{\vw_k(T_0)}\big] \subseteq [A\sigmainit, 2B\sigmainit]; \\
		\text{if } \bara_k < 0: &~~ \dotpsm{\vw_k(T_0)}{\barvmusubt} + a_k(T_0) \in \big[-2 \cdot \normtwosm{\vw_k(T_0)}, -\normtwosm{\vw_k(T_0)}\big] \subseteq [-2B\sigmainit, -A\sigmainit].
	\end{align*}
	Then for every $\barb_k$,
	\begin{align*}
		\text{if } \bara_k > 0: & \qquad \abssm{\barb_k} = \frac{\abssm{\dotpsm{\vw_k(T_0)}{\barvmuplus} + a_k(T_0)}}{2\sqrt{m} \sigmainit} \in \left[ \frac{A}{2\sqrt{m}}, \frac{B}{\sqrt{m}} \right] ; \\
		\text{if } \bara_k < 0: & \qquad \abssm{\barb_k} = \frac{\abssm{\dotpsm{\vw_k(T_0)}{\barvmusubt} + a_k(T_0)}}{2 (\sqrt{m} \sigmainit)^{1-\kappa}}  \in \left[ \frac{\sigmainit^{\kappa} A}{2m^{(1-\kappa)/2}}, \frac{\sigmainit^{\kappa} B}{m^{(1-\kappa)/2}} \right].
	\end{align*}
	Letting $\bar{A} := \frac{A}{2m^{(1-\kappa)/2}}$ and $\bar{B} := \frac{B}{m^{(1-\kappa)/2}}$ completes the proof.
\end{proof}

\subsection{Phase II}

As shown in our analysis for Phase I, if the intialization scale is small, the
weight vectors of neurons with $\bara_k > 0$ move towards the direction of
$\barvmuplus$, and all the other neurons are negligible. Now we show that the
dynamic of $\vtheta(t)$ is close to that of a one-neuron dynamic in a similar
manner as we do for the symmetric case.

First we slightly extend the definition of embedding. For $\hatvtheta = (\hatvw_1, \hatvw_2, \hata_1, \hata_2)$ and an embedding vector $\vb \in \R^m$, we say that $\vb$ is compatible with $\hatvtheta$ if the following holds:
\begin{enumerate}
	\item If $\brmsplus = 0$, then $\normtwosm{\hatvw_1} = \abssm{\hata_1} = 0$;
	\item If $\brmsminus = 0$, then $\normtwosm{\hatvw_2} = \abssm{\hata_2} = 0$.
\end{enumerate}
When $\vb$ is compatible with $\hatvtheta$, we define the (exact) embedding from two-neuron into $m$-neuron neural nets as $\pi_{\vb}(\hatvtheta) := (\vw_1, \dots, \vw_m, a_1, \dots, a_m)$, where
\[
a_k   = \begin{cases}
\frac{b_k}{\brmsplus} \hat{a}_1, & \textrm{if } b_k > 0 \\
\frac{b_k}{\brmsminus} \hat{a}_2,& \textrm{if } b_k < 0 \\
\vzero, & \textrm{if } b_k = 0 \\
\end{cases}, \qquad
\vw_k = \begin{cases}
\frac{b_k}{\brmsplus} \hat{\vw}_1, & \textrm{if } b_k > 0 \\
\frac{b_k}{\brmsminus} \hat{\vw}_2, &\textrm{if } b_k < 0 \\
\vzero, &\textrm{if } b_k = 0 \\
\end{cases}.
\]
One can easily show that \Cref{lm:two-neuron-embedding} continue to hold when $\vb$ is compatible with $\hatvtheta$.

\begin{lemma} \label{lm:nonsym-phase2-main}
Let $\barvb(\sigmainit)$ be the same vector as in the statement of \Cref{lem:non-sym-phase1-main}.  Let  $T_{12}(\sigmainit):=T_0 + \frac{1}{\lambda^+_0} \ln \frac{1}{\sqrt{m} \sigmainit}$ and $T_2(r) := \frac{1}{\lambda^+_0}\ln \frac{1}{r}$. For width $m \ge 1$, the following statements hold with probability $1-2^{-m}$ over the random draw of $\barvtheta_0 = (\bar{\vw}_1, \dots, \bar{\vw}_m, \bar{a}_1, \dots, \bar{a}_m) \sim \Dinit(1)$. Let $\sigma_1, \sigma_2, \dots$ be any sequence of initialization scales so that $\sigma_j$ converges to $0$ as $j \to +\infty$ and the limit $\hatvb := \lim_{j \to +\infty} \barvb(\sigma_j)$ exists.
\begin{enumerate}
	\item $\brmsplus[\hatb] > 0$ and $\brmsminus[\hatb] = 0$;
	\item For the two-neuron dynamics starting with rescaled initialization in the direction of
	$\hatvtheta := (\brmsplus[\hatb]\barvmuplus, \vzero, \brmsplus[\hatb], 0)$,
	the following limit exists for all $t \ge 0$,
	\begin{equation}
	\tildevtheta(t) := \lim_{r \to 0} \phitheta\left(r\hatvtheta, T_2(r) + t\right) \ne \vzero;
	\end{equation}
	\item For the $m$-neuron dynamics of $\vtheta_j(t)$ with initialization scale $\sigmainit = \sigma_j$, the following holds for all $t \ge 0$,
	\begin{equation}
	\lim_{j \to \infty} \vtheta_{j}\left(T_{12}(\sigma_j) + t\right) = \pi_{\hatvb}(\tildevtheta(t)).
	\end{equation}
\end{enumerate}
\end{lemma}

\begin{proof}
	The proof is similar to \Cref{lm:phase2-main} for the symmetric case. Apply \Cref{thm:phase-2-general} and then the lemma is straightforward.
\end{proof}

\subsection{Phase III}

In Phase III, we show that the dynamic of $\vtheta(t)$ converges to the same
classifier as the one-neuron dynamic.

Let $\Splus := \argmin_{i \in [n]} \left\{ y_i \dotpsm{\vwplus}{\vxplus_i}
\right\} \subseteq [n]$. Let $\Delta^{h-1} = \{ \vp \in \R^h : \sum_{i \in [h]}
p_i = 1, p_i \ge 1 \}$ be the probability simplex. Let $\Lambdaplus := \left\{
\vlambda \in \Delta^{n-1} : \lambda_i = 0, \forall i \notin \Splus \right\}$.

The theorem below characterizes the solution found by the one-neuron dynamic.
\begin{theorem} \label{thm:one-neuron-max-margin}
	Under \Cref{ass:lin}, for $m = 1$, if initially $a_1 = \normtwosm{\vw_1}$, $\dotp{\vw_1}{\vwopt} > 0$, then $\vtheta(t)$ directionally converges to the following global-max-margin direction,
	\[
	\lim_{t \to +\infty} \frac{\vtheta(t)}{\normtwosm{\vtheta(t)}} = \frac{1}{\sqrt{2}}(\vwplus, 1).
	\]
\end{theorem}
\begin{proof}
	By \Cref{thm:loss-convergence}, $\Loss(\vtheta(t)) \to 0$. Then by \Cref{thm:converge-kkt-margin}, $\frac{\vtheta(t)}{\normtwosm{\vtheta(t)}}$ converges along a KKT-margin direction. Combining this with \Cref{lm:weight-zero}, we know that this direction must has the form $\frac{1}{\sqrt{2}}(\barvw, 1)$ for some $\barvw \in \sphS^{d-1}$.
	
	By \Cref{def:kkt-margin-lkrelu}, $y_i \cdot \frac{1}{2}
	\phi(\dotpsm{\barvw}{\vx_i}) > 0$ and $\barvw$ can be expressed by a convex
	combination of $y_i \phi'(\dotpsm{\barvw}{\vx_i}) \vx_i$ among $i \in
	\argmin\{\frac{1}{2} \phi(\dotpsm{\barvw}{\vx_i})\}$. Equivalently. we know
	that $y_i \dotpsm{\barvw}{\vxplus_i}$ and $\barvw$ can be expressed by a
	convex combination of $y_i \vxplus_i$ among $i \in \Splus$. Then the only
	possibility is $\barvw = \vwplus$.
\end{proof}

Now we turn to analyze the trajectory of $\vtheta(t)$ on $m$-neuron neural net. First we prove the following lemma, then we prove \Cref{thm:nonsym-local-opt} for local-max-margin directions.
\begin{lemma} \label{lm:gamma-minimax}
	Let $\Theta_- := \{ \vtheta = (\vw_1, \dots, \vw_m, a_1, \dots, a_m) : m \ge 1, a_k \le 0 \}$. Then we have the following characterization for the global maximum of the normalized margin on the dataset $\{ (\vx_i, y_i) : i \in \Splus \}$:
	\[
		\sup_{\vtheta \in \Theta_-} \left\{\frac{\min_{i \in \Splus} q_i(\vtheta)}{\normtwosm{\vtheta}^2}\right\} = \inf_{\lambda \in \Lambdaplus} \sup_{\vu \in \sphS^{d-1}} \left\{ -\frac{1}{2} \sum_{i \in [n]} \lambda_i y_i \phi(\dotpsm{\vu}{\vx_i}) \right\}
	\]
\end{lemma}
\begin{proof}
	The proof is inspired by \citet[Proposition 12]{chizat20logistic}.
	By \Cref{lm:a-w-equal}, the maximum normalized margin is attained when $\abssm{a_k} = \normtwosm{\vw_k}$ for all $k \in [m]$. Note that we can rewrite each neuron output $a_k \phi(\dotpsm{\vw_k}{\vx_i})$ as $-a_k^2 \phi(\dotpsm{\vw_k/\normtwosm{\vw_k}}{\vx_i})$ for any such solution, and it is easy to see $\sum_{k \in [m]} a_k^2 = \frac{1}{2}\normtwosm{\vtheta}^2$.
	Let $\mathcal{V}$ be the set of probability distributions supported on finitely many points of $\sphS^{d-1}$. Then
	\begin{align*}
		\sup_{\vtheta \in \Theta_-} \left\{\frac{\min_{i \in \Splus} q_i(\vtheta)}{\normtwosm{\vtheta}^2}\right\}
		&= \sup_{\substack{m \ge 1, \vp \in \Delta^{m-1} \\ \vu_1, \dots, \vu_m \in \sphS^{d-1}}} \min_{i \in \Splus} \left\{ -\frac{1}{2} y_i \sum_{k \in [m]} p_k \phi(\dotpsm{\vu_k}{\vx_i}) \right\} \\
		&= \sup_{\nu \in \mathcal{V}} \min_{i \in \Splus} \E_{\vu \sim \nu} \left[-\frac{1}{2} y_i \phi(\dotpsm{\vu}{\vx_i}) \right].
	\end{align*}
	By minimax theorem, we can swap the order between $\sup$ and $\min$ in the following way:
	\begin{align*}
		\sup_{\nu \in \mathcal{V}} \min_{i \in \Splus} \E_{\vu \sim \nu} \left[-\frac{1}{2}y_i \phi(\dotpsm{\vu}{\vx_i})\right]
		&= \sup_{\nu \in \mathcal{V}} \inf_{\vlambda \in \Lambdaplus} \E_{\vu \sim \nu} \left[-\frac{1}{2} \sum_{i \in \Splus} \lambda_i y_i \phi(\dotpsm{\vu}{\vx_i})\right] \\
		&= \inf_{\vlambda \in \Lambdaplus} \sup_{\nu \in \mathcal{V}} \E_{\vu \sim \nu} \left[-\frac{1}{2} \sum_{i \in \Splus} \lambda_i y_i \phi(\dotpsm{\vu}{\vx_i})\right] \\
		&= \inf_{\vlambda \in \Lambdaplus} \sup_{\vu \in \sphS^{d-1}} \left\{ -\frac{1}{2}\sum_{i \in \Splus} \lambda_i y_i \phi(\dotpsm{\vu}{\vx_i})\right\},
	\end{align*}
	which proves the claim.
\end{proof}

\begin{theorem} \label{thm:nonsym-local-opt}
	Let $\hatvtheta := (\frac{1}{\sqrt{2}} \vwplus, \frac{1}{\sqrt{2}}, \vzero, 0)$ and $P$ be a non-empty subset of $[m]$. Let $\bar{\mathcal{Q}}$ be the following subset of $\sphS^{D-1}$:
	\[
		\bar{\mathcal{Q}} := \{ \vtheta = (\vw_1, \dots, \vw_m, a_1, \dots, a_m) \in \sphS^{D-1} : a_k \ge 0 \text{ for all } k \in P \text{ and } a_k \le 0 \text{ otherwise} \}.
	\]
	For any embedding vect $\vb$ be an embedding vector satisfying the following:
	\begin{itemize}
		\item $\vb$ is compatible with $\hatvtheta$;
		\item $b_k > 0$ for all $k \in P$;
		\item $b_k = 0$ for all $k \notin P$;
	\end{itemize}
	the following statements are true under \Cref{ass:margin-negative-net},
	\begin{enumerate}
		\item $\pi_{\vb}(\hatvtheta)$ is a local maximizer of $\gamma(\vtheta)$ among $\vtheta \in \bar{\mathcal{Q}}$;
		\item If $\vtheta \in \bar{\mathcal{Q}}$ has the same normalized margin as $\pi_{\vb}(\hatvtheta)$ and $\vtheta$ is sufficiently close to $\pi_{\vb}(\hatvtheta)$, then $f_{\vtheta}(\vx) = f_{\hatvtheta}(\vx)$ for all $\vx \in \R^d$.
	\end{enumerate}
\end{theorem}
\begin{proof}
	It is easy to see that $\pi_{\vb}(\hatvtheta)$ is a KKT-margin direction with $\gamma(\pi_\vb(\hatvtheta)) = \gamma(\hatvtheta) = \frac{1}{2} \gammaplus$.
	Also, $\argmin_{i \in [n]}\{q_i(\pi_{\vb}(\hatvtheta))\} = \argmin_{i \in [n]} \{ q_i(\hatvtheta) \} =  \Splus$.
	Let $\epsilon > 0$ be a small constant such that the following holds whenever $\normMsm{\vtheta - \pi_{\vb}(\hatvtheta)} < \epsilon$:
	\begin{enumerate}
		\item $\sgn(\dotp{\vw_k}{\vx_i}) = \sgn(\dotp{\vwplus}{\vx_i})$ for all $i \in [n]$ and for all $k \in P$;
		\item $\argmin_{i \in [n]} \{q_i(\vtheta)\} \subseteq \Splus$.
	\end{enumerate}

	Let $\vtheta \in \bar{\mathcal{Q}}$ be any parameter satisfying $\normMsm{\vtheta - \pi_{\vb}(\hatvtheta)} < \epsilon$.
	We can decompose $\vtheta$ into $\vtheta^+ + \vtheta^-$, where $\vtheta^+ = (\vw_1^+, \dots, \vw_m^+, a_1^+, \dots, a_m^+)$, $\vtheta^- = (\vw_1^-, \dots, \vw_m^-, a_1^-, \dots, a_m^-)$, and
	\[
		\vw_k^+ = \onec{k \in P} \vw_k, a_k^+ = \onec{k \in P} a_k, \qquad \vw_k^- = \onec{k \notin P} \vw_k, a_k^- = \onec{k \notin P} a_k.
	\]
	Let $r_+ = \normtwosm{\vtheta^+}$ and $r_- = \normtwosm{\vtheta^-}$. Define $\barvtheta^+$ and $\barvtheta^-$ to be two unit-norm parameters so that $\vtheta^+ = r_+ \barvtheta^+$, $\vtheta^- = r_- \barvtheta^-$.
	Then we have
	\[
		\gamma(\vtheta)
		= \min_{i \in \Splus} \{q_i(\vtheta)\}
		= \min_{i \in \Splus} \left\{ q_i(\vtheta^+) + q_i(\vtheta^-) \right\}
		= \min_{i \in \Splus} \left\{ r_+^2 q_i(\barvtheta^+) + r_-^2 q_i(\barvtheta^-) \right\}
	\]
	Note that $r_+^2 + r_-^2 = 1$. By minimax theorem (similar to \Cref{lm:gamma-minimax}),
	\begin{align*}
		\min_{i \in \Splus} \left\{ r_+^2 q_i(\barvtheta^+) + r_-^2 q_i(\barvtheta^-) \right\}
		\le \min_{\lambda \in \Lambdaplus} \max\left\{ \sum_{i \in \Splus} \lambda_i q_i(\barvtheta^+), \sum_{i \in \Splus} \lambda_i q_i(\barvtheta^-) \right\}.
	\end{align*}
	By definition of $\vwplus$ and KKT conditions, we can find $\vlambda^* \in \Lambdaplus$ so that $\sum_{i \in \Splus} \lambda^*_i \vxplus = \gammaplus \vwplus$.
	Letting $\vlambda = \vlambda^*$ for the above inequality, we can obtain
	\[
		\gamma(\vtheta) \le \max\left\{ \sum_{i \in \Splus} \lambda^*_i q_i(\barvtheta^+), \sum_{i \in \Splus} \lambda^*_i q_i(\barvtheta^-) \right\}.
	\]
	We only need to prove that both $\sum_{i \in \Splus} \lambda^*_i q_i(\barvtheta^+)$ and $\sum_{i \in \Splus} \lambda^*_i q_i(\barvtheta^-)$ are no more than $\frac{1}{2} \gammaplus$.
	Note that combining \Cref{ass:margin-negative-net} and \Cref{lm:gamma-minimax} directly implies that $\sum_{i \in \Splus} \lambda^*_i q_i(\barvtheta^-) < \frac{1}{2} \gammaplus$.
	Now we focus on $\sum_{i \in \Splus} \lambda^*_i q_i(\barvtheta^+)$.

	According to our choice of $\epsilon$, we have $a_k\phi(\dotpsm{\vw_k}{\vx_i}) = \dotpsm{a_k\vw_k}{\vxplus_i}$.
	For $\sum_{i \in \Splus} \lambda^*_i q_i(\barvtheta^+)$, we have
	\begin{align*}
		\sum_{i \in \Splus} \lambda^*_i q_i(\barvtheta^+)
		= \sum_{i \in \Splus} \lambda^*_i y_i \sum_{k \in P} \dotpsm{a_k \vw_k}{\vxplus_i}
		&= \sum_{k \in P} a_k \dotp{\vw_k}{\sum_{i \in \Splus} \lambda^*_i y_i\vxplus_i} \\
		&= \sum_{k \in P} a_k \dotp{\vw_k}{\gammaplus \vwplus}.
	\end{align*}
	By Cauchy-Schwartz inequality,
	\[
		\sum_{k \in P} a_k \dotp{\vw_k}{\gammaplus \vwplus}
		\le \sqrt{\sum_{k \in P} a_k^2} \cdot \sqrt{\sum_{k \in P} \dotp{\vw_k}{\gammaplus \vwplus}^2}
		\le \frac{1}{\sqrt{2}} \cdot \frac{1}{\sqrt{2}} \gammaplus
		= \frac{1}{2}\gammaplus.
	\]
	This proves that $\sum_{i \in \Splus} \lambda^*_i q_i(\barvtheta^+) \le \frac{1}{2} \gammaplus$,
	and thus $\gamma(\vtheta) \le \frac{1}{2} \gammaplus = \gamma(\hatvtheta)$. Therefore Item 1 is true.

	For Item 2, we only need to note that the equality in $\gamma(\vtheta) \le \frac{1}{2} \gammaplus$ only holds if $r_{-} = 0$ and $\vw_k = a_k \vwplus$ for all $k \in P$, so $f_{\vtheta}$ represents the same function as $f_{\hatvtheta}$.
\end{proof}

For proving \Cref{thm:nonsym_main}, we only need to show this:
\begin{theorem} \label{thm:nonsym_main2}
	For any sequence of $\sigma_1, \sigma_2, \dots$ converging to $0$,
	there is a subsequence $\sigma_{p_1}, \sigma_{p_2}, \dots$ and a constant $\sigmainitmax$
	such that \Cref{thm:nonsym_main} holds for $\sigmainit = \sigma_{p_i}$ as long as $\sigma_{p_i} < \sigmainitmax$.
\end{theorem}
\begin{proof}[Proof for \Cref{thm:nonsym_main}]
	Assume to the contrary that \Cref{thm:nonsym_main} does not hold.
	Then there exists $\vx \in \R^d$ and a sequence of initialization scales $\sigma_1, \sigma_2, \dots$ converging to $0$
	such that $f^{\infty}(\vx) \ne \frac{1}{2}\phi(\dotp{\vwplus}{\vx})$ for any $\sigma_j$.
	However, by \Cref{thm:nonsym_main2}, we can find a subsequence $\sigma_{p_1}, \sigma_{p_2}, \dots$ and a constant $\sigmainitmax$
	such that $\frac{1}{2}\phi(\dotp{\vwplus}{\vx})$ holds for $\sigma_{p_i}$
	as long as $\sigma_{p_i} < \sigmainitmax$, contradiction.
\end{proof}

\begin{proof}[Proof for \Cref{thm:nonsym_main2}]
	With probability $1$ over the random draw of $\barvtheta_0 \sim \Dinit(1)$,
	by \Cref{lem:non-sym-phase1-main}, the prerequisites of \Cref{lem:non_sym_first_converge_to_cone_realistic} hold
	and we can find a subsequence of initialization scales $\sigma_{p_1}, \sigma_{p_2}, \dots$
	so that the limit $\hatvb := \lim_{j \to +\infty} \barvb(\sigma_{p_j})$ exists.
	
	Let $\vtheta_j(t) = \phitheta(\sigma_{p_j} \barvtheta_0, t)$.
	By \Cref{lm:nonsym-phase2-main}, with probability $1 - 2^{-m}$,
	$\lim_{j \to \infty} \vtheta_j(T_{12}(\sigma_{p_j}) + t) = \pi_{\hatvb}(\tildevtheta(t))$.
	By \Cref{thm:one-neuron-max-margin}, $\lim_{t \to +\infty} \frac{\tildevtheta(t)}{\normtwosm{\tildevtheta(t)}} = \frac{1}{\sqrt{2}}(\vwplus, \vzero, 1, 0) =: \tildevtheta_{\infty}$.
	Then we can argue in a similar way as \Cref{thm:sym_main} to show that
	for any $\epsilon > 0$ and $\rho > 0$,
	we can choose a time $t_1 \in \R$ such that
	$\normtwo{\frac{\vtheta_j(T_{12}(\sigma_{p_j}) + t_1)}{\normtwosm{\vtheta_j(T_{12}(\sigma_{p_j} + t_1))}} - \pi_{\hatvb}(\tildevtheta_{\infty})}$
	and $\normtwo{\vtheta_j(T_{12}(\sigma_{p_j}) + t_1)} \ge \rho$ for $\sigma_{p_j}$ small enough.

	By \Cref{cor:weight-equal}, the trajectory of gradient flow starting with $\sigma_{p_j} \barvtheta_0$ lies in the set 
	$\mathcal{Q} := \{ \vtheta : a_k \bara_k \ge 0 \text{ for all } k \in [m] \}$ for all $j \ge 1$,
	that is, every $a_k$ has the same sign as its initial value during training.
	By a variant of \Cref{thm:phase-4-general}, there exists $\sigmainitmax$ such that for all $\sigma_{p_j} < \sigmainitmax$,
	$\frac{\vtheta(t)}{\normtwosm{\vtheta(t)}} \to \barvtheta \in \mathcal{Q}$,
	where $\gamma(\barvtheta) = \gamma(\pi_{\barvb}(\tildevtheta_{\infty}))$
	and $\normtwosm{\barvtheta - \pi_{\barvb}(\tildevtheta_{\infty})} \le \delta$.
	Applying \Cref{thm:nonsym-local-opt} proves that $\finftime(\vx) = \frac{1}{2} \phi(\dotpsm{\vwplus}{\vx})$ for $\sigma_{p_j} < \sigmainitmax$.
\end{proof}

\section{Proofs for the Orthogonally Separable Case}\label{sec:orthogonal_separable}

In this section, we revisit the orthogonally separable setting considered by
\citet{phuong2021the}. Suprisingly, in this setting, all KKT points which
contains at least one positive neuron and negative neuron are indeed
global-max-margin directions and unique in function space. This means it is
possible to prove the global optimality of margin in \citet{phuong2021the}'s setting even without a trajectory-based
analysis.

\begin{definition}[Orthogonally Separable Data, \citealt{phuong2021the}]
\label{def:ortho-separ}
	A binary classification dataset $\{(\vx_1, y_1), \dots, (\vx_n, y_n)\}$ is
	called \emph{orthogonally separable} if for all $i, j\in[n]$, if
	$\vx_i^\top\vx_j>0$ whenever $y_i=y_j$ and $\vx_i^\top\vx_j\le0$ whenever $y_i=-y_j$.
\end{definition}

Let  $\vtheta = (\vw_1, \dots, \vw_m, a_1, \dots, a_m) \in \R^D$ and  $f_\vtheta(\vx) :=
\sum_{i=1}^m a_i\phi(\dotpsm{\vx}{\vw_i})$ where $\phi$ is ReLU, i.e.,
$\phi(x) =\max\{0,x\}$. The following theorem shows that for orthogonally
separable data, all KKT-margin directions are global-max-margin directions.
\begin{theorem}\label{thm:orthogonally_separable} Suppose the dataset is
 orthogonally separable, for all KKT-margin directions $\vtheta \in \sphS^D$, their corresponding functions $f_\vtheta$ are the same and
 thus they are all global-max-margin directions.
\end{theorem}

The \Cref{thm:orthogonally_separable} is a simple corollary of the following lemma \Cref{lem:orthogonally_separable}.

\begin{lemma}\label{lem:orthogonally_separable} If  $\vtheta$ satisfies the KKT
    conditions of \eqref{eq:prob-P}, then for $a_k \ne 0$,  $|a_k| =
    \normtwo{\vw_k}$ and $(\sum_{j:a_ja_k>0}a_j^2)\frac{\vw_k}{a_k}$ is the
    global minimizer of the following optimization problem \eqref{eq:prob-Q}:
\begin{equation}\label{eq:prob-Q}
    \min_{\vw} \quad \frac{1}{2}\normtwo{\vw}^2 
\quad\mathrm{s.t.\ }  \dotp{\vw}{\vx_i}\ge 1, \quad\text{for all } i\in[n] \text{ with } y_i=\sgn(a_k). \tag{Q} 
\end{equation}
In other words, all the non-zero $a_k,\vw_k$ can be split into $2$ groups according to the sign of $a_k$, where in each group, $\frac{\vw_k}{a_k}$ is the same. 
\end{lemma}

\begin{proof}[Proof of \Cref{thm:orthogonally_separable}]
By \Cref{lem:orthogonally_separable}, we know for any $\vtheta$ satisfying the
KKT condition of \eqref{eq:prob-P},
\begin{align}\label{eq:proof_orthogonally_separable_1}
\vw_k=\left(\sum_{j:a_ja_k>0}a_j^2\right)^{-1}a_k\vw^{\sgn(a_k)}, \end{align}
where $\vw^{\sgn(a_k)}$ ($\vw^{+}$ or $\vw^{-}$) are the unique global minimzer
of the constrained convex optimization of \eqref{eq:prob-Q}.

Thus $\normtwo{\vtheta}^2 = \sum_{i\in [m]}(|a_i|^2+\normtwosm{\vx_i}^2) =
2\sum_{i\in [m]}|a_i|^2 = \normtwosm{\vw^{-}}+\normtwosm{\vw^{+}}$ is the same
for all $\vtheta$ satisfying the condition in the theorem statement. Here the
last equality uses \eqref{eq:proof_orthogonally_separable_1} and $|a_k| =
\normtwo{\vw_k}$.

Next we check the uniqueness of $f_\vtheta$. For any $\vx$, we have 
\begin{align*}
f_\vtheta(\vx) = 
\sum_{k\in [m]} a_k\phi(\dotp{\vx}{\vw_k}) &= \phi\left(\dotp{\vx}{\sum_{k:a_k>0} a_k\vw_k}\right) + \phi\left(\dotp{\vx}{\sum_{k:a_k<0} a_k\vw_k}\right)\\
&=  \phi\left(\dotp{\vx}{\vw^{+}}\right) + \phi\left(\dotp{\vx}{\vw^{-}}\right), 
\end{align*}
which completes the proof.
\end{proof}

\begin{proof}[Proof of \Cref{lem:orthogonally_separable}]
By KKT conditions (\Cref{def:kkt-margin-lkrelu}), there exist $\lambda_1, \dots
\lambda_n \ge 0$, such that for each $k\in [m]$, there are $\hki[1], \dots, \hki[n]
\in \R$ such that for all $i\in [n]$,  $\hki \in
\cder{\phi}(\dotpsm{\vw_k}{\vx_i})$, and the following conditions hold:
\[
	\vw_k = a_k\sum_{i \in [n]} \lambda_i \hki y_i \vx_i,  \qquad a_k = \sum_{i \in [n]} \lambda_i y_i \phi(\vw_k^\top \vx_i),
\]
and $\lambda_i = 0$ whenever $y_if_\vtheta(\vx_i) > 1$. By \Cref{lm:a-w-equal},
$\normtwosm{\vw_k} = \abssm{a_k}$.

We claim that for all $i \in [n]$ so that $\lambda_i \hki>0$, it holds that $y_i = \sgn(a_k)$ and $\dotpsm{\vw_k}{\vx_i} > 0$. Let $i \in [n]$ be any index so that
$\lambda_i \hki > 0$. Then $\hki > 0$. By KKT conditions,
\begin{equation} \label{eq:ortheq1}
	\dotpsm{\vw_k}{\vx_i} =
	\dotp{a_k\sum_{j \in [n]} \lambda_j \hki[j] y_j \vx_j}{\vx_i} = a_k y_i \sum_{j
	\in [n]} \lambda_j \hki[j] \dotp{y_j \vx_j}{y_i \vx_i}.
\end{equation}
Since $\phi(x) = \max\{x,0\}$, it holds that $\dotpsm{\vw_k}{\vx_i}\ge 0$;
otherwise $\hki\in \cder{\phi}(\dotpsm{\vw_k}{\vx_i})=\{0\}$, which contradicts
to $\hki > 0$. Then \eqref{eq:ortheq1} implies that the product of $a_ky_i$ and
$\sum_{j \in [n]} \lambda_j \hki[j] \dotp{y_j \vx_j}{y_i \vx_i}$ is non-negative. By
orthogonal separability, $\dotp{y_j \vx_j}{y_i \vx_i}\ge 0$ and thus $\sum_{j
\in [n]} \lambda_j \hki[j] \dotp{y_j \vx_j}{y_i \vx_i} \ge \lambda_i \hki
\normtwosm{y_i\vx_i}^2 > 0$. Then we can conclude that $a_k y_i \ge 0$ and thus
$y_i = \sgn(a_k)$. Since $y_i = \sgn(a_k)$ and $a_k \ne 0$, indeed we have $a_k y_i > 0$.
Now using \eqref{eq:ortheq1} again, we obtain $\dotpsm{\vw_k}{\vx_i} \ge a_ky_i \cdot \lambda_i \hki
\normtwosm{y_i\vx_i}^2 > 0$ if $\lambda_i \hki>0$.

Furthermore, for any $a_k\neq 0$, since $\normtwosm{\vw_k}= |a_k| > 0$, there is
at least one index $j_* \in [n]$ such that $\lambda_{j_*} \hki[j_*] >0$ (otherwise
$\vw_k = \vzero$ by KKT conditions). For all $i\in [n]$, again by \eqref{eq:ortheq1}, it holds  that 
\[
	\sgn(a_k) y_i \dotpsm{\vw_k}{\vx_i} = |a_k|  \sum_{j
	\in [n]} \lambda_j \hki[j] \dotp{y_j \vx_j}{y_i \vx_i}
	\ge |a_k| \lambda_{j_*} \hki[j_*] \dotp{y_{j_*}\vx_{j_*}}{y_i\vx_i}>0,
\]
where the last inequality is from the assumption of orthogonally separability. This further implies $\hki = \onec{y_i = \sgn(a_k)}$ and thus 
$\vw_k = a_k \sum_{i=1}^{n} \onec{y_i = \sgn(a_k)} \lambda_i y_i \vx_i$ for all $k \in [m]$.

Therefore we can split the neurons with non-zero $a_k$ into two parts: $K^+ =
\{k \in [m] : a_k>0\}$, $K^- = \{k \in [m] : a_k<0\}$. Every $k \in K^+$ satisfies the following:
\begin{align}
	a_k &= \normtwosm{\vw_k},\\
	\vw_k &= a_k\sum_{i=1}^{n} \onec{y_i = 1} \lambda_i \vx_i.
\end{align}
This implies $\forall k\in K^+$, $\frac{\vw_k}{a_k} =
\frac{\vw_k}{\normtwosm{\vw_k}} = \sum_{i=1}^{n} \onec{y_i = 1} \lambda_i \vx_i$. Define $\barvw :=\sum_{k\in K^+}a_k\vw_k$, then
\[
	\barvw = \left(\sum_{k\in K^+}a^2_k\right) \sum_{i=1}^{n} \onec{y_i = 1} \lambda_i  \vx_i.
\]
Recall that $\lambda_i = 0$ whenever $y_i f_{\vtheta}(\vx_i) > 1$. When $y_i =
1$, $f_{\vtheta}(\vx_i)$ can be rewritten as
\[
	f_{\vtheta}(\vx_i) = \sum_{k \in [m]} a_k \onec{\sgn(a_k) = 1} \dotp{\vw_k}{\vx_i} = \dotp{\vx_i}{\barvw}.
\]
So we can verify that $\barvw$ satisfies the KKT conditions of the following constrained convex optimization problem:
\begin{align}
    \min_{\vw} & \normtwo{\vw}^2 \\
\mathrm{s.t.}  & \dotp{\vw}{\vx_i}\ge 1, \quad \text{for all } i\in[n] \text{ with } y_i=1. 
\end{align}
By convexity, $\barvw$ is the unique minimizer of the above problem. The negative part $K^-$ can be analyzed in the same way.
\end{proof}

\section{Additional Discussions}
\subsection{Illustrations for Figure \ref{fig:decision_boundary}} \label{sec:proof-figure}
In this section we further illustrate the the relationship between KKT-margin and max-margin directions, as the examples have showed in \Cref{fig:decision_boundary}.

\subsubsection{Left: Symmetric Data}

\paragraph{Example.} For some symmetric data, there are KKT-margin directions
with non-linear decision boundary (and thus by \Cref{thm:maxmar-linear} are not
global-max-margin directions).

Let $\lambda_i$ be the dual variable for $(x_i,y_i)$, then the KKT conditions
(\Cref{def:kkt-margin-lkrelu} and \Cref{lm:a-w-equal}) ask
\begin{enumerate}
    \item for all $k \in [m]$, $\vw_k \in \sum_{i \in [n]} \lambda_i y_i a_k \cder{\phi}(\vw_k^{\top}\vx_i) \vx_i$;
    \item for all $k \in [m]$, $\abssm{a_k} = \normtwosm{\vw_k}$;
    \item for all $i \in [n]$, if $q_i(\vtheta) \ne \qmin(\vtheta)$ then
    $\lambda_i = 0$ (recall that $q_i(\vtheta) = y_i f_{\vtheta}(\vx_i)$).
\end{enumerate}

For $\alphaLK=0$, the example is simpler. Consider the following case: the data
points are $\vx_1=(1,-1)$, $\vx_2=(1,0)$, $\vx_3=(1,1)$ with label $1$ and the
symmetric counterpart $\vx_4=(-1,1)$, $\vx_5=(-1,0)$, $\vx_6=(-1,-1)$ with label
$-1$. As we have proved, the global-max-margin solution is a linear function and
in this case $\vwopt=(1,0)$. On the other hand, for hidden neurons $m\geq 3$,
one KKT-margin direction is as follows:
\[
\left\{
\begin{array}{ll}
 a_1= 2^{-1/4}& \vw_1=\frac{1}{2^{3/4}}(1,1)\\
 a_2= 2^{-1/4}& \vw_2=\frac{1}{2^{3/4}}(1,-1)\\
 a_3= -1& \vw_3=(-1,0)\\
 a_k=0 & \vw_k=\vzero \qquad\qquad\qquad \text{for all } k>3.
\end{array}\right.
\]
In this case, all the data points  $\vx_i$ share the same output margin
$q_i(\vtheta)$, so they are all support vectors. A possible choice of dual
variables is $\vlambda=(\frac{1}{\sqrt{2}},0,\frac{1}{\sqrt{2}},0,1,0)$. It is
easy to verify that this KKT-margin direction does not have linear decision
boundary and is thus not global-max-margin.

For $\alphaLK>0$, we can adapt the above case to construct a KKT point. 
Let $\beta$ be a solution to the equation
\[(2\sin^2\beta+\cos\beta)\alphaLK^2-(1+\cos\beta)\alphaLK+\cos 2\beta=0.\] Let
the data be $\vx_1=(1,\cot\beta),\vx_2=(1,0),\vx_3=(1,-\cot\beta)$ with label
$1$ and the corresponding opposites
$\vx_4=(-1,-\cot\beta),\vx_5=(-1,0),\vx_6=(-1,\cot\beta)$ with label $-1$.
Then we can have a KKT point with
$\vlambda=(\frac{\sin^2\beta}{\cos\beta(1-\alphaLK)},0,\frac{\sin^2\beta}{\cos\beta(1-\alphaLK)},0,1-\frac{2\alphaLK\sin^2\beta}{(1-\alphaLK)\cos\beta},0)$
and
\[
\left\{
\begin{array}{ll}
 a_1= (2(1+\alphaLK)\cos\beta)^{-1/2}& \vw_1=a_1\sin\beta(\cot\beta,1)\\
 a_2= a_1& \vw_2=a_2\sin\beta(\cot\beta,-1)\\
 a_3= -(1+\alphaLK)^{-1/2}& \vw_3=-a_3(-1,0)\\
 a_k=0 & \vw_k= \vzero \qquad\qquad\qquad \text{for all } k>3.
\end{array}\right.
\]

When $\cos 2\beta=0$ we already have solution $\alphaLK=0$, and it is easy to verify that for any $\alphaLK\in[0,1)$ there is a solution $\beta$ that satisfy the KKT conditions. Thus in the leaky ReLU case we are considering in the previous chapters, there are also KKT-margin directions that have non-linear decision boundaries and therefore have sub-optimal margin.

\subsubsection{Middle and Right: Non-symmetric Data}
In \Cref{fig:decision_boundary} we further show two examples of non-symmetric data that gradient flow from small initialization converges to a linear-boundary classifier that has a suboptimal margin.

The idea of the middle plot dataset comes from \cite{shah2020pitfalls}. In the middle subplot, we exhibit a data example that is linear separable in the first dimension $x$ but not linear separable in the second dimension $y$. The data is distributed on $(A_\epsilon,1)$ and $(A_\epsilon, -1)$ with label $1$ and on $(-A_{\epsilon'},0)$ with label $-1$ (here $A_c=[c,\infty)$ is an interval in one dimension). We add identical entries $c$ to all the data in the third dimension $z$ so in the $x-y$ plane with $z=c$ the two-layer ReLU network can represent decision patterns with bias. 

To apply \Cref{thm:nonsym_main} on this dataset, we need to make $c$ smaller than $\epsilon$ and $\epsilon/\epsilon'\ll \alphaLK$, so that the points at $(\epsilon, 1)$ and $(\epsilon, -1)$ becomes the support vectors for the one-neuron function. Also we can control the principal direction by taking more data points from the positive class, and then gradient flow  will converge to the one-neuron max-margin solution as predicted by  \Cref{thm:nonsym_main}. This solution cannot be global max-margin when $\epsilon\ll 1$, as a two-neuron network can express a function where these two support vectors have much larger distances to the decision boundary (and possess larger output margins).

In the right plot, we add three hints to a linear separable dataset so that gradient flow converges to the solution with a linear decision boundary and suboptimal margin. The result follows from \Cref{thm:hinted_main}.

\subsubsection{Experimental Results}
We run gradient descent with small learning rate and 0.001 times the He
intialization \citep{He_2015_ICCV} on the two-layer LeakyReLU network for the
examples in \Cref{fig:decision_boundary}. The contours of the neural net outputs
are displayed in \Cref{fig:image2}. In the three settings the neural nets
actually converge to linear classifiers.
\begin{figure}[htbp]
\begin{subfigure}{0.33\textwidth}
\includegraphics[width=\linewidth]{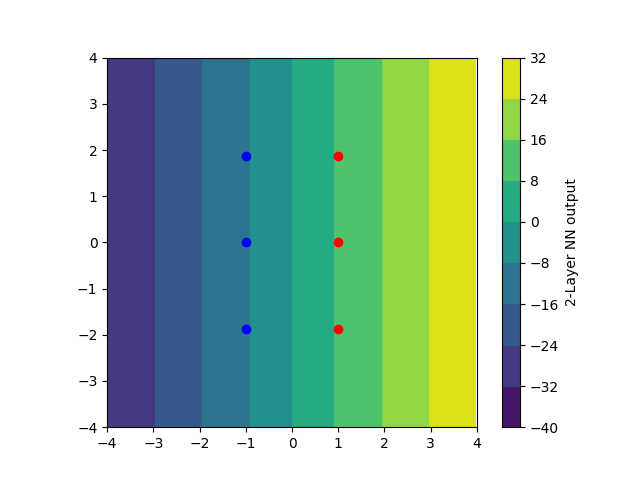} 
\caption{LEFT}
\end{subfigure}
\begin{subfigure}{0.33\textwidth}
\includegraphics[width=\linewidth]{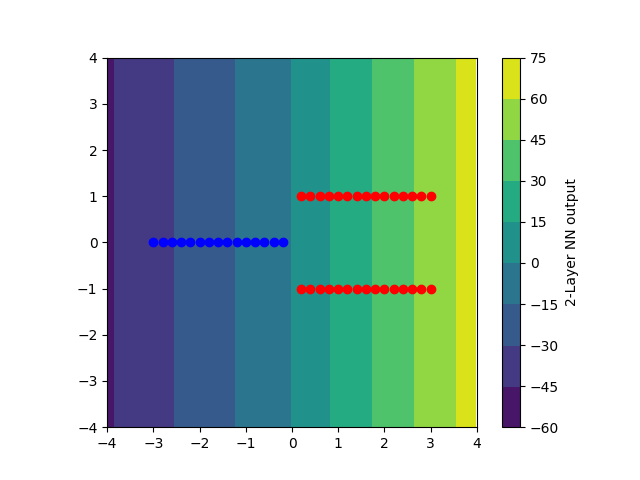} 
\caption{MIDDLE}
\end{subfigure}
\begin{subfigure}{0.33\textwidth}
\includegraphics[width=\linewidth]{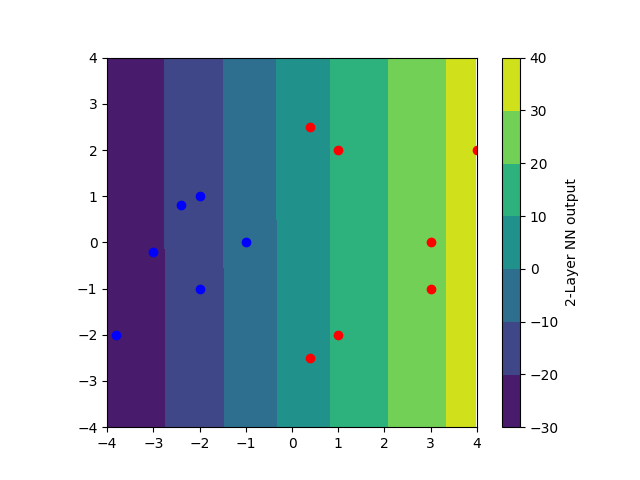} 
\caption{RIGHT}
\end{subfigure}
\caption{Two-layer Leaky ReLU neural nets converge to functions with linear decision boundary for the examples in \Cref{fig:decision_boundary}. The output contours are displayed in colors, and lighter  colors mean higher outputs.}
\label{fig:image2}
\end{figure}

\subsection{On the Non-branching Starting Point Assumptions} \label{sec:non-b}
\label{sec:nonunique}
In the proofs of the main theorems we make assumptions regarding the starting
point of gradient flow trajectories being non-branching
(\Cref{ass:non-branching} for the symmetric case and
\Cref{ass:non-branching-non-sym} for the non-symmetric case). The assumptions
address a technical difficulty due to the potential non-uniqueness of gradient flow
trajectories on general non-smooth loss functions. The motivations for these
assumptions are explained below.

\subsubsection{The non-uniqueness of gradient flow trajectories}

Gradient flow trajectories are unique on smooth loss functions by the classic
theory of ordinary differential equations. In this case, for trajectory
defined by $\frac{\dd \vtheta}{\dd t}=-\nabla \Loss(\vtheta)$, at any point
$\vtheta_0$, if both $\nabla \Loss(\vtheta_0)$ and $\nabla^2
\Loss(\vtheta_0)$ are continuous, then the trajectory is unique as long as it exists.

\begin{figure}[t]
\begin{subfigure}{0.33\textwidth}
\includegraphics[width=1\linewidth]{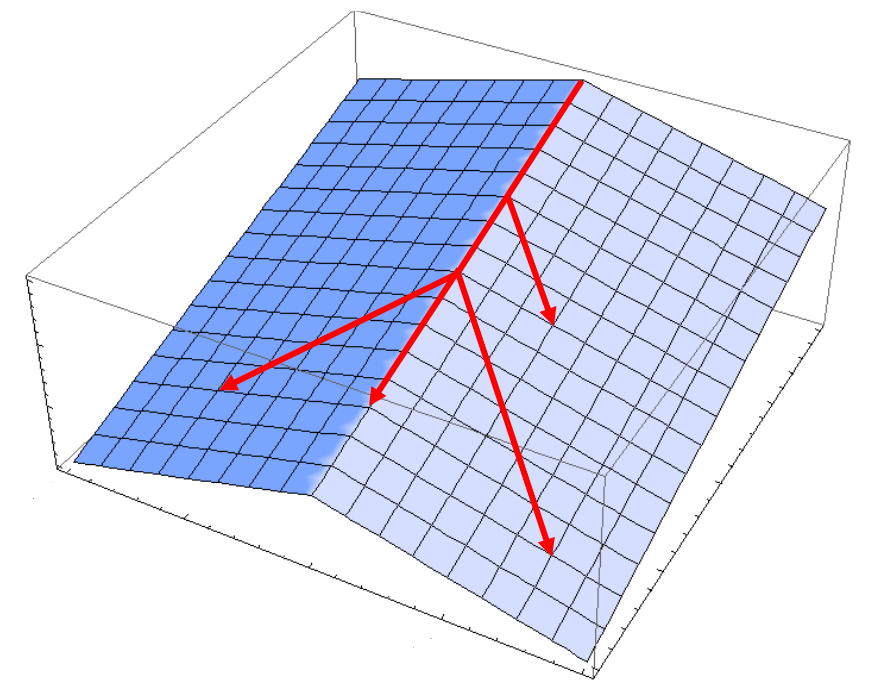} 
\caption{Ridge}
\end{subfigure}
\begin{subfigure}{0.33\textwidth}
\includegraphics[width=1\linewidth]{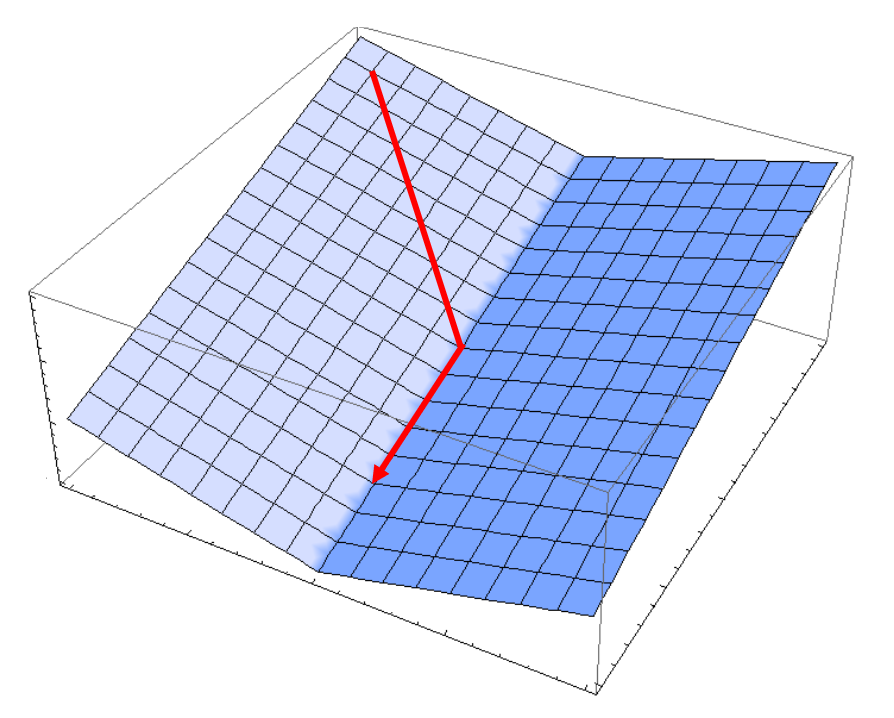} 
\caption{Valley}
\end{subfigure}
\begin{subfigure}{0.33\textwidth}
\includegraphics[width=1\linewidth]{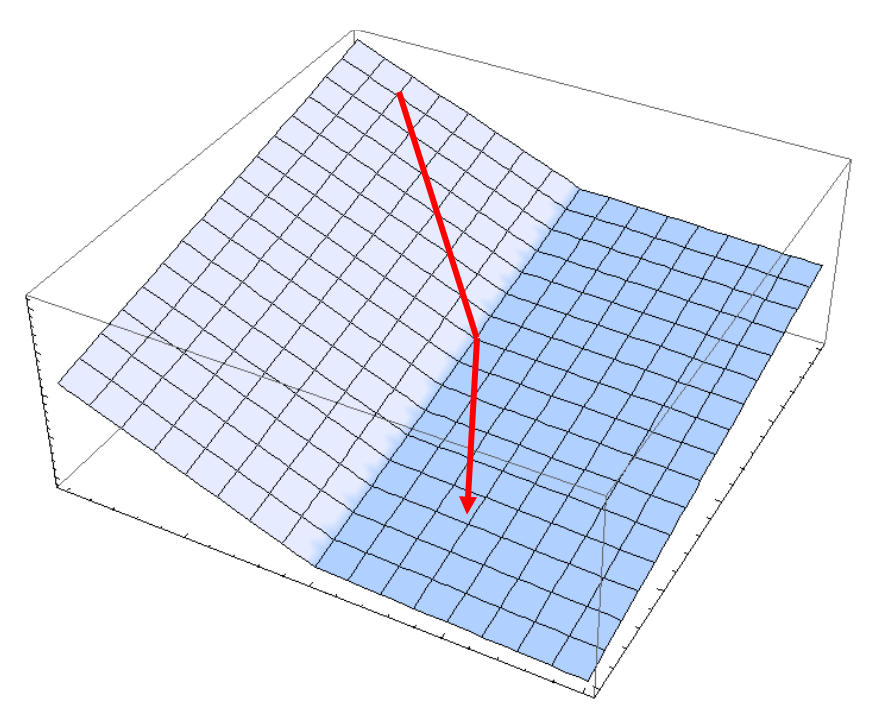} 
\caption{Refraction edge}
\end{subfigure}
\caption{Gradient flow trajectories behave differently in different landscapes. The trajectory may be non-unique only after arriving at a point on the "ridge".}
\label{fig:demoimage}
\end{figure}

For the non-smooth case with differential inclusion $\frac{\dd \vtheta}{\dd
t}\in-\cpartial \Loss(\vtheta)$, when $\Loss$ is continuous and convex, the
Clarke subdifferentials agree with the subdifferentials for convex functions,
and gradient flow trajectory is also unique (for instance see
\citealt{bolte2010characterizations}). However, on loss functions that are
non-smooth and non-convex, gradient flow may not be unique and the trajectory
may branch at non-differentiable points (see \Cref{fig:demoimage}). When a
non-differentiable point is atop a ``ridge'', a gradient flow reaching it may go
down different slopes next. Then any starting points wherefrom gradient flow can
reach such on-the-ridge points are not non-branching starting points as the
trajectory is not unique. For instance, with
$\Loss(\vtheta)=-\abssm{\dotpsm{\vtheta}{\vw}}$, then the trajectory with $\vtheta(t)=0$ for
$t<t_s$ and $\vtheta(t)=\pm (t-t_s) \vw$ for $t\geq t_s$ is a valid gradient
flow trajectory for any $t_s\geq 0$. On the other hand, when the point is
either at the bottom of a ``valley'' or at a ``refraction edge'', the trajectory
would not split. \Cref{fig:demoimage} sketches in red the possible gradient flow
trajectories in different circumstances.

In the case of two-layer Leaky ReLU network dynamics, there are settings where
\Cref{ass:non-branching} or \Cref{ass:non-branching-non-sym} holds. When data
points are  orthogonally separable (\Cref{def:ortho-separ}), all starting points
are non-branching. In this case, the output of each Leaky ReLU neuron will
change monotonically. By the chain rule, for any neuron $k\in[m]$, on any data
sample $i\in[n]$,
\[\dotp{\frac{\dd \vw_k}{\dd t}}{x_i} \in -\frac{a_k}{n}\sum_{j\in
[n]}\ell'(q_j(\vtheta))y_j\phi^{\circ}(\dotp{\vw_k}{ \vx_i})\dotp{x_i}{x_j}. \]
Then as $y_iy_j\dotp{\vx_i}{\vx_j}\geq 0$ by the orthogonally separability, the
sign of RHS is controlled by $\sgn(a_ky_i)$. With \Cref{thm:loss-convergence},
we know each $a_k$ does not change its sign along the gradient flow trajectory,
and therefore $\dotp{\vw_k}{x_i}$ changes monotonically. Then following the
arguements of the classic theory of ordinary differential equation, by applying
\Gronwall's inequality to both intervals $\{t:\dotp{\vw_k(t)}{x_i}>0\}$ and
$\{t:\dotp{\vw_k(t)}{x_i}\leq 0\}$ we know the trajectory is unique. In this
setting all the  non-differentiable landscapes resemble the ``refraction edges''.

In the general cases, it is a future research direction to find other analyses
that can replace the non-branching starting point assumptions, and doing so may
deepen our understanding in the trajectory behaviors in non-smooth settings. 

\section{Additional Experiments}
We conducted several additional experiments on synthetic datasets. The goal is
to show that 2-layer Leaky ReLU networks actually converges to the max-margin
linear classifiers in different settings with moderately small initialization.
The results are summarized in \Cref{tbl:exptbl} and \Cref{fig:expimage2}.

\begin{table}[htbp]
\begin{center}
\begin{tabular}{ |c|c|c| } 
 \hline
 Dataset size & SVM test error & 2-Layer neural net test error\\
 \hline
 10 & 30.2\% & 30.5 \% \\
 20	& 19.7\% & 18.9\% \\
 30 & 17.6\% & 15.6\% \\
 40 & 8.0\% & 7.1\% \\
 50 & 6.4\% & 5.9\% \\
 60 & 6.3\%	& 5.1\% \\ 
 70 & 7.6\%	& 6.5\% \\
 80 & 3.9\%	& 3.1\% \\
 90 & 6.1\%	& 5.2\% \\
 100 & 2.9\% & 2.9\% \\
 \hline
\end{tabular}
\end{center}
\caption{Test errors for SVM max-margin linear classifiers and 2-Layer ReLU neural networks are nearly the same across different data size.}
\label{tbl:exptbl}
\end{table}

\myparagraph{Data.} $n=10,20,\cdots,100$ data points are randomly sampled from
the standard gaussian distribution $\Normal(0,\mI)$ in the space of dimension
$d=50$, and are classified with a linear classifier through zero. Then the
points are translated mildly away from the classifier to make a small nonzero
margin that assists learning.

\myparagraph{Model and Training.} We used the two-layer leaky ReLU network with
hidden layer width $m=100$ and with bias terms. In out setting the bias term is
equivalent to adding an extra dimension of value $0.1$ to all the data points.
We trained our model with the gradient descent method from 0.001 times the He
initialization \citep{He_2015_ICCV} and initial learning rate 0.01. The
learning rate is raised after interpolation to boost margin increase.

We compare the neural network output with the max-margin linear classfier produced by the support vector machine (SVM) on hinge loss. In \Cref{tbl:exptbl}, the test errors are calculated from 10000 test points from the same distribution. In \Cref{fig:expimage2}, we drawn the decision boundaries for both the SVM max-margin linear classifier and the neural network restricted to a plane passing 0. 
The results show that the neural network classifier converges to the max-margin linear classfier in our setting.

\begin{figure}[htbp]

\begin{subfigure}{0.33\textwidth}
\includegraphics[width=\linewidth]{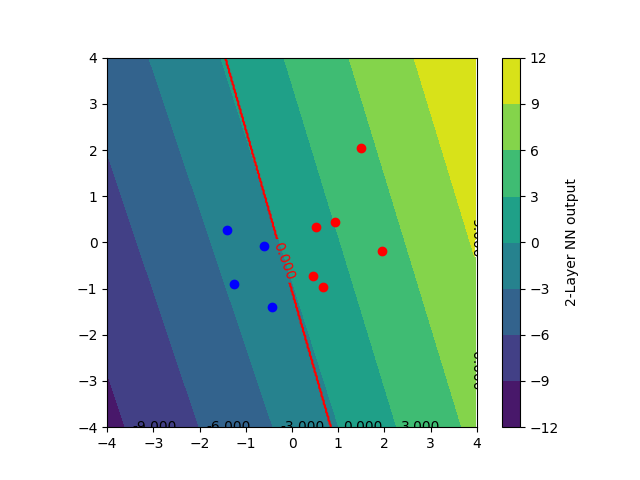} 
\caption{$n=10$}
\end{subfigure}
\begin{subfigure}{0.33\textwidth}
\includegraphics[width=\linewidth]{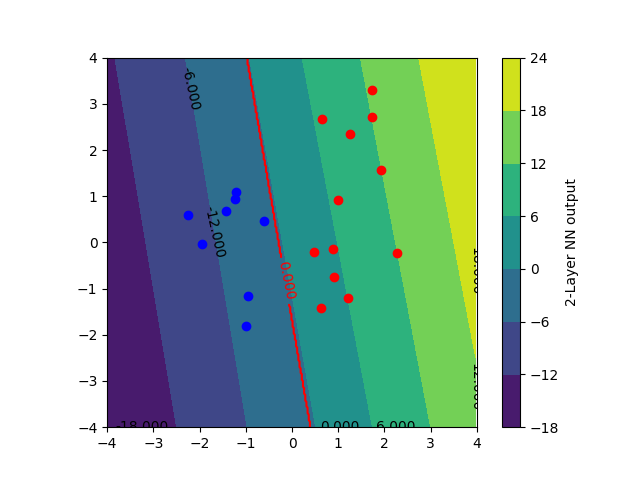} 
\caption{$n=20$}
\end{subfigure}
\begin{subfigure}{0.33\textwidth}
\includegraphics[width=\linewidth]{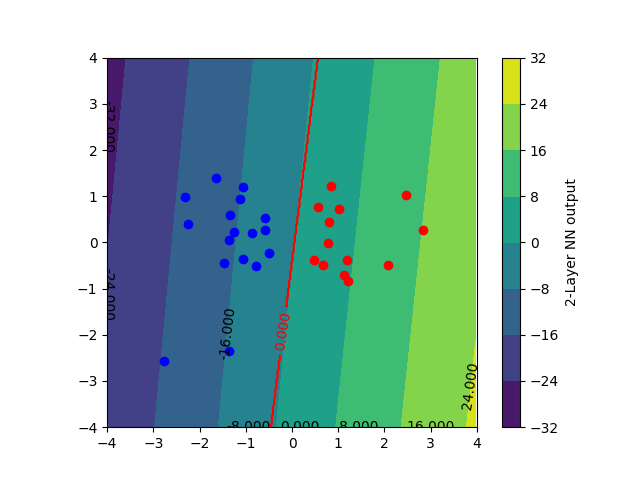} 
\caption{$n=30$}
\end{subfigure}
\begin{subfigure}{0.33\textwidth}
\includegraphics[width=\linewidth]{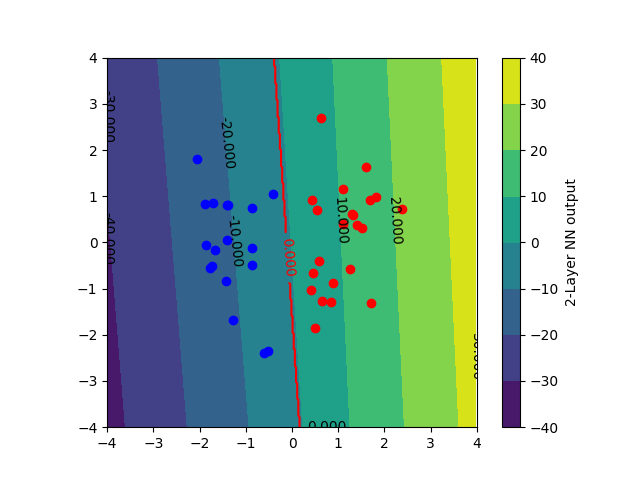} 
\caption{$n=40$}

\end{subfigure}
\begin{subfigure}{0.33\textwidth}
\includegraphics[width=\linewidth]{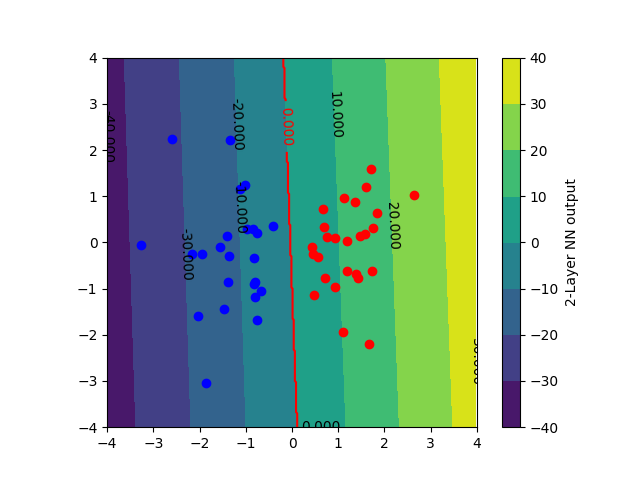} 
\caption{$n=50$}
\end{subfigure}
\begin{subfigure}{0.33\textwidth}
\includegraphics[width=\linewidth]{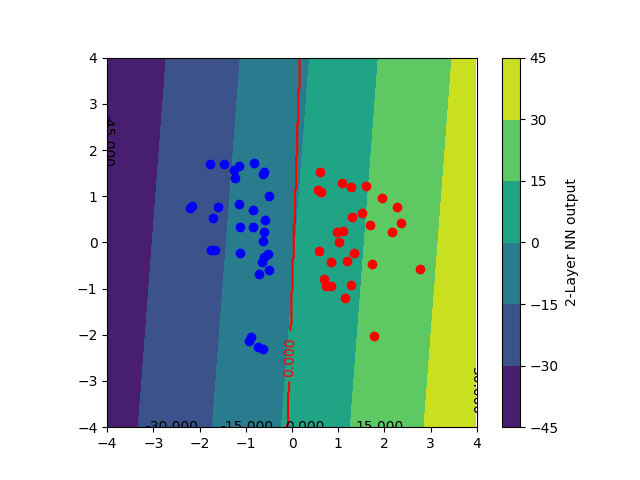} 
\caption{$n=60$}
\end{subfigure}
\begin{subfigure}{0.33\textwidth}
\includegraphics[width=\linewidth]{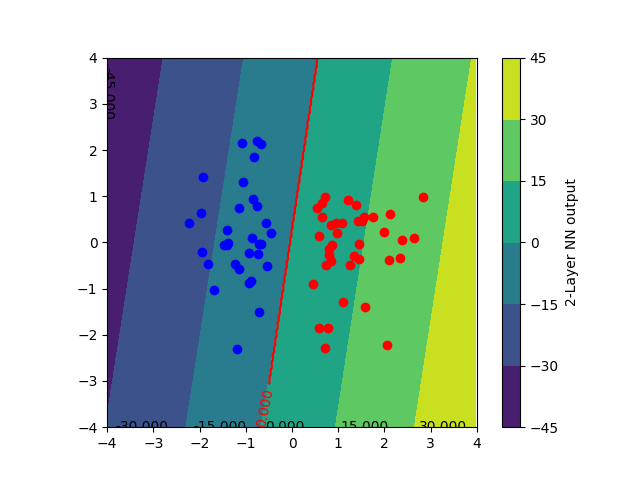} 
\caption{$n=70$}
\end{subfigure}
\begin{subfigure}{0.33\textwidth}
\includegraphics[width=\linewidth]{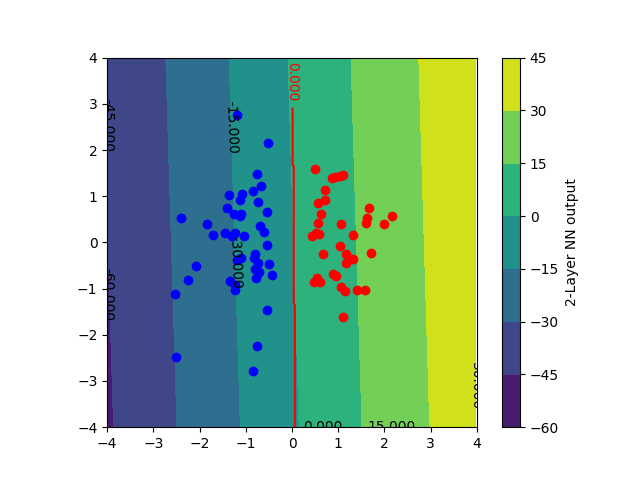} 
\caption{$n=80$}
\end{subfigure}

\begin{subfigure}{0.33\textwidth}
\includegraphics[width=\linewidth]{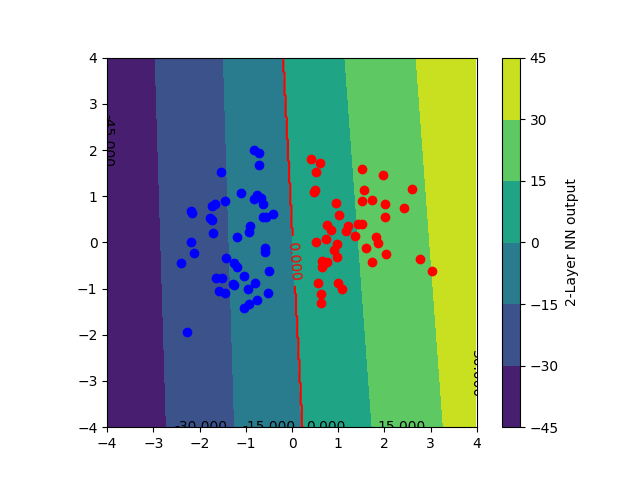} 
\caption{$n=90$}
\end{subfigure}
\begin{subfigure}{0.33\textwidth}
\includegraphics[width=\linewidth]{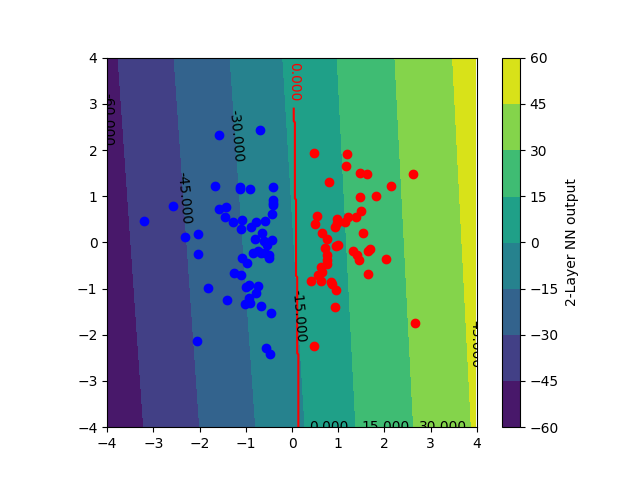}
\caption{$n=100$}
\end{subfigure}
\caption{Two-layer Leaky ReLU neural net converges in direction to the SVM max-margin linear classifier. \textbf{Red and Blue Dots:} two classes of data points. \textbf{Red Lines:} the decision boundaries of the SVM max-margin linear classifiers. \textbf{Background:} the contours of two-layer leaky-ReLU neural network outputs. Lighter colors mean higher outputs. The underlying true separator is the vertical line through zero.}
\label{fig:expimage2}
\end{figure}

\end{document}